\documentclass[12pt]{article}
\usepackage[margin=1in]{geometry}

\usepackage{times,upgreek,morenotations,rotating}

\usepackage[utf8]{inputenc} 
\usepackage[T1]{fontenc}    
\usepackage{url}            
\usepackage{booktabs}       
\usepackage{amsfonts}       
\usepackage{nicefrac}       
\usepackage{microtype}      

\usepackage{rotating,natbib}
\usepackage{array,colortbl}
\usepackage{morenotations,textcomp}

\usepackage{graphicx}
\usepackage{color,soul,textcomp}

\newcommand{\y}{x_i}
\newcommand{\x}{x'_i}

\title{The Crossover Process: \\
Learnability and Data Protection from Inference Attacks}

\author{
  Richard Nock\\
{\normalsize Data61, The Australian National University \& The
  University of Sydney}\\
  \texttt{richard.nock@data61.csiro.au}\\
\and
Giorgio Patrini\\
{\normalsize Data61 \& The Australian National University}\\
  \texttt{giorgio.patrini@anu.edu.au}\\
\and
  Finnian Lattimore\\
{\normalsize The Australian National University \& Data61}\\
  \texttt{finnian.lattimore@nicta.com.au}\\
\and
  Tiberio Caetano\\
{\normalsize Ambiata \& The University of Sydney}\\
  \texttt{tiberio.caetano@gmail.com}
}
\date{}

\begin{document}
\pagenumbering{Alph}
\thispagestyle{empty}
\maketitle
\pagenumbering{arabic}

\begin{abstract}

It is usual to consider data protection and learnability as conflicting
objectives. This is not always the case: we show how to jointly
control inference --- seen as the attack --- \textit{and} learnability by a noise-free process that mixes
training examples, the Crossover Process (\cp). 
One key point is that the
\cp~is typically able to alter joint distributions \textit{without} touching on
marginals, nor altering the sufficient
statistic for the class. In other words, it saves (and sometimes
improves) generalization for supervised
learning, but can alter the relationship between covariates --- and
therefore fool measures of nonlinear
independence and causal inference into misleading \textit{ad-hoc}
conclusions. For example, a \cp~can increase / decrease odds ratios, bring fairness or break
fairness, tamper with disparate impact,
strengthen, weaken or reverse causal directions, change
observed statistical measures of dependence. For each of these, we
quantify changes brought by a \cp, as well as its statistical impact
on generalization abilities via a new complexity measure that we call
the Rademacher \cp~complexity.
Experiments on a dozen readily available domains validate the theory.
\end{abstract}

\noindent\textbf{Keywords}: Supervised learning, Privacy, Fairness, Statistical inference, Causality.
\section{Introduction}

We study the problem of sensitive data sharing under two conflicting objective: protection of the data from unwanted inference and guarantees that supervised learning can be effective after the data protection mechanism has been applied. The two goals are inherently in tension, yet not necessarily in contradiction. Any credible solution to the combined issues would have a considerable impact on open data policies for research and commercial enterprises. 

We motivate our goal by an example. A medical laboratory aims to
release to the research community a newly collected dataset about its
genetic, behavioural, habits and infection history of the patients
affected by cancer, with the intent of letting other institutions to
test their predictive models on it so as to improve diagnosis
methodologies. Even assuming that we perfectly anonymize it, this data
is still extremely sensitive. Anyone possessing it could directly make
statistical queries and causal inference between specific patient
traits, or combine it with his/her own private data using powerful causal
inference techniques \citep{pbEV}.

Ideally, researchers should be able to release the data and hide discriminatory and sensitive relationships, such as smoking tendency by ethnicity and gender-prone infections \citep{rhMD},  \textit{while} making sure
that the utility of the dataset for predicting the sickness state remains unaltered.  Can we design a procedure that would transform the data in a form apt to publication, that is, erasing any trace of statistical or causal relationship between particular pairs of attributes \emph{and}, at the same time, leaving prediction performance virtually untouched? Several streams of research may share similarities with this open question thus we start by covering the current background.

Privacy is a growing concern in the
public sphere \citep{bnBD,ecTE,geaBC,mkDS} (and references therein). Two leading mechanisms for the private release of data
are differential privacy and $k$-anonymity
\citep{drTA,ecTE,sAK}. They guarantee \textit{individual level}
protection, \emph{i.e.} identifiability. We depart from this view on
the problem because we are concerned with inference at global level
over the present data,
for which for example differential privacy does not provide any sort
of guarantee.

In fact, as pointed out in \cite{bnBD}, "even when individuals are not \emph{identifiable} they may still be
\emph{reachable} [...] and subject to \textit{consequential
  inferences} and predictions taken on that basis". The reference is
to the possibility of performing \emph{inference attacks} by a
malicious agent willing to uncover \emph{causal relationships}, or even
just measure \textit{statistical independence}, between
sensitive covariates. Even when true causality is sometimes
considered "a research field in its infancy" \citep{geaBC}, it is
hard to exaggerate the recent burst in causal inference techniques
\citep{cmT2,dmzsAP,gbssMS,gftsssAK,hjmpsNC,jmzlzdssIG,lllljsmFO,mpjzsDC},
as well as the threats this may pose on privacy
\citep{bnBD,bnCE,ecTE,ksswIT,mkDS}. 

In simple terms with a quite general example, the attacker
estimates from the data/outputs some $\Pr[U|\mathcal{V}]$, where $U$ is typically a
sensitive attribute and $\mathcal{V}$ is built from one (or a set of) protected
attribute(s) \citep{hpsEO}. Protection against such attacks involve in general
controlling similar estimates or odds ratios. Since the advent of differential privacy \citep{drTA}, these
questions have received a steadily increasing treatment, with a
further surge over the past two years over fairness considerations
\citep{ffmsvCA,hpsEO,kmrIT,mkDS}. 


When it comes
to supervised learning, there is often a single sensitive
attribute $U$ to protect, in general a score or prediction \citep{ffmsvCA,hpsEO,kmrIT}. This
is quite restrictive for our purpose if we consider that the total number of
observation variables is blowing up in hundreds, thousands or more in
mainstream datasets. This is not
to say that previous techniques do not or cannot apply, but there could be at
least a
serious combinatorial overload to treating a lot of sensitive
attributes with techniques fit for one. Finally, data protection is not the sole constraint ---
otherwise, communicating noise would just solve the
problem. Guaranteed protection has to come with provable
utility, \textit{i.e.} learnability \citep{drTA, hpsEO}.
In the context of differential privacy, the trade-off does not play favourably for learning \citep{djwLP}.

In the design a solution to the problem, we keep in mind an additional requirement.
At the age where protection is shifting towards
statistical information --- in constrast with computation, 
\textit{e.g.} for public key encryption ---, a
good protection mechanism that targets specific utility is one that,
knowing all the public part of the protection mechanism\footnote{This is Kerckhoff's
principle, \cite{mkDS}.}, gives the
\textit{least} information about any other sensible content. 
This is not trivial to satisfy.
In fact,
knowing for example that a dataset was protected with a specific technique \textit{for fairness}
(say, odds
ratios = 1, \cite{hpsEO}) leaks information: if the attacker sees that
some attributes that are important for him do not display fairness in
data, then they were probably not treated and he can use this data for
his own analyses \citep{pbEV}. If, however, a ``suspicious'' amount of
fairness is detected, then the attributes were probably treated, and
if those attributes relate smoking and cancer, then it is not hard to
imagine the most likely imbalance in original odds that justified protection (smoking causes
cancer). In sum, we should seek a data protection mechanism able to bring, \textit{or
break}, fairness, and thereby be able to just
\textit{fool} inference into misleading \textit{ad-hoc}
conclusions.

\begin{center}
\fbox{%
    \parbox{0.95\textwidth}{
In short, we carry out protection at the \textit{upstream level}, \textit{i.e.} the
training sample's\footnotemark[2]
 and we target two goals:
the dataset's
utility for the black-box supervised prediction task remains
\textit{within control}\footnotemark[3], but
it is
surgically altered against fined-grained specific inference attacks
among description features. Alteration can work in \textit{all
  directions}: increasing / decreasing odds, being fair or breaking
fairness, tampering with disparate impact,
strengthening, weakening or reversing causal directions, changing
observed statistical measures of dependence, and so on.
}
}\footnotetext[2]{Neither the algorithm \citep{cmsDP}, nor its
output \cite{hpsEO} but on the input data as in \emph{local privacy} \citep{djwLP}.}
\footnotetext[3]{Like \citep{cmsDP}, we investigate the
  \textit{generalization} abilities impacted by data protection, and do not remain within the realm
of the empirical risk.}
\setcounter{footnote}{3}

\end{center}

\noindent We show that this task is within reach \textbf{and it
  involves the same protection process for all}. It can also be very surgical
--- for example, marginals can remain untouched and therefore may not raise
suspicions about protection. Coping with the desired level of protection to statistical
independence and causal inference attacks may require wrangling the
complete data, but this may be done with
a tight \textit{explicit} control of its utility for supervised
learning, and, as we show, it may
even yield \textit{better} models for prediction. Although counterintuitive, this last fact should not come with great surprise considering the success of sophisticated noisification methods, \emph{e.g.} dropout \citep{shkssDA}, to enhance learning.\\

\noindent Our \textbf{main contribution} is the introduction of the \textit{Crossover Process}, \cp. An
analogy may be done with the biological crossover: a population of DNA
strands gets mixed with a crossover, but there is a single zone for chiasma
(\textit{i.e.} contact point) for
the whole population. In the same way as DNA strands exchange genetic
material during recombination, feature values get mixed between
observations during a \cp~, although in a more general way than in genetic recombination.
The key to learning and generalization is that the \cp~may be
done without changing the
sufficient statistic for the class \citep{pnrcAN, pnncLF}, nor touching
class-based marginals. The key
to interfering with measures of (un)fairness, independence and causal calculus is that the
\cp~is able to surgically alter joint distributions. Our contribution
is therefore twofolds: (i) we introduce the \cp~and show how it drives the
generalisation abilities of linear and some non-linear classifiers by the
introduction of a new statistical complexity measure, the Rademacher
\cp~complexity (\rcp). We show that the \rcp~can be very significantly
smaller than the standard empirical Rademacher complexity, thereby being
a lightweight player --- and a tractable knob --- for
generalisation. Then, (ii) on the data protection standpoint, we show
\begin{enumerate}
\item [(a)] how the
components of a \cp~may be chosen to alter odds ratios and measures of
(un)fairness, equality of opportunity, equalized odds or disparate impact \citep{ffmsvCA,hpsEO,kmrIT}, 
\item [(b)] how it can be built to alter
the powerful
Hilbert-Schmidt independence criterion \citep{gftsssAK}, 
\item [(c)] how it
may be devised to blow-up causal estimation errors \citep{cmT2}, and
finally 
\item [(d)] how it
can interfere with identifiable causal queries on a causal graph in
the \textit{do} calculus framework
\citep{pCM,shpitser2006identification}.
\end{enumerate}
Targeting all these different models of dependence \textit{exhaustively} would
require far more than the paper's current size and technical
content. Yet, all of them are important and forgetting one would
reduce \textit{de facto} the scope of the \cp~from the protection standpoint. 
This is why we deliberately chose to make a very specific treatment of
some, in particular for \citep{cmT2}.

\noindent \textbf{Organisation of the paper} --- Section $\S$\ref{sec-dsp} gives
general definitions. $\S$\ref{sec-dsp-pres} presents the
Crossover Process, $\S$\ref{sec-dsp-fair} relates the \cp~to measures
of (un)fairness and $\S$\ref{sec-dsp-learn} presents its relationships with
learnability. $\S$\ref{sec-statind} shows the impact of the \cp~on
measures of independence and $\S$\ref{sec-causal} shows the impact of
the \cp~on causal queries. A last Section discusses and concludes.
An Appendix provides all proofs, additional
results and some extensive experiments performed to assess the theory. A movie\footnote{Available anonymously at
  \texttt{https://youtu.be/4d5Z23cwEyY}},
presented in Subsection \ref{res-movie}, 
shows the effects of the \cp~on a popular domain for causal discovery
\citep{hjmpsNC}. Figure \ref{f-orga} provides a high-level overview of
the papers topics for the main technical Sections.

\begin{figure}[t]
\begin{center}
\includegraphics[trim=170bp 420bp 470bp
100bp,clip,width=0.30\columnwidth]{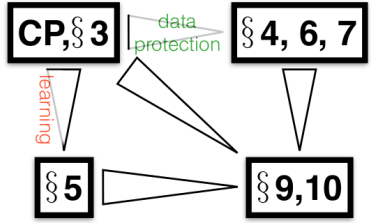}
\end{center}
\caption{Organisation and dependencies of the main Sections of the paper according to
  learning and data protection. Sections are independent
  within each rectangle.}
\label{f-orga}
\end{figure}

\section{General notations and definitions}\label{sec-dsp}

\textbf{Learning setting} --- We let $[m] \defeq \{1, 2, ...,
m\}$ and $\Sigma_m \defeq \{\bm{\sigma}
\in \{-1,1\}^m\}$. ${\mathcal{X}} \subseteq {\mathbb{R}}^d$ is a
domain of observations. 
Examples are couples (observation, label) $\in  {\mathcal{X}}\times \Sigma_1$, sampled i.i.d. according to some unknown but fixed distribution
  ${\mathcal{D}}$. We denote ${\mathcal{F}} \defeq [d]$ the set of
  observation attributes (or features). ${\mathcal{S}} \defeq \{(\bm{x}_i, y_i), i \in
  [m]\} \sim {\mathcal{D}}_m$ is a training sample of
  $|{\mathcal{S}}| = m$ examples. 
For any vector $\ve{z}\in
  {\mathbb{R}}^d$, $z_j$ denotes its coordinate $j$. Finally, notation
  $x
  \sim X$ for $X$ a set denotes uniform sampling in
  $X$, and the mean operator is
  $\ve{\mu}_{{\mathcal{S}}} \defeq \expect_{(\ve{x},y)
    \sim{\mathcal{S}}}[y \cdot \ve{x}]$ \citep{pnrcAN,pnncLF}.

In supervised learning, the task is to learn a
  classifier ${\mathcal{H}} \ni h : {\mathcal{X}} \rightarrow
  {\mathbb{R}}$ from ${\mathcal{S}}$ with good generalisation
  properties, that is, having a small \textit{true risk}
  $\expect_{(\ve{x},y) \sim \mathcal{D}} [L_{0/1}(y, h(\bm{x}))]$, with $L_{0/1}(z,
  z') \defeq 1_{zz' \leq 0}$ the 0/1 loss ($1_.$ is the indicator
  variable). In general, this is achieved by minimising over
  ${\mathcal{S}}$ a $\varphi$-\textit{risk} $\expect_{(\ve{x},y) \sim\mathcal{S}}
  [\varphi(yh(\bm{x}))] = (1/m)\cdot\sum_i \varphi(y_ih(\bm{x}_i))$, where $\varphi(z)\geq 1_{z\leq 0}$ is a
  \textit{surrogate} of the 0/1 loss. In this paper, $\varphi$ is any
  differentiable proper symmetric (PS) loss \citep{nnBD,pnrcAN} (symmetric meaning
  that there is no class-dependent misclassification cost). The
  logistic, square and Matsushita losses are examples of PS
  losses. Set ${\mathcal{H}}$ is a predefined set of classifiers, such
  as linear separators, decision trees, etc. .

\noindent \textbf{Matrix quantities} --- The set of unnormalised
column stochastic
matrices, ${\mathcal{M}}_{n} \subset {\mathbb{R}}^{n \times n}$, is the
superset of column stochastic matrices for which we drop the
non-negativity constraint, thus keeping the sole constraint of unit
per-\textit{column} sums. We let $S_n \subset {\mathcal{M}}_{n}$ denote the symmetric group
    of order $n$. For any $\matrice{a}, \matrice{b}\in
    {\mathbb{R}}^{n\times n}$ and $\matrice{m} \in {\mathcal{M}}_n$,
    we let 
\begin{eqnarray*}
\cipr{\matrice{a}}{\matrice{b}}{\matrice{m}} & \defeq & 
\trace{(\matrice{i}_n - \matrice{m})^\top \matrice{a} (\matrice{i}_n -
  \matrice{m}) \matrice{b}}
\end{eqnarray*}
denote the \textit{centered inner product} of $\matrice{a}$ and $\matrice{b}$
with respect to $\matrice{m}$. It is a generalisation of the centered inner product
used in kernel statistical tests of independence \citep{gbssMS}, for
which $\matrice{m} =
(1/n)\ve{1}\ve{1}^\top$. 

Without loss of generality, we shall assume
that indices in ${\mathcal{S}}$ cover first the positive
class: $(y_i = +1 \wedge y_{i'} = -1) \Rightarrow i<i'$. A key subset
of matrices of ${\mathbb{R}}^{m\times m}$ consists of block matrices whose coordinates on indices corresponding to different
classes in ${\mathcal{S}}$ are zero: block-class matrices.
\begin{definition}
$\matrice{a} \in {\mathbb{R}}^{m\times m}$ is a \textbf{block-class
matrix} iff $(y_i \cdot y_{i'} =
-1) \Rightarrow \matrice{a}_{ii'} = 0, \forall i, i'$.
\end{definition}
An asterisk exponent in a subset of matrices indicates the intersection
of the set with block class matrices, such as for 
${\mathcal{M}}^*_{n} \subset {\mathcal{M}}_{n}$ and $S^*_n \subset
S_n$. Finally, matrix entries are noted with double indices like
$\matrice{m}_{ii'}$; replacing an index by a dot, ``.'', indicates
a sum over the index, like $\matrice{m}_{i.} \defeq \sum_{i'} \matrice{m}_{ii'}$.

\section{The Crossover Process}\label{sec-dsp-pres}

The Crossover process (\cp) 
transforms ${\mathcal{S}}$ in two steps: the split and the shuffle
step. In the split step, a bi-partition of the features set
${\mathcal{F}}$ is computed: ${\mathcal{F}} = {\mathcal{F}}_\re \cup
\mathcal{F}_\tr$. ${\mathcal{F}}_\re$ is the \textit{a}nchor set and
${\mathcal{F}}_\tr$ is the \textit{s}huffle set. 
To perform the shuffle step, we need additional notations. Without loss of generality, we
assume ${\mathcal{F}}_\re \defeq [d_\re]$ and
${\mathcal{F}}_\tr \defeq \{d_\re+j, j\in [d_\tr]\}$, $d_\re > 0, d_\tr > 0, d_\re+ d_\tr =
d$. So, ${\mathcal{F}}_\re$ contains the first $d_\re$ features and
${\mathcal{F}}_\tr$ contains the last $d_\tr$ features. Let
$\matrice{i}_d$ be the identity matrix, and $[\matrice{f}^\re|
\matrice{f}^\tr] = \matrice{i}_d$ a vertical block partition where
$\matrice{f}^\re \in
{\mathbb{R}}^{d\times d_\re}$ ($\matrice{f}^\tr \in
{\mathbb{R}}^{d\times d_\tr}$) has columns
representing the features of ${\mathcal{F}}_\re$ (${\mathcal{F}}_\tr$)
--- we use notation $[.]$ both for integer sets and block
matrices without ambiguity. Finally, we
define the (\textit{row}-wise) observation matrix $\matrice{s} \in
{\mathbb{R}}^{m\times d}$ with $(\matrice{s})_{ij} \defeq x_{ij}$. Let
$\ve{1}_i$ be the $i^{th}$ canonical basis vector.
\begin{definition}\label{deftdp}
For any block partition $[\matrice{f}^\re|
\matrice{f}^\tr] = \matrice{i}_d$ and any \textbf{shuffle matrix} $\matrice{m} \in {\mathcal{M}}_{n}$, the
Crossover process $\process{T}\defeq \cp(\mathcal{S}; \matrice{f}^\re,
\matrice{f}^\tr, \matrice{m})$ returns $m$-sample
${\mathcal{S}}^{\process{T}}$ such that its observation matrix is
$\matrice{s}^{\tinymatrice{m}}\defeq [\matrice{s}\matrice{f}^\re |
\matrice{m}\matrice{s}\matrice{f}^\tr]$, and each example
${\mathcal{S}}^{\process{T}}  \ni (\ve{x}^{\tinymatrice{m}}_i, y_i) \defeq
((\matrice{s}^{\tinymatrice{m}})^\top \ve{1}_i, y_i)$.
\end{definition}
We consider $\matrice{m}$ fixed beforehand. Figure \ref{f-tdp} (top) presents the \cp~on a toy data with
$\matrice{m}$ a permutation matrix (invertible). Figure \ref{f-tdp} (bottom)
presents another example with $\matrice{m}$ block-uniform (non invertible).

\begin{figure}[t]
\begin{center}
\begin{tabular}{|cc|}\hline\hline
\multicolumn{2}{|c|}{\includegraphics[width=0.90\columnwidth]{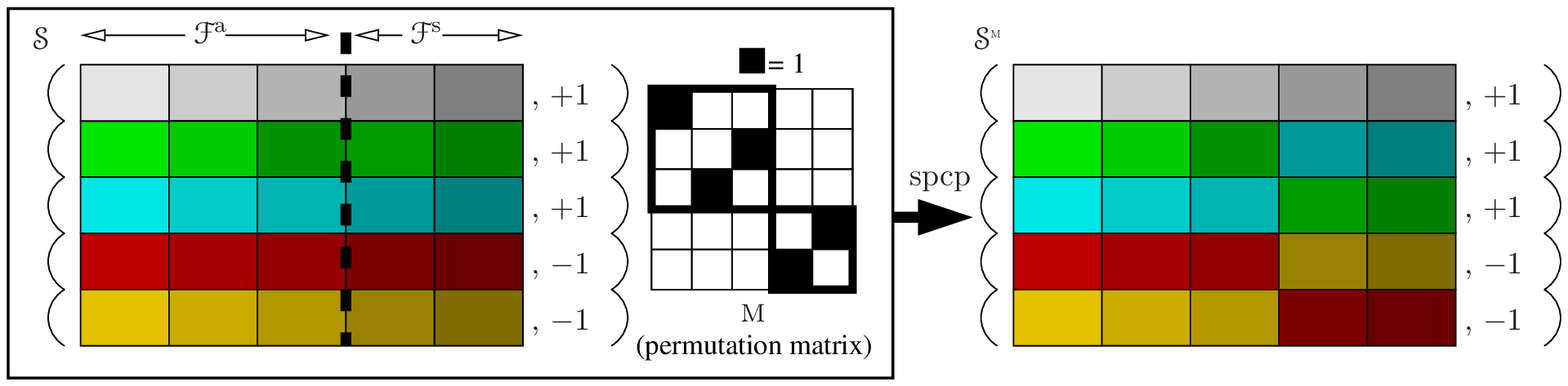}}\\ \hline
\includegraphics[trim=60bp 70bp 55bp
70bp,clip,width=0.45\columnwidth]{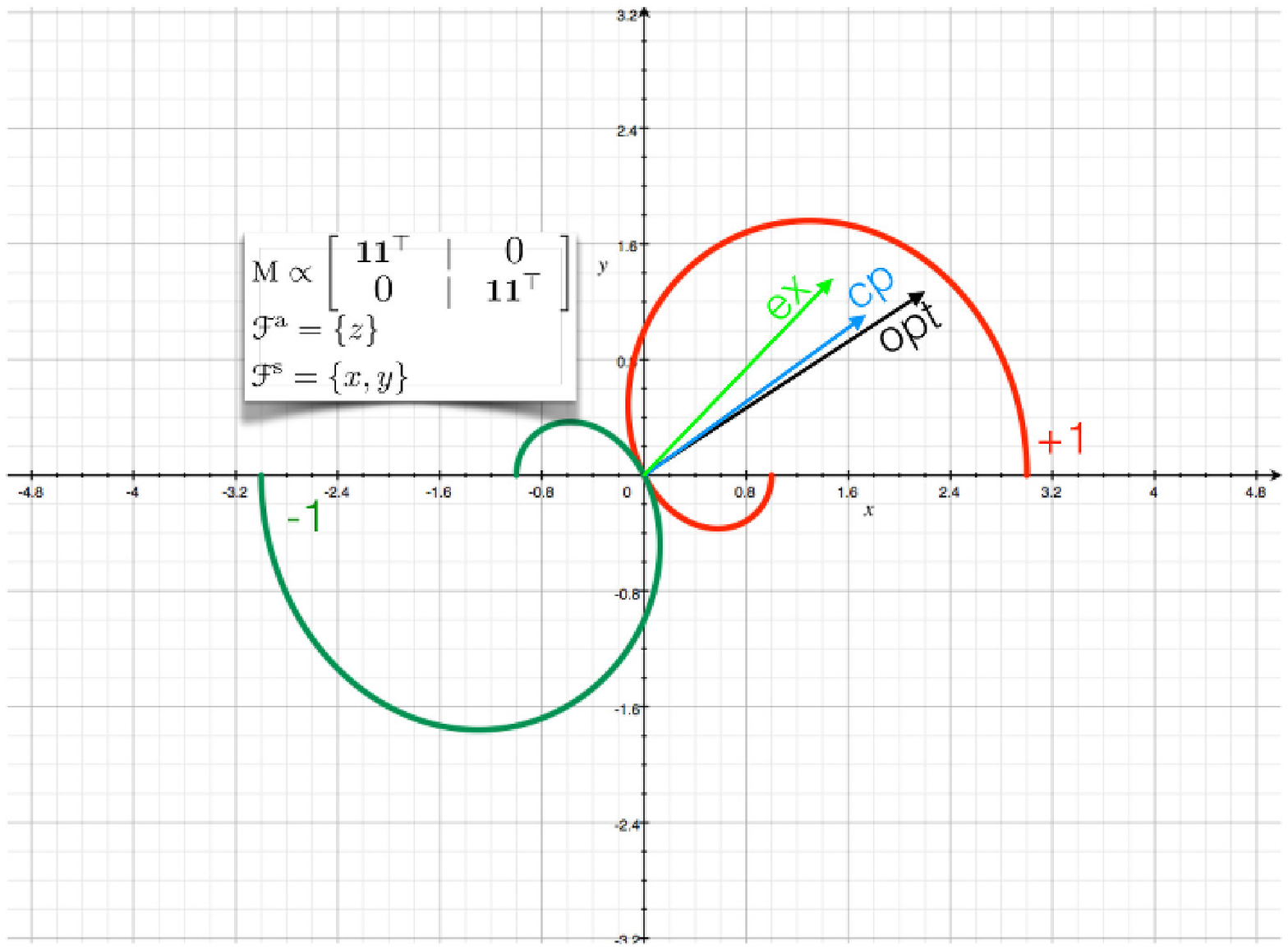} &  \includegraphics[trim=20bp 10bp 20bp
10bp,clip,width=0.40\columnwidth]{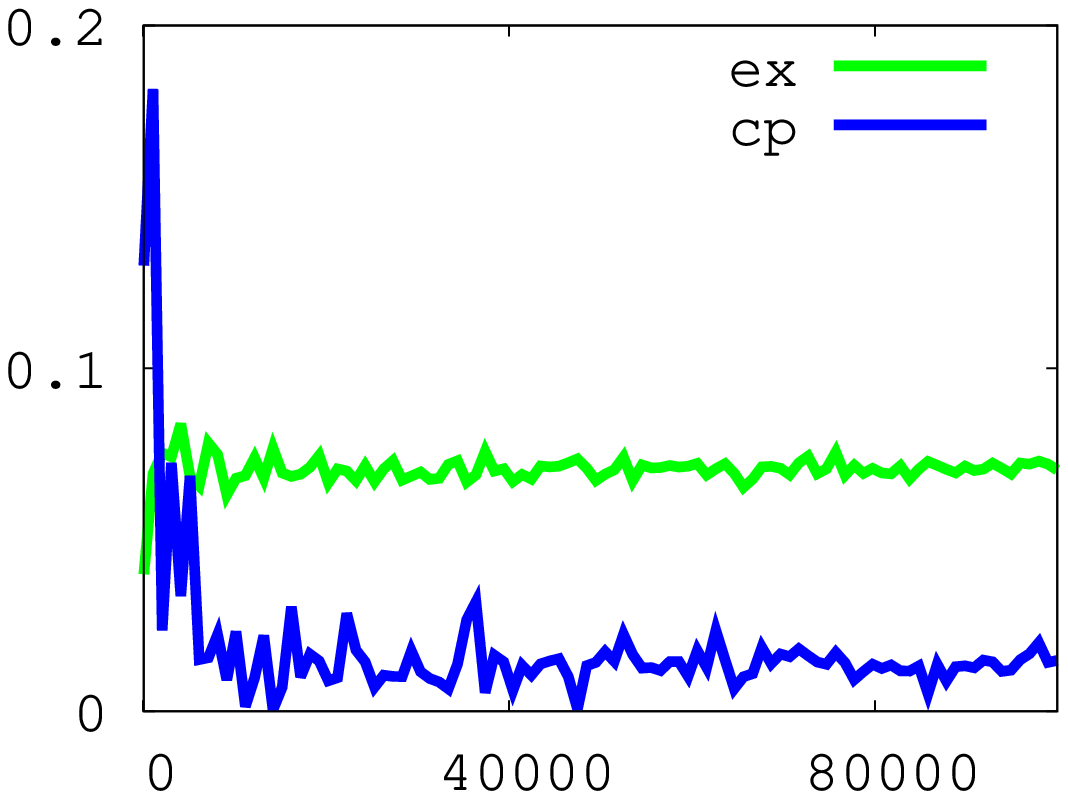}\\ \hline\hline
\end{tabular}
\caption{\textit{Top}: example of \cp~with $\matrice{m}$ a block-class permutation
  matrix (the two blocks are in bold). \textit{Bottom}: toy domain where
  $d=3$, but all examples have zero $z$-coordinate (not shown). The
  \cp~uniformly mixes examples by class. The domain consists of two
  spirals (red for positive, green for negative examples) with
  $\mathcal{D}$ = uniform
  distribution. Arrows
  depict respectively the optimal direction (black), and the
  directions learned by minimizing $\varphi$ = square loss over
  ${\mathcal{S}}$ (light green, "ex") and ${\mathcal{S}}^{\process{T}}$
  (blue, "\cp"). The right plot displays test errors ($y$-scale) on uniform sampling of
  datasets of different sizes ($x$-scale). The effect of the \cp~is to produce in
  ${\mathcal{S}}^{\process{T}}$ \textit{two} distinct examples that \textit{average} the positive / negative
  examples, and yield a better approximation of the optimum.}
\label{f-tdp}
\end{center}
\end{figure}

\section{The Crossover Process and measures of (un)fairness}\label{sec-dsp-fair}


Before drilling into the technical impact on learnability of a \cp, it
is good to make a small incursion in how the \cp~can be used 
a simple model of data protection that has received a surge of
treatment over the last years \citep{mkDS,ffmsvCA,hpsEO,kmrIT}. It is
sometimes related to as fairness, equality of opportunity, equalized
odds or disparate impact. It essentially builds on odds ratios.

Let $x_C$ and $x_A$ be two binary
attributes and $\ve{\pi}$ a predicate defined on other description variables, like
for example $\ve{\pi}\equiv \ve{x}_{\mathcal{V}} = \ve{v}$, where $\ve{x}_{\mathcal{V}} \subseteq
\mathcal{F}\backslash\{x_C, x_A\}$ and $\ve{v}$ is an
instantiation of $\ve{x}_{\mathcal{V}} $. Define the \textit{odds ratio} 
\begin{eqnarray}
\rho(x_C,x_A,\ve{\pi}|\mathcal{S}) & \defeq &
\frac{\Pr_{\mathcal{S}}[x_C = 1 | x_A = 0, \ve{\pi}]}{\Pr_{\mathcal{S}}[x_C
  = 1 | x_A = 1, \ve{\pi}]}\:\:.\label{defodds}
\end{eqnarray}
In this definition, $x_C$ is the sensitive feature, $x_A$ is a protected
attribute \citep{hpsEO} and the eventual additional features in
$\ve{x}_{\mathcal{V}}$ are 
a \textit{private}
subset of attributes. For example, private attributes can contain
additional features on which we want to constrain fairness measures,
like qualification in \cite{hpsEO}. If $\ve{\pi} = 
\top$ (the predicate that is always true), we just write $\rho(x_C,x_A|\mathcal{S})$.
\begin{definition}\label{dodds}
Let $x_A, x_C$ be two binary attributes and $\ve{\pi}$ a predicate defined on other description variables. For any $\rho \in \mathbb{R}_+$, we
say that sample ${\mathcal{S}}$ has $\rho$-odds ratio for the triple
$(x_A,x_C,\ve{\pi})$ iff $\rho(x_C,x_A,\ve{\pi}|\mathcal{S}) = \rho$.
\end{definition}
We can also replace real $\rho$ by a subset $\mathcal{R} \subseteq
\mathbb{R}_+$, in which case we must have $\rho(x_C,x_A,\ve{\pi}|\mathcal{S}) \in \mathcal{R}$.
Here are some examples of how this definition aligns with previous
works. 
If one takes $\ve{\pi} \equiv x_{\hat{C}} = 1$ where
$x_{\hat{C}}$ is a proxy for $x_C$, like an estimate for $x_C$
obtained using a specific procedure, then requiring $\rho = 1$
brings the condition for balance on the positive class from
\cite{kmrIT}; if on the other hand $\ve{\pi} \equiv x_{\hat{C}} = 0$,
then requiring $\rho = 1$
brings the condition for balance on the negative class from
\cite{kmrIT}; if finally $\ve{\pi} = 
\top$, requiring $\rho = 1$ brings the condition for balance within
groups from \cite{kmrIT}. Replacing $\rho = 1$ by $\mathcal{R} \defeq
[1-\epsilon, 1+\epsilon]$ brings the corresponding approximated
fairness conditions of \cite{kmrIT}. Permuting, in the balance for
positive class, the role of $x_C$ and
$x_{\hat{C}}$, still with $\alpha = 1$, brings the condition for equal
opportunity in \cite{hpsEO}, adding a second $1$-odds ratio condition for
$(x_A,x_{\hat{C}},x_{\hat{C}} = 0)$ brings equalized odds in
\cite{hpsEO}. Finally, replacing, $\rho = 1$ by $\mathcal{R} = (0.8,+\infty)$ in
the condition for balance within
groups above yields the no-disparate impact condition of
\cite{ffmsvCA}. 
Let us see now what a simple \cp~can do to alter odds like in
eq. (\ref{defodds}), via the following Definition.

\begin{table}[t]
    \centering
\begin{center}
\begin{tabular}{c||c|c||}
 & $x_C = 0$ & $x_C = 1$\\ \hline \hline
 $x_A = 0$ & $a$ & $b$\\ \hline
 $x_A = 1$ & $c$& $d$\\ \hline \hline
\end{tabular}
\caption{Contingency table (conditioned on $\ve{\pi}$ being true) for two binary attributes $A$ and $C$
  between which a \cp~is going to change the dependency relationships
  and odds ratios ($a+b+c+d \leq m$, see text for details).}\label{t-cont}
\end{center}
\end{table}

\begin{definition}\label{defoddsratios}
We say that a \cp~$\process{T}$ shifts the odds ratio for the triple
$(x_A,x_C,\ve{\pi})$ by $\Delta$ on dataset
${\mathcal{S}}$ iff $\rho(x_C,x_A,\ve{\pi}|\mathcal{S}^{\process{T}})
= \rho(x_C,x_A,\ve{\pi}|\mathcal{S}) + \Delta$.
\end{definition}

\begin{lemma}\label{lemmodds1}
 Suppose Table \ref{t-cont} describes the observed joint distribution
for attributes $x_A$ and $x_C$ in sample ${\mathcal{S}}$, conditioned
on $\ve{\pi}$ being true (hence, $a+b+c+d \leq m$). Let
\begin{eqnarray}
\Delta(i) & \defeq & \frac{b+d}{d-i}\cdot\frac{i}{d}\:\:, \forall i\in \mathbb{Z}\:\:.
\end{eqnarray}
Then, for any sample ${\mathcal{S}}$ and any $i\in \left\{-\min\{b,c\} , -\min\{b,c\}+1, ..., \min\{a,d\}
\right\}$, 
there exists a \cp~$\process{T}_i$ that shifts the odds ratio for the triple
$(x_A,x_C,\ve{\pi})$ by $\Delta(i)$ on
${\mathcal{S}}$.
\end{lemma}
(proof in Subsection \ref{proof_lemmodds1}) The proof of the Lemma
involves very simple \cp s, for which the shuffle matrix $\matrice{m}$
is a permutation matrix. As a consequence, the Lemma implies that such
a simple
\cp~can produce fairness
($\rho(x_C,x_A,\ve{\pi}|\mathcal{S}^{\process{T}}) = 1$) as long as $b\leq d +
2\min\{a,d\}$ and $d\leq b + 2\min\{b,c\}$, \textit{i.e.} as long as
the joint distribution is not too unbalanced; since $i$ can take on
both positive and negative values, a \cp~can also shifts the odds ratio
in $\mathcal{S}$ to smaller or larger values. If $b=d$, it can
therefore also break fairness with shifts $\Delta(i) = 2i/(d-i)$.

\section{The Crossover Process and learnability}\label{sec-dsp-learn}

\noindent\textbf{Generalization} --- We
now explore the effect of the \cp~on generalisation. We need two assumptions on ${\mathcal{H}}$ and 
$\varphi$. The first is a
weak linearity condition on ${\mathcal{H}}$:
\begin{itemize}
\item [(\textbf{i})]  $\forall h
  \in {\mathcal{H}}$, $\exists$ classifiers $h_\re, h_\tr$ over
  ${\mathcal{F}}_\re, {\mathcal{F}}_\tr$ s. t.
  $h(\bm{x}) = h_\re((\matrice{f}^\re)^\top \ve{x}) + h_\tr((\matrice{f}^\tr)^\top \ve{x})$.
\end{itemize}
$(\matrice{f}^\tr)^\top \ve{x}$ picks the features of
$\ve{x}$ in 
${\mathcal{F}}_\tr$. 
Such an assumption is also made in the feature bagging
model \citep{ssmRW}.
Any linear classifier satisfies (\textbf{i}), but also any linear
combination of arbitrary classifiers, each learnt over one of
${\mathcal{F}}_\re$ and ${\mathcal{F}}_\tr$. We let ${\mathcal{H}}_\tr$ denote the set of all $h_\tr$. The second assumption postulates
that key quantities are bounded \citep{bmRA}:
\begin{itemize}
\item [\textbf{(ii)}] $0\leq \varphi(z) \leq K_\varphi, \forall z$ and $|h_\tr((\matrice{f}^\tr)^\top \ve{x})| \leq
  K_\tr, \forall \ve{x} \in {\mathcal{X}}, \forall h_\tr \in {\mathcal{H}}_\tr$.
\end{itemize}

Let $\rad_{{\mathcal{S}}} ({\mathcal{H}}) \defeq \expect_{\ve{\sigma} \sim
   \Sigma_m} \left[ \sup_{h \in {\mathcal{H}}} \left| (1/m) \cdot \sum_i
{\sigma_i h \left(\ve{x}_i\right)}\right|\right]$ be the empirical
Rademacher complexity of ${\mathcal{H}}$. Additionally, we coin the
Rademacher \cp~complexity, \rcp.
\begin{definition}
The Rademacher \cp~complexity (\rcp) of ${\mathcal{H}}$ with
respect to $\process{T} \defeq \cp(\mathcal{S}; \matrice{f}^\re,
\matrice{f}^\tr, \matrice{m})$ is:
\begin{eqnarray}
\disc_{\process{T}} ({\mathcal{H}})& \defeq & \expect_{\ve{\sigma} \sim
   \Sigma_m} \left[ \sup_{h \in {\mathcal{H}}_\tr} \left|\frac{1}{m} \sum_{i}
{\sigma_{i} \left( h((\matrice{s}\matrice{f}^\tr)^\top \ve{1}_i) - h((\matrice{m}\matrice{s}\matrice{f}^\tr)^\top \ve{1}_i)\right)}\right|\right]\:\:.\label{defradd}
\end{eqnarray}
\end{definition}
Notice that the \rcp~is computed over the
shuffle set of features only, and
$(\matrice{s}\matrice{f}^\tr)^\top \ve{1}_i = (\matrice{f}^\tr)^\top \ve{x}_i$.
The next Theorem expresses a generalisation bound \emph{wrt} the \cp.

\begin{theorem}\label{rader}
Consider any ${\mathcal{H}}$, $\varphi$ and split ${\mathcal{F}} = {\mathcal{F}}_\re \cup
\mathcal{F}_\tr$ such that \textbf{(i)} and
\textbf{(ii)} hold. For any $m$ and any $\updelta>0$, with probability $\geq 1 - \updelta$ over i.i.d. $m$-sample
${\mathcal{S}}$, we have:
\begin{eqnarray}
\expect_{{\mathcal{D}}} \left[L_{0/1}(y, h(\bm{x}))\right] & \leq &
\expect_{{\mathcal{S}}^{\process{T}}} \left[\varphi(y
   h(\bm{x}))\right] + \disc_{\process{T}} ({\mathcal{H}}) + \frac{4}{b_\varphi} \cdot \rad_{{\mathcal{S}}} ({\mathcal{H}}) +  (2 K_\varphi + K_\tr) \cdot
\sqrt{\frac{2}{m}\log \frac{3}{\updelta}}\:\:, \nonumber
\end{eqnarray}
for every classifier $h$ and every $\process{T} \defeq  \cp(\mathcal{S}; \matrice{f}^\re,
\matrice{f}^\tr, \matrice{m})$ such that $\matrice{m} \in
{\mathcal{M}}^*_m$. Here, $b_\varphi > 0$ is a constant depending on $\varphi$.
\end{theorem}
(proof in Subsection \ref{proof_rader}) Notice that Theorem
\ref{rader} requires that $\matrice{m}$ is a block-class matrix. A key to
the proof is the invariance of the mean operator: $\ve{\mu}_{{\mathcal{S}}}
= \ve{\mu}_{{\mathcal{S}}^{\process{T}}}$. 
Theorem \ref{rader} says that a key to good generalisation is the
control of $\disc_{\process{T}}
({\mathcal{H}})$. We would typically want it to be small compared to
the Rademacher complexity penalty. The rest of this Section shows that (and when) this is
indeed achievable.\\

\noindent\textbf{Upperbounds on $\disc_{\process{T}} ({\mathcal{H}})$}
--- We consider different configurations of ${\mathcal{H}}$ and / or $\process{T}$:
\begin{itemize}
\item [] Setting (A): Classifiers $h^\tr$ and $h^\re$ in \textbf{(i)} above are linear;
\item []  Setting (B): $\matrice{m} \in
S^*_{m}$.
\end{itemize}
The following Lemma establishes a first bound on $\disc_{\process{T}} ({\mathcal{H}}) $.
\begin{lemma}\label{lemcrude}
if $\process{T}$ satisfies the conditions of Theorem \ref{rader}, then $\disc_{\process{T}} ({\mathcal{H}})  \leq 
2\cdot\rad_{{\mathcal{S}}'} ({\mathcal{H}}_\tr)$, for $\matrice{s}' \defeq
(\matrice{i}_m - \matrice{m})\matrice{s}\matrice{f}^\tr$ in Setting
(A), and $\matrice{s}' \defeq \matrice{s}\matrice{f}^\tr$ in Setting
(B). $\matrice{s}'$ is the row-wise observation matrix of $\mathcal{S}'$.
\end{lemma}
\begin{proof}
(Sketch) Consider for example Setting (B). In this case, recalling
that $(\matrice{s}\matrice{f}^\tr)^\top \ve{1}_i =
(\matrice{f}^\tr)^\top \ve{x}_i$ and letting $\varsigma : [m] \rightarrow [m]$
denote the permutation that $\matrice{m}$ represents, we have because
of the triangle inequality:
\begin{eqnarray}
\disc_{\process{T}} ({\mathcal{H}})& = & \expect_{\ve{\sigma} \sim
   \Sigma_m} \left[ \sup_{h \in {\mathcal{H}}_\tr} \left|\frac{1}{m} \sum_{i}
{\sigma_{i} \left( h((\matrice{f}^\tr)^\top \ve{x}_i) -
    h((\matrice{f}^\tr)^\top
    \ve{x}_{\varsigma (i)})\right)}\right|\right]\label{ineqc}\\
 & \leq & \expect_{\ve{\sigma} \sim
   \Sigma_m} \left[ \sup_{h \in {\mathcal{H}}_\tr} \left|\frac{1}{m} \sum_{i}
\sigma_{i} h((\matrice{f}^\tr)^\top \ve{x}_i)\right|\right] + \expect_{\ve{\sigma} \sim
   \Sigma_m} \left[ \sup_{h \in {\mathcal{H}}_\tr} \left|\frac{1}{m} \sum_{i}
\sigma_{i} h((\matrice{f}^\tr)^\top \ve{x}_{\varsigma (i)})\right|\right] \nonumber\\
 & & = 2\cdot\rad_{{\mathcal{S}}'} ({\mathcal{H}}_\tr)\nonumber\:\:,
\end{eqnarray}
as claimed. The case of Setting (A) follows the same path.
\end{proof}
Lemma \ref{lemcrude} says that $\disc_{\process{T}} ({\mathcal{H}})$
is at most twice a Rademacher complexity over the \textit{shuffle set}. This bound is however
loose since many terms
can cancel in the sum of eq. (\ref{ineqc}), and the inequality
does not take this into account. In particular,

\begin{theorem}\label{thradcompUU}
Under Setting (A), suppose any $h_\tr$ is of the form $h_\tr(\ve{x}) =
\ve{\theta}^\top \ve{x}$ with $\|\ve{\theta}\|_2 \leq
r_\tr$, for some $r_\tr>0$. Let $\matrice{k}^\tr\defeq \matrice{s}\matrice{f}^\tr (\matrice{s}\matrice{f}^\tr)^\top $. Then $\exists u\in (0,1)$ depending
only in ${\mathcal{S}}$ such that for any $\matrice{m} \in {\mathcal{M}}_m$, 
\begin{eqnarray}
\disc_{\process{T}} ({\mathcal{H}}) & \leq & (u r_\tr /m)\cdot
\sqrt{
  \cipr{\matrice{i}_m}{\matrice{k}^\tr}{\matrice{m}}} \:\:.\label{feq1}
 \end{eqnarray}
\end{theorem}
Notice that $\matrice{k}^\tr$ is a Gram matrix in the shuffle feature space.
The proof technique (Subsection
\ref{proof_thm_thradcompUU}) relies on a data-dependent
expression for $u$ which depends on the cosines of angles between the
observations in ${\mathcal{S}}$. It can be used to refine
and improve a popular bound on
the empirical Rademacher complexity of linear classifiers \citep{kstOT}
(we give the proof in Theorem \ref{improveRC} in the Appendix). 
We
now investigate an upperbound on Setting (B) in which classifiers in ${\mathcal{H}}^\tr$
are (rooted) directed acyclic graph (\dagr), like decision
trees, with bounded real valued predictions (say, $K_\tr>0$) at the leaves. 
Each classifier $h_\tr$ defines a partition over ${\mathcal{X}}$. We
let 
${\mathcal{H}}^\tr_+$ be the subset of ${\mathcal{H}}^\tr$ in which
all leaves have in absolute value the
largest magnitude, \textit{i.e.}, $K_\tr$. Remark that we may have $|{\mathcal{H}}^\tr_+| \ll \infty$ while $|{\mathcal{H}}^\tr| = \infty$ in general. 
\begin{theorem}\label{thradcompSettingB}
Under Setting (B), suppose ${\mathcal{H}}^\tr$ is \dagr~and assumption \textbf{(ii)}
is satisfied. Suppose that $\log |{\mathcal{H}}^\tr_+| \geq (4 \varepsilon/3) \cdot 
  m$ for some $\varepsilon > 0$. Then, letting $\oddcycle(\matrice{m})$ denote the set of odd cycles (excluding fixed
points) of $\matrice{m}$, we have:
\begin{eqnarray}
\disc_{\process{T}} ({\mathcal{H}})  &
\leq & K_\tr \cdot \sqrt{ \frac{2}{m} \cdot \log \frac{|{\mathcal{H}}^\tr_+|}{(1+\varepsilon)^{|\oddcycle(\tinymatrice{m})| }} }\:\:.\label{rddp}
\end{eqnarray}

\end{theorem}
(proof in Subsection \ref{proof_thm_thradcompSettingB}) The
assumption on ${\mathcal{H}}^\tr_+$ is not restrictive and would be
met by decision trees, branching programs, etc. (and 
subsets). Usual bounds on the Rademacher complexity of decision trees would roughly be
the right-hand side of (\ref{rddp}) \textit{without} the denominator
in the $\log$ (see for example \cite[Chapter 5]{sfBF}). Hence, the \rcp~may be significantly smaller
than the Rademacher complexity for more ``involved'' \cp s. The number of
cycles is not the only relevant parameter of the \cp~on which relies
non-trivial bounds on $\disc_{\process{T}} ({\mathcal{H}})$: the
Appendix presents, for the interested reader, a
proof that the number of fixed points is another parameter which can
decrease significantly the \textit{expected} \rcp~ (by a
factor $\sqrt{1 - |\mathrm{fixed}\_\mathrm{points}|/m}$), when
\cp s are picked at random (see Theorem \ref{thexpradcomp} and
discussion in Subsection \ref{proof_thm_thradcompSettingB}).

At last, we notice that Theorem \ref{rader} gives a perhaps counterintuitive rationale for the \cp~that goes beyond
our framework to machine learning at large: \textit{learning over a \cp'ed
${\mathcal{S}}$ may improve generalisation over $\mathcal{D}$ as
well}. By means of words, learning over transformed data may improve
generalisation over the initial domain. Figure \ref{f-tdp} (bottom) gives a toy example for which this
holds. It is also not hard to exhibit domains for which we even have:
\begin{eqnarray}
\min_h \expect_{{\mathcal{S}}^{\process{T}}} \left[\varphi(y
   h(\bm{x}))\right] + \disc_{\process{T}} ({\mathcal{H}}) & < & \min_h \expect_{{\mathcal{S}}} \left[\varphi(y
   h(\bm{x}))\right]\:\:.
\end{eqnarray}
In order not to load the paper's body, we present such an example in the Appendix (Subsection \ref{ex_domain}). Having discussed learning guarantees, we are ready
to dive into more applications of the Crossover Process for data protection. In addition to the
measures of (un)fairness developed in Section \ref{sec-dsp-fair}, we
develop the \cp~in two other
frameworks: Hilbert-Schmidt independence and \emph{do}
calculus.
In the former one, our results exploit the design of the
shuffle matrix $\matrice{m}$ to alter independence; in the latter one, our results exploit
the split step of the \cp~to interfere with causal inference. 

\section{The Crossover Process and statistical independence}\label{sec-statind}

Here, we assume that ${\mathcal{S}}$ is subject to
quantitative tests of independence, that is, assessing ${\mathcal{U}}
\indep {\mathcal{V}}$ for some ${\mathcal{U}}, {\mathcal{V}}\subset
\mathcal{X}$. We compute \cp s such that ${\mathcal{U}} \subseteq {\mathcal{F}}_\re$
and ${\mathcal{V}} \subseteq {\mathcal{F}}_\tr$, so that the
\cp~alters the measure of independence. 
One popular criterion to determine (conditional) (in)dependence
is
Hilbert-Schmidt Independence Criterion \citep{dmzsAP,gbssMS,gftsssAK}.
\begin{definition}\label{defhs}
Let ${\mathcal{U}} \subset [d]$ and ${\mathcal{V}} \subset [d]$ be
non-empty and disjoint. Let $\matrice{k}^u$ and $\matrice{k}^v$ be two kernel functions
over ${\mathcal{U}}$ and ${\mathcal{V}}$ computed using ${\mathcal{S}}$.
The (unnormalised) Hilbert-Schmidt Independence Criterion (\hsic)
between ${\mathcal{U}}$ and ${\mathcal{V}}$ is defined as
$\hsic(\matrice{k}^u,\matrice{k}^v) \defeq
\cipr{\matrice{k}^u}{\matrice{k}^v}{(1/m) \ve{1}\ve{1}^\top}$.
\end{definition}
(We choose not to normalise the \hsic: various exist but
they mainly rely on a multiplicative factor depending on $m$ only, so they do not affect the results to
come.) The choice of $\matrice{m}$ in the \cp~directly influences the value the result of \hsic; therefore,
we can design a search strategy aimed to alter it. 
In the same way as we did for Section \ref{sec-dsp-fair}, we provide a
Definition for the alteration in the \hsic~criterion caused by
a \cp, and then a Theorem that quantifies it precisely for any \cp.
\begin{definition}
We say that a \cp~$\process{T}$ shifts the \hsic~criterion between
${\mathcal{U}}$ and ${\mathcal{V}}$ by $\Delta$ (on dataset
${\mathcal{S}}$) iff $\hsic_{\process{T}} -
\hsic = \Delta$, where $\hsic$ and $\hsic_{\process{T}}$ respectively denote the
\hsic~before and after applying the \cp.
\end{definition}
\begin{theorem}\label{thspectrA}
For any dataset $\mathcal{S}$, let $\tilde{\ve{u}} \defeq (1/m) \sum_i \lambda_i (\ve{1}^\top
\ve{u}_i) \ve{u}_i$, $\tilde{\ve{v}} \defeq (1/m) \sum_i \mu_i (\ve{1}^\top
\ve{v}_i) \ve{v}_i$, where $\{\lambda_i, \ve{u}_i\}_{i\in [d]}$,
$\{\mu_i, \ve{v}_i\}_{i\in [d]}$ are respective eigensystems of
$\matrice{k}^u$ and $\matrice{k}^v$. Then for \textbf{any}
\cp~$\process{T}$, $\process{T}$ shifts the \hsic~criterion between
${\mathcal{U}}$ and ${\mathcal{V}}$ by $\Delta$ (on
${\mathcal{S}}$) with
\begin{eqnarray}
\Delta & = & 2 m \cdot \tilde{\ve{u}}^\top (\matrice{i}_m -
\matrice{m}) \tilde{\ve{v}}\:\:,
\end{eqnarray}
where $\matrice{m}$ is the shuffling matrix of the \cp.
\end{theorem}
(proof in Subsection \ref{proof_thm_thspectrA}) 
This result shows that altering the \hsic~criterion with a \cp~is also algorithmic friendly:
(a) while storing kernels requires $O(m^2)$ space,
controlling the evolution of the $\hsic$ requires only \textit{linear}-space
information about kernels, and (b) this information can be computed beforehand, and
can be efficiently approximated from low-rank approximations
of the kernels \citep{bSA}. Theorem \ref{thspectrA} also shows that
the sign of the shift $\Delta$ is determined by the eigenspace of the
shuffling matrix (when it can be diagonalized). 

We now go one step further in showing the algorithmic friendliness of
computing a \cp~that shifts \hsic. In
the following Theorem, we \textit{compose} \cp~processes with $T$ different
elementary permutation shuffling matrices. Notice
that since the composition of permutation matrices is a permutation
matrix, when the matrix of the final process $\hsic_{{\process{T}}_T}$
is block class, Theorem
\ref{rader} can be applied \textit{directly} to
$\hsic_{{\process{T}}_T}$, so we get, at each iteration, a quantification
of the protection achieved (shift $\Delta$) and the impact on learnability. We
let $\mathscr{R}^{u,v} \defeq m\left(1 - (\matrice{k}^u_{..} + \matrice{k}^v_{..})/(2m^2)\right)$.
\begin{theorem}\label{thspectrB}
Suppose $\process{T}$ is built by a sequence of $T = \epsilon m$
elementary permutation ($\epsilon > 0$) and the
kernels $\matrice{k}^u$ and $\matrice{k}^v$ have unit
diagonal.
Suppose that the initial $\hsic > \mathscr{R}^{u,v}$ (before applying
the \cp). Then there
exists such a sequence of elementary permutations such that
${\process{T}}_T$ shifts the \hsic~criterion between
${\mathcal{U}}$ and ${\mathcal{V}}$ by $\Delta$ (on
${\mathcal{S}}$) with
\begin{eqnarray*}
\Delta& \leq & -(1-\alpha)\cdot(\hsic - \mathscr{R}^{u,v})\:\:,
\end{eqnarray*}
where $\alpha \defeq \exp(-8\epsilon) < 1$.
\end{theorem}
The proof (Subsection \ref{proof_thm_thspectrB}) states a
more general result, not restricted to unit diagonal kernels. Theorem
\ref{thspectrB} is a worst-case result that shows sufficient
conditions for negative shift, and therefore decrease the
\hsic~criterion. Note that some sequences of permutations
may be much more efficient in decreasing \hsic. If we compare this bound to Theorem 3
in \citep{gftsssAK}, then $\mathscr{R}^{u,v}$ may be \textit{below} the
expectation of the \hsic, so indeed we can obtain $\Delta<0$ and somehow
trick statistical tests in keeping independence after the \cp, while
they would eventually reject it before. 

\textbf{Remark}: it is in fact possible to kill two birds with one
stone, namely trick statistical tests into keeping independence \textit{and}
then incur arbitrarily large errors in estimating causal effects. Our
basis is the Cornia-Mooij (CM) model \citep{cmT2} which, for
space considerations, we defer to the Appendix (Subsection \ref{pres_dcm}). One interesting feature
of this particular causal graph is that it is so simple that it may be found as subgraph
  of real-world domains, thus for which the results we
  give would directly transfer.

\section{The Crossover Process and causality}\label{sec-causal}

We assume here basic knowledge of the causality and \textit{do}-calculus
frameworks \cite{pCM}. We consider the case where the causal directed acyclic graph is known, and the goal is to interfere with the inference of causal effects between covariates. The key challenge for causal inference is the existence of confounding variables that are causes of both the exposure and outcome variables. For example, suppose impact of hormone replacement therapy on women's health was captured by the causal DAG $I \rightarrow T \rightarrow H, I \rightarrow H$, where $I$ represents income, $T$ represents taking the treatment and $H$ is the health outcome of interest. If wealthier women are more likely to see a doctor for treatment and also have generally better health, then $\Pr[H|T]$ will be more positive than the true causal effect $\Pr[H|do(T)]$.   

Adjusting for such nuisance or confounding variables, either by matching \citep{greenwood1945experimental,rubin1973matching} or regression \citep{fisher1935design}, is a central tool in economics and social sciences \citep{morgan2014counterfactuals}. The back-door criterion \citep{pCM} clarifies which variables it is appropriate to condition on in order to achieve unbiased estimates of causal effects. The \cp{} can be designed to interfere with obtaining causal estimates via such adjustments.

Let $G = ({\mathcal{U}\cup \mathcal{F}},\mathcal{A})$ be a causal directed acyclic graph over observable vertices $\mathcal{F}$, latent variables $\mathcal{U}$ and arcs $\mathcal{A}$ \citep{pCM}. We are given a set ${\mathcal{Q}} \defeq
\{(\y, \x), i \in [q]\}$ of $q$ causal queries, each of which
represents the estimation of $\Pr[\y|\mathrm{do}(\x)]$. 

A covariate adjustment (adjustment for short) for a query $(\y,\x)$ is a set ${\mathcal{Z}}_i \subset {\mathcal{F}}$ such that $\x, \y \not\in
{\mathcal{Z}}_i$ and
\begin{eqnarray}
\Pr[\y|\mathrm{do}(\x)] & = & \sum_{z \sim{\mathcal{Z}}_i} \Pr[\y |
\x, z] \Pr[z]\:\:,\label{adjeq}
\end{eqnarray}
An adjustment is not guaranteed to exist. In our example, for the query $(H,T)$, there is no adjustment if $I \in \mathcal{U}$. An adjustment is minimal iff it does not contain any other adjustment as proper subset. Note that ${\mathcal{Z}}_i$ can be the empty set.



In the same way as we did for Sections \ref{sec-dsp-fair} and
\ref{sec-statind}, we provide a
Definition for the alteration in the \textit{do}-calculus caused by
a \cp, and then a Lemma that qualifies it more precisely for any
\cp. It is a weaker result than the former ones because we do not
quantify a shift, but rather just state changes in estimation.
\begin{definition}
We say that a \cp~$\process{T}$ \textit{interferes} with causal
inference via adjustment for a query $(\y, \x)$ if the solution to (\ref{adjeq})
differs between the datasets $\mathcal{S}$ and
${\mathcal{S}}^{\process{T}}$.
\end{definition}
We put no constraint on the magnitude of the change, so interfering with causal queries is essentially a matter of biasing the distributions involved in the right-hand side of (\ref{adjeq}). Let $\Z_i$ denote the set of minimal adjustments for query $(\y, \x)$.
\begin{lemma}\label{lemsplit} Let $\mathcal{V}_i =  \x \cup \y \cup
  \mathcal{Z}_i$. A \cp~interferes with causal inference via
  adjustment for the query $(\y, \x)$ iff $\, \forall {\mathcal{Z}}_i
  \in \Z_i$, $\exists$ variables $v_\re,v_\tr \in \mathcal{V}_i \text{
    such that } v_\re \in {\mathcal{F}}_\re$ and $ v_\tr \in
  {\mathcal{F}}_\tr$, where ${\mathcal{F}}_\re$ and
  ${\mathcal{F}}_\tr$ are the shuffle and anchor set of features of
  the \cp.
\end{lemma}
The proof is a direct consequence of eq. (\ref{adjeq}) and the fact
that the shuffle matrix $\matrice{m}$ alters joint distributions
between variables that do not belong to the same split set, 
without touching marginals. We can always interfere with a single query $(\y, \x)$ by ensuring $\y$ and $\x$ are in different splits. To simultaneously interfere with the set of queries ${\mathcal{Q}}$, we must first find the set of minimal adjustments $\Z_i, \forall i \in [q]$, then select a split that satisfies Lemma \ref{lemsplit} for every query. This involves heavy combinatorics. Enumerating the adjustments can be done with cubic \textit{delay} per adjustment \citep{tlAC}(the set of minimal adjustments for a given query can grow exponentially with $d$). The second step subsumes the infamous Set Splitting problem \citep{gjCA}, which is $NP$-Complete. In practice, causal graphs must often be constructed by humans so this approach can still be computationally feasible for a small set of queries. Exploring the addition of constraints on the graph (such as sparsity) to develop more efficient algorithms is an interesting avenue for future research.

A causal query is \textit{identifiable} if we can obtain an expression
for it purely in terms of distributions over the observable variables
$\mathcal{F}$. The existence of an adjustment is sufficient but not
necessary for identifiability. In theory, a causal query for which we
have interfered with any adjustments, could still be identified via
another approach. The \textit{do}-calculus \citep{pCM} and Identify Algorithm \citep{shpitser2006identification} provide a complete framework for determining if a query is identifiable and computing an expression for it. In principle, we could utilize the \cp{} to interfere with all routes to identifiability. This would require an algorithm that could enumerate the expressions for a causal query. We are not aware of such an algorithm in the literature. In practice, expressions that are not of the form of \ref{adjeq} are rarely used.

\section{Discussion and conclusion}

This paper introduces the Crossover Process (\cp), a mechanism that
cross-modifies data using a generalisation of stochastic
matrices. This process can be used to cope with data optimisation for
supervised learning, as well as for the problem of handling a
process-level protection on data such as causal
inference attacks on a supervised learning dataset. In this case, the \cp~allows to release
data with spotless low-level description (variable names, observed
values, marginals), substantial
utility (learnability), but disclosing dependences
and causal effects under control, and thus that could
 even be crafted to be conflicting with a ground truth to
 protect\footnote{Note that the initial data may not be lost, as opposed to
 differential privacy: knowing the noise parameters does not
   allow to revert differential privacy protection, while a
   \cp~protection
   is reversible when the shuffling matrix $\matrice{m}$ is invertible.}. We have chosen to focus here on three major components of the
actual trends, namely (un)fairness measures, statistical measures of causal inference and causal
queries. In these directions, there are some very interesting and non-trivial avenues
for future research, like for example the control of combinatorial
blow-up in the worst case for causal queries, ideally as a function of
the causal graph structure. There are also more applications of the \cp~in the field of causal discovery.
Suppose for example that description features denote transactions. Since we modify
joint distributions without touching on marginals, our technique has
direct applications in causal rule mining, with the potential to fool
any level-wise association rule mining algorithms, that is, any spawn of Apriori 
\citep{lllljsmFO}.

The theory we develop fo \cp~introduces a new complexity measure of
the process, the Rademacher
\cp~complexity. We do believe that the \cp~is also a good contender in the pool of methods optimising data for
learning, and it
may provide new metrics, algorithms and tools to devise improved
solutions that fit 
to challenging domains not restricted to optimizing learning or data privacy.

{
\bibliography{bibgen}
}

\section{Proofs}\label{proof_proofs}

We shall use the following notations and shorthands. We shall sometimes replace notation
  $\ve{x}_i$ by $\ve{x}^b_i$ for $b\in \{-,+\}$, indicating explicitly an
  observation from class $b1$. We also let $m_b$ denote the number of examples in class $b1$
  for $b\in \{-,+\}$ ($m = m_+ + m_-$). Furthermore, $[m]_{m'} \defeq \{m'+ i:
  i\in [m]\}$ ($m \in {\mathbb{N}}_*, m' \in {\mathbb{N}}$). Matrix
  $\matrice{u}_m \defeq (1/m) \ve{1}\ve{1}^\top$ denotes the uniform
  Markov chain.

\subsection{Proof of Lemma \ref{lemmodds1}}\label{proof_lemmodds1}

Consider \cp~$\process{T}$ in which the anchor contains $x_C$ and the variables of
$\ve{\pi}$, and the shuffle set contains $x_A$. Other features can be
split arbitrarily. Let the shuffle matrix $\matrice{m} \in S_n$ be any
permutation matrix that transfers $\upomega$ observations for which $x_C
= 1$ and $x_A = 1$ to observations for which $x_C
= 1$ and $x_A = 0$, thereby also moving $\upomega$ observations for which $x_C
= 0$ and $x_A = 0$ to observations for which $x_C
= 0$ and $x_A = 1$. The resulting contingency table is therefore:
\begin{center}
\begin{tabular}{c||c|c||}
 & $x_C = 0$ & $x_C = 1$\\ \hline \hline
 $x_A = 0$ & $a-\upomega$ & $b+\upomega$\\ \hline
 $x_A = 1$ & $c+\upomega$& $d-\upomega$\\ \hline \hline
\end{tabular}
\end{center}
Plus, we get the additional constraints that $\upomega \leq d$ and
$\upomega \leq a$. But we can also transfer $\upomega'$ in the opposite
direction, in which case
the resulting contingency table is therefore:
\begin{center}
\begin{tabular}{c||c|c||}
 & $x_C = 0$ & $x_C = 1$\\ \hline \hline
 $x_A = 0$ & $a+\upomega'$ & $b-\upomega'$\\ \hline
 $x_A = 1$ & $c-\upomega'$& $d+\upomega'$\\ \hline \hline
\end{tabular}
\end{center}
This time, we get the additional constraints that $\upomega' \leq b$ and
$\upomega' \leq c$. The corresponding odds ratio satisfy:
\begin{eqnarray}
\rho(x_C,x_A,\ve{\pi}|\mathcal{S}^{\process{T}}) & = &
\frac{b+\upomega}{d-\upomega} \in \left[\frac{b}{d}, \frac{b+\min\{a,d\}}{d-\min\{a,d\}}\right]\:\:,\\
\rho(x_C,x_A,\ve{\pi}|\mathcal{S}^{\process{T}}) & = &
\frac{b-\upomega'}{d+\upomega'} \in \left[\frac{b-\min\{b,c\}}{d+\min\{b,c\}}, \frac{b}{d}\right]\:\:.
\end{eqnarray}
Therefore, there always exists a \cp~${\process{T}}$ bringing $\rho$-odds for the triple
$(x_A,x_C,\ve{\pi})$ in \textit{modified} sample
${\mathcal{S}}^{\process{T}}$,
for 
\begin{eqnarray}
\rho & = & \frac{b+ i}{d-i}\:\:, \forall i \in \left\{-\min\{b,c\} , -\min\{b,c\}+1, ..., \min\{a,d\} \right\}\:\:.
\end{eqnarray}
There just remain to compute the difference in odds ratios,
\begin{eqnarray}
\Delta(i) & \defeq & \rho(x_C,x_A,\ve{\pi}|\mathcal{S}^{\process{T}})
- \rho(x_C,x_A,\ve{\pi}|\mathcal{S})\nonumber\\
 & = & \frac{b+ i}{d-i} - \frac{b}{d} \nonumber\\
 & = & \frac{(b+d)}{d-i}\cdot\frac{i}{d}\:\:,
\end{eqnarray}
as claimed.

\subsection{Proof of Theorem \ref{rader}}\label{proof_rader}

The first steps of the proof are the same as \citep{bmRA} (Theorems 5,
8). We sketch them. First, for any \cp~$\process{T}$,
\begin{eqnarray}
\expect_{{\mathcal{D}}} \left[L_{0/1}(y, h(\bm{x}))\right] & \leq &
\expect_{{\mathcal{D}}} \left[\varphi(y h(\bm{x}))\right]\nonumber\\
 & \leq & \expect_{{\mathcal{S}}^{\process{T}}} \left[\varphi(y
   h(\bm{x}))\right] + \sup_{h\in {\mathcal{H}}} \left\{ \expect_{{\mathcal{D}}}
   \left[\varphi(y h(\bm{x}))\right] - \expect_{{\mathcal{S}}^{\process{T}}} \left[\varphi(y
   h(\bm{x}))\right]\right\}\:\:.
\end{eqnarray}
Then, since $\varphi(z) \in [0,K_\varphi]$ (assumption \textbf{(ii)}), the use of the independent
bounded differences inequality \citep{mdC} yield for any sample ${\mathcal{S}}$ sampled from ${\mathcal{D}}$, and
any $\varsigma \in S_m$, and any $\updelta_1$, we have with probability $\geq 1 - \updelta_1$:
\begin{eqnarray}
\lefteqn{\sup_{h\in {\mathcal{H}}} \left\{ \expect_{{\mathcal{D}}}
   \left[\varphi(yh(\bm{x}))\right] - \expect_{{\mathcal{S}}^{\process{T}}} \left[\varphi(y
   h(\bm{x}))\right]\right\}}\nonumber\\
 & \leq & \expect_{{\mathcal{S}} \sim {\mathcal{D}}}\left[ \sup_{h\in {\mathcal{H}}} \left\{ \expect_{{\mathcal{D}}}
   \left[\varphi(y h(\bm{x}))\right] - \expect_{{\mathcal{S}}^{\process{T}}} \left[\varphi(y
   h(\bm{x}))\right]\right\}\right] + K_\varphi\cdot \sqrt{\frac{2}{m}
\log\frac{1}{\updelta_1}}\:\:. \label{bdelta1}
\end{eqnarray}
We also have, because of the convexity of $\sup$ (see \citep{bmRA}),
\begin{eqnarray}
\lefteqn{\expect_{{\mathcal{S}} \sim {\mathcal{D}}}\left[ \sup_{h\in {\mathcal{H}}} \left\{ \expect_{{\mathcal{D}}}
   \left[\varphi(y h(\bm{x}))\right] - \expect_{{\mathcal{S}}^{\process{T}}} \left[\varphi(y
   h(\bm{x}))\right]\right\}\right]}\nonumber\\
 & \leq & \expect_{{\mathcal{S}}, {\mathcal{S}}' \sim {\mathcal{D}}}\left[ \sup_{h\in {\mathcal{H}}} \left\{ \expect_{{\mathcal{S}}'}
   \left[\varphi(y h(\bm{x}))\right] - \expect_{{\mathcal{S}}^{\process{T}}} \left[\varphi(y
   h(\bm{x}))\right]\right\}\right] \nonumber\:\:.
\end{eqnarray}
The proof now takes a fork compared to \citep{bmRA}, as we integrate new steps to upperbound the right-hand side. We split the right supremum in two, one which involves different
datasets of size $m$ not being subject to \cp, and one which involves the same
dataset with and without \cp:
\begin{eqnarray}
\lefteqn{\expect_{{\mathcal{S}}, {\mathcal{S}}' \sim {\mathcal{D}}}\left[ \sup_{h\in {\mathcal{H}}} \left\{ \expect_{{\mathcal{S}}'}
   \left[\varphi(y h(\bm{x}))\right] - \expect_{{\mathcal{S}}^{\process{T}}} \left[\varphi(y
   h(\bm{x}))\right]\right\}\right]}\nonumber\\
 & = & \expect_{{\mathcal{S}}, {\mathcal{S}}' \sim {\mathcal{D}}}\left[ \sup_{h\in {\mathcal{H}}} \left\{ \left(\expect_{{\mathcal{S}}'}
   \left[\varphi(y h(\bm{x}))\right] - \expect_{{\mathcal{S}}} \left[\varphi(y
   h(\bm{x}))\right]\right) + \left(\expect_{{\mathcal{S}}} \left[\varphi(y
   h(\bm{x}))\right] - \expect_{{\mathcal{S}}^{\process{T}}} \left[\varphi(y
   h(\bm{x}))\right]\right)\right\}\right]\nonumber\\
 & \leq & \underbrace{\expect_{{\mathcal{S}}, {\mathcal{S}}' \sim {\mathcal{D}}}\left[ \sup_{h\in {\mathcal{H}}} \left\{ \expect_{{\mathcal{S}}'}
   \left[\varphi(y h(\bm{x}))\right] - \expect_{{\mathcal{S}}} \left[\varphi(y
   h(\bm{x}))\right]\right\}\right]}_{\defeq A} \nonumber\\
 & & + \underbrace{\expect_{{\mathcal{S}} \sim {\mathcal{D}}}\left[ \sup_{h\in {\mathcal{H}}} \left\{ \expect_{{\mathcal{S}}} \left[\varphi(y
   h(\bm{x}))\right] - \expect_{{\mathcal{S}}^{\process{T}}} \left[\varphi(y
   h(\bm{x}))\right]\right\}\right]}_{\defeq B}\:\:.\label{ss1}
\end{eqnarray}
We handle $A$ and $B$ separately.\\

\noindent \textbf{Upperbound on $A$}. Handling $A$ is achieved in the
usual way \citep{bmRA}. Following the usual symmetrisation trick \citep{bmRA} (Theorem 8), and
the fact \citep{bmRA} (Theorem 12.4) that $\varphi$ is
$1/b_\varphi$-Lipschitz \citep{nnOT} for some $b_\varphi > 0$, we obtain that with probability $\geq 1
- \delta_1$, we have:
\begin{eqnarray}
\lefteqn{\expect_{{\mathcal{S}}, {\mathcal{S}}' \sim {\mathcal{D}}}\left[ \sup_{h\in {\mathcal{H}}} \left\{ \expect_{{\mathcal{S}}'}
   \left[\varphi(y h(\bm{x}))\right] - \expect_{{\mathcal{S}}} \left[\varphi(y
   h(\bm{x}))\right]\right\}\right]}\nonumber\\
 & \leq & \frac{4}{b_\varphi} \expect_{\ve{\sigma} \sim
   \Sigma_m} \left[ \sup_{h\in {\mathcal{H}}} \left|\frac{1}{m} \sum_i
{\sigma_i h \left(\ve{x}_i\right)}\right|\right] + K_\varphi \cdot
\sqrt{\frac{2}{m}\log \frac{1}{\updelta_1}}\:\:,\label{bdelta1_2}
\end{eqnarray}
$\forall \updelta_1 > 0$.\\

\noindent \textbf{Upperbound on $B$}. This penalty appears when
$\matrice{m}\neq \matrice{i}_m$. The trick is because $\varphi$ is
proper symmetric, there is a simple way to make appear Rademacher
variables and a particular Rademacher complexity, which follows from the fact that \citep{pnrcAN}:
\begin{eqnarray}
\expect_{{\mathcal{S}}} \left[\varphi(y
   h(\bm{x}))\right] & = & \frac{b_\varphi}{2m} \sum_{\sigma \in \Sigma_1} {\sum_i
{\varphi(\sigma h(\bm{x}_i))}}  - \frac{\overline{h}\left({\mathcal{S}}\right)}{2} \:\:,\label{defrel}
\end{eqnarray}
where 
\begin{eqnarray}
\overline{h}\left({\mathcal{S}}\right) & \defeq & \frac{1}{m}\cdot \sum_i
{y_i h(\bm{x}_i)}\label{defhmeano}
\end{eqnarray} 
is the \textit{$h$-mean-operator}, a statistics which can be proven to be
minimally sufficient for classes given $h$ \citep{pnrcAN}. Now, assumption \textbf{(i)} yields the invariance of $\overline{h}\left({\mathcal{S}}\right)$
under permutation operation. This is proved in the following Lemma.
\begin{lemma}(mean-operator consistency of $\process{T}$)\label{lemint1}
Under the conditions of Theorem \ref{rader},
\begin{eqnarray}
\overline{h}\left({\mathcal{S}}\right) & = &
\overline{h}\left({\mathcal{S}}^{\process{T}}\right) \:\:, \forall h,
\forall {\mathcal{S}}, \forall {\process{T}}\:\:.\label{pmeano}
\end{eqnarray}
\end{lemma}
\begin{proof}
To prove it, we let $\oplus$ denote the vector concatenation operation
over the features of ${\mathcal{F}}_\re$ and ${\mathcal{F}}_\tr$. We
now first write using Assumption \textbf{(i)}:
\begin{eqnarray}
m \cdot \overline{h}\left({\mathcal{S}}\right) & = & \sum_i
{y_i h(\ve{x}_i)} \nonumber\\
& = & \sum_{i}
{y_i \cdot h((\matrice{s}\matrice{f}^\re)^\top \ve{1}_i \oplus (\matrice{s}\matrice{f}^\tr)^\top \ve{1}_i)} \nonumber\\
& = & \sum_{i}
{y_i \cdot h_\re((\matrice{s}\matrice{f}^\re)^\top \ve{1}_i)} + \sum_{i}
{y_i \cdot h_\tr((\matrice{s}\matrice{f}^\tr)^\top \ve{1}_i)} \label{meano1}\:\:,
\end{eqnarray}
and also, for the same reason, 
\begin{eqnarray}
m \cdot \overline{h}\left({\mathcal{S}}^{\process{T}}\right) & = & \sum_{i}
{y_i \cdot h_\re((\matrice{s}\matrice{f}^\re)^\top \ve{1}_i)} + \sum_{i}
{y_i \cdot h_\tr((\matrice{m}\matrice{s}\matrice{f}^\tr)^\top \ve{1}_i)} \label{meano2}\:\:.
\end{eqnarray}
Therefore, we need to prove an equality that depends only upon the
features of ${\mathcal{F}}_\tr$:
\begin{eqnarray}
\sum_{i}
{y_i \cdot h_\tr((\matrice{s}\matrice{f}^\tr)^\top \ve{1}_i)} & = & \sum_{i}
{y_i \cdot h_\tr((\matrice{m}\matrice{s}\matrice{f}^\tr)^\top \ve{1}_i)} \label{meano3}\:\:.
\end{eqnarray}
We have two
cases to consider to prove eq. (\ref{meano3}). Notice that since it is
a block-class matrix, $\matrice{m}$ admits the following block matrix
decomposition:
\begin{eqnarray}
\matrice{m} & = & \left[
\begin{array}{ccc}
\matrice{m}_+ & |  & \matrice{0}\\
\matrice{0} & |  & \matrice{m}_-
\end{array}
\right]\:\:,\label{blockm}
\end{eqnarray}
with $\matrice{m}_b \in {\mathbb{R}}^{m_b\times m_b}$. We distinguish two cases.

\noindent Case 1 --- Setting (A). In this case,
\begin{eqnarray}
 \sum_{i\in [m]}
{y_i \cdot h_\tr((\matrice{m}\matrice{s}\matrice{f}^\tr)^\top
  \ve{1}_i)}  & = & \sum_{i\in [m]}
{y_i \cdot h_\tr\left(\bigoplus_{k\in [d_\tr]_{d_\re}} \sum_{l\in [m]}
    \matrice{m}_{il} \matrice{s}_{lk}\right)}\nonumber\\
 & = & \sum_{i\in [m]}
{y_i \cdot h_\tr\left(\sum_{l\in [m]}
    \matrice{m}_{il} \bigoplus_{k\in [d_\tr]_{d_\re}} \matrice{s}_{lk}\right)}\nonumber\\
& = & \sum_{i\in [m]}
{y_i \cdot h_\tr\left( \sum_{l\in [m]}
    \matrice{m}_{il} (\matrice{s}\matrice{f}^\tr)^\top\ve{1}_l\right)}\nonumber\\
& = & \sum_{i\in [m]}
{y_i \sum_{l\in [m]}
    \matrice{m}_{il}  \cdot
    h_\tr\left((\matrice{s}\matrice{f}^\tr)^\top\ve{1}_l\right)}\label{hlin}\\
& = & \sum_{l\in [m]}
{\left(\sum_{i\in [m]}
    \matrice{m}_{il}\right)  y_l \cdot h_\tr\left((\matrice{s}\matrice{f}^\tr)^\top\ve{1}_l\right)}\label{classcond}\\
& = & \sum_{l\in [m]}
{ y_l \cdot h_\tr\left((\matrice{s}\matrice{f}^\tr)^\top\ve{1}_l\right)}\label{mark}\:\:.
\end{eqnarray}
Here, $\oplus$ denotes the concatenation operator.
Eq. (\ref{hlin}) holds because of Setting (A), eq. (\ref{classcond})
holds because of the class-consistency assumption (eq. (\ref{blockm})), and eq. (\ref{mark}) holds
because $\matrice{m}\in {\mathcal{M}}_m$.\\

\noindent Case 2 --- Setting (B). We have $\matrice{m} \in S^*_{m}$ (and no further
assumption on $h_\tr$). In this case, letting $\varsigma : [m]
\rightarrow [m]$ represent the (block-class) permutation, we
have:
\begin{eqnarray}
\sum_{i\in [m]}
{y_i \cdot h_\tr((\matrice{m}\matrice{s}\matrice{f}^\tr)^\top
  \ve{1}_i)}  & = & \sum_{i\in [m]}
{y_i \cdot h_\tr\left(\bigoplus_{k\in [d_\tr]_{d_\re}} \sum_{l\in [m]}
    \matrice{m}_{il} \matrice{s}_{lk}\right)}\nonumber\\
& = & \sum_{i\in [m]}
{y_i \cdot h_\tr\left(\bigoplus_{k\in [d_\tr]_{d_\re}}
    \matrice{s}_{\varsigma(i) k}\right)}\nonumber\\
 & = & \sum_{i\in [m]}
{y_i \cdot h_\tr\left( (\matrice{s}\matrice{f}^\tr)^\top\ve{1}_{\varsigma(i)}\right)}\nonumber\\
& = & \sum_{i\in [m]}
{y_{\varsigma(i)} \cdot h_\tr\left( (\matrice{s}\matrice{f}^\tr)^\top\ve{1}_{\varsigma(i)}\right)}\label{conds1}\\
& = & \sum_{l\in [m]}
{y_l \cdot h_\tr\left( (\matrice{s}\matrice{f}^\tr)^\top\ve{1}_{l}\right)}\label{conds2}\:\:.
\end{eqnarray}
Eq. (\ref{conds1}) holds because $\matrice{m}$ is a block-class matrix
(eq. (\ref{blockm}),
and eq. (\ref{conds2}) holds because $\varsigma$ is a permutation. This ends the proof of Lemma
\ref{lemint1}
\end{proof}
As a
remark, when $h_\tr(\ve{x}) = \ve{\theta}_\tr^\top \ve{x}_\tr$, with
$\ve{\theta}_\tr \in {\mathbb{R}}^{d_\tr}$, eq. (\ref{pmeano}) shows the
invariance of the mean operator
\begin{eqnarray}
\mu_{{\mathcal{S}}^{\process{T}}}  = \mu_{{\mathcal{S}}} & \defeq & \frac{1}{m}\cdot \sum_i
{y_i \bm{x}_i}\label{defmeano}\:\: (\forall \process{T})\:\:,
\end{eqnarray} 
 as minimal
sufficient statistic for the classes \citep{pnrcAN}. Using eqs. (\ref{defrel}) and
(\ref{pmeano}) yield the first following identity ($\forall
{\mathcal{S}}, \process{T}, h$):
\begin{eqnarray}
\lefteqn{\expect_{{\mathcal{S}}} \left[\varphi(y
   h(\ve{x}))\right] - \expect_{{\mathcal{S}}^{\process{T}}} \left[\varphi(y
   h(\ve{x}))\right]}\nonumber\\
 & = & \expect_{{\mathcal{S}}} \left[\varphi(y
   h(\ve{x}))\right] - \expect_{{\mathcal{S}}^{\process{T}}} \left[\varphi(y
   h(\ve{x}))\right]\nonumber\\
 & = & \frac{b_\varphi}{2m} \left\{
\begin{array}{c}
\sum_{\sigma \in \Sigma_1} {\sum_{i}
{\varphi(\sigma h((\matrice{s}\matrice{f}^\re)^\top \ve{1}_i \oplus
  (\matrice{s}\matrice{f}^\tr)^\top \ve{1}_i))}} \\
- \\
\sum_{\sigma \in \Sigma_1} {\sum_{i}
{\varphi(\sigma h((\matrice{s}\matrice{f}^\re)^\top \ve{1}_i \oplus
  (\matrice{m}\matrice{s}\matrice{f}^\tr)^\top \ve{1}_i))}}
\end{array}\right\}\nonumber\\ 
 & = & \frac{b_\varphi}{2m} \cdot \expect_{\ve{\sigma} \sim
   \Sigma_m}\left[ 
\begin{array}{c}
\sum_{i}
{\varphi(\sigma_i h((\matrice{s}\matrice{f}^\re)^\top \ve{1}_i \oplus
  (\matrice{s}\matrice{f}^\tr)^\top \ve{1}_i))}\\
 - \\
\sum_{i}
{\varphi(\sigma_i h((\matrice{s}\matrice{f}^\re)^\top \ve{1}_i \oplus
  (\matrice{m}\matrice{s}\matrice{f}^\tr)^\top \ve{1}_i))}
\end{array}\right]\label{pbIV}\\ 
 & \leq & \frac{1}{2m} \cdot \expect_{\ve{\sigma} \sim \Sigma_m}\left[ \left|{\sum_{i}
{\sigma_i h((\matrice{s}\matrice{f}^\re)^\top \ve{1}_i \oplus (\matrice{s}\matrice{f}^\tr)^\top \ve{1}_i)} - \sigma_i h((\matrice{s}\matrice{f}^\re)^\top \ve{1}_i \oplus
  (\matrice{m}\matrice{s}\matrice{f}^\tr)^\top
  \ve{1}_i)}\right|\right]\label{pIV}\\ 
 & & =  \frac{1}{2m} \cdot \expect_{\ve{\sigma} \sim \Sigma_m}\left[ \left|\sum_{i}
{\sigma_i 
\left\{
\begin{array}{c}
(h_\re((\matrice{s}\matrice{f}^\re)^\top \ve{1}_i) -
h_\re((\matrice{s}\matrice{f}^\re)^\top \ve{1}_i)) \\
+ \\
 (h_\tr((\matrice{s}\matrice{f}^\tr)^\top
  \ve{1}_i) - h_\tr((\matrice{m}\matrice{s}\matrice{f}^\tr)^\top
  \ve{1}_i))
\end{array}
\right\}}\right|\right]\label{pbV}\\ 
 &  =  & \frac{1}{2m} \cdot \expect_{\ve{\sigma} \sim \Sigma_m}\left[ \left|\sum_{i}
{\sigma_i (h_\tr((\matrice{s}\matrice{f}^\tr)^\top
  \ve{1}_i) - h_\tr((\matrice{m}\matrice{s}\matrice{f}^\tr)^\top
  \ve{1}_i))}\right|\right]\:\:. \label{pbVb}
\end{eqnarray}
Eq. (\ref{pbIV}) holds because $\ve{\sigma}$ is
Rademacher. 
Ineq. (\ref{pIV}) holds because $F_\varphi$ is
$(1/b_\varphi)$-Lipschitz \citep{nnOT}. Eq. (\ref{pbV}) holds because
of assumption \textbf{(i)}. We thus get the
following upperbound for $B$ in ineq. (\ref{ss1}):
\begin{eqnarray}
\lefteqn{\expect_{{\mathcal{S}} \sim {\mathcal{D}}}\left[ \sup_{h\in {\mathcal{H}}} \left\{ \expect_{{\mathcal{S}}} \left[\varphi(y
   h(\bm{x}))\right] - \expect_{{\mathcal{S}}^{\process{T}}} \left[\varphi(y
   h(\bm{x}))\right]\right\}\right]}\nonumber\\
 & \leq & \expect_{{\mathcal{S}} \sim
{\mathcal{D}}}\left[ \sup_{h_\tr}  \expect_{\ve{\sigma} \sim
   \Sigma_m}\left[ \left|\frac{1}{m} \sum_{i}
{\sigma_i (h_\tr((\matrice{s}\matrice{f}^\tr)^\top
  \ve{1}_i) - h_\tr((\matrice{m}\matrice{s}\matrice{f}^\tr)^\top
  \ve{1}_i))}\right|\right]\right]
\nonumber\\
 & \leq & \expect_{{\mathcal{S}} \sim
{\mathcal{D}}, \ve{\sigma} \sim
   \Sigma_m}\left[ \sup_{h_\tr} \left|\frac{1}{m} \sum_{i}
{\sigma_i (h_\tr((\matrice{s}\matrice{f}^\tr)^\top
  \ve{1}_i) - h_\tr((\matrice{m}\matrice{s}\matrice{f}^\tr)^\top
  \ve{1}_i))}\right|\right]\:\:,\label{leq1}
\end{eqnarray}
where ineq. (\ref{leq1}) holds because of the convexity of
$\mathrm{sup}$. Using assumption \textbf{(ii)} ($h_\tr(.) \in [0, K_\tr]$), another use of the independent
bounded differences inequality \citep{mdC} yield with probability $\geq 1 - \updelta_2$:
\begin{eqnarray}
\lefteqn{\expect_{{\mathcal{S}} \sim
{\mathcal{D}}, \ve{\sigma} \sim
   \Sigma_m}\left[ \sup_{h_\tr} \left|\frac{1}{m} \sum_{i}
{\sigma_i (h_\tr((\matrice{s}\matrice{f}^\tr)^\top
  \ve{1}_i) - h_\tr((\matrice{m}\matrice{s}\matrice{f}^\tr)^\top
  \ve{1}_i))}\right|\right]}\nonumber\\
 & \leq & \expect_{\ve{\sigma} \sim
   \Sigma_m} \left[ \sup_{h_\tr} \left|\frac{1}{m} \sum_{i}
{\sigma_i (h_\tr((\matrice{s}\matrice{f}^\tr)^\top
  \ve{1}_i) - h_\tr((\matrice{m}\matrice{s}\matrice{f}^\tr)^\top
  \ve{1}_i))}\right|\right] + K_\tr \cdot \sqrt{\frac{2}{m}
\log\frac{1}{\updelta_2}}\:\:. \label{bdelta2}
\end{eqnarray}
We now put altogether ineqs. (\ref{bdelta1}), (\ref{bdelta1_2}) and
(\ref{bdelta2}) and obtain that with probability $\geq 1 -
(2\delta_1 + \delta_2)$, we shall have
\begin{eqnarray}
\expect_{{\mathcal{D}}} \left[L_{0/1}(y, h(\bm{x}))\right] & \leq &
\expect_{{\mathcal{S}}^{\process{T}}} \left[\varphi(y
   h(\bm{x}))\right] 
\nonumber\\
 & & + \frac{4}{b_\varphi} \cdot \expect_{\ve{\sigma} \sim
   \Sigma_m} \left[ \sup_{h\in {\mathcal{H}}} \left|\frac{1}{m} \sum_i
{\sigma_i h \left(\ve{x}_i\right)}\right|\right] \nonumber\\
 & & + \expect_{\ve{\sigma} \sim
   \Sigma_m} \left[ \sup_{h_\tr} \left|\frac{1}{m} \sum_{i}
{\sigma_i (h_\tr((\matrice{s}\matrice{f}^\tr)^\top
  \ve{1}_i) - h_\tr((\matrice{m}\matrice{s}\matrice{f}^\tr)^\top
  \ve{1}_i))}\right|\right] \nonumber\\
 & & + 2 K_\varphi \cdot
\sqrt{\frac{2}{m}\log \frac{1}{\updelta_1}} + K_\tr \cdot \sqrt{\frac{2}{m}
\log\frac{1}{\updelta_2}}\:\:.
\end{eqnarray}
To simplify this expression, we fix $\updelta_1 =
\updelta_2 = \updelta/3$ and get with probability $\geq 1 - \delta$,
\begin{eqnarray}
\expect_{{\mathcal{D}}} \left[L_{0/1}(y, h(\bm{x}))\right] & \leq &
\expect_{{\mathcal{S}}^{\process{T}}} \left[\varphi(y
   h(\bm{x}))\right] 
\nonumber\\
 & &  + \expect_{\ve{\sigma} \sim
   \Sigma_m} \left[ \sup_{h_\tr} \left|\frac{1}{m} \sum_{i}
{\sigma_i (h_\tr((\matrice{s}\matrice{f}^\tr)^\top
  \ve{1}_i) - h_\tr((\matrice{m}\matrice{s}\matrice{f}^\tr)^\top
  \ve{1}_i))}\right|\right] \nonumber\\
 & & + \frac{4}{b_\varphi} \cdot \expect_{\ve{\sigma} \sim
   \Sigma_m} \left[ \sup_{h\in {\mathcal{H}}} \left|\frac{1}{m} \sum_i
{\sigma_i h \left(\ve{x}_i\right)}\right|\right]   + (2 K_\varphi + K_\tr) \cdot
\sqrt{\frac{2}{m}\log \frac{3}{\updelta}}\nonumber\\
 & & = \expect_{{\mathcal{S}}^{\process{T}}} \left[\varphi(y
   h(\bm{x}))\right]  +
\disc_{\process{T}} ({\mathcal{H}}) + \frac{4}{b_\varphi} \cdot \expect_{\ve{\sigma} \sim
   \Sigma_m} \left[ \sup_{h\in {\mathcal{H}}} \left|\frac{1}{m} \sum_i
{\sigma_i h \left(\ve{x}_i\right)}\right|\right] \nonumber\\
 & & + (2 K_\varphi + K_\tr) \cdot
\sqrt{\frac{2}{m}\log \frac{3}{\updelta}}\nonumber\:\:,\nonumber
\end{eqnarray}
from which we obtain the statement of Theorem \ref{rader}.

\subsection{Example of domain and \cp~for which (min risk over
  ${\mathcal{S}}^{\process{T}}$ + Rademacher \cp~complexity) is strictly smaller
  than (min risk over ${\mathcal{S}}$)}\label{ex_domain}

We exhibit a toy domain which shows that
\begin{eqnarray}
\min_h \expect_{{\mathcal{S}}^{\process{T}}} \left[\varphi(y
   h(\bm{x}))\right] + \disc_{\process{T}} ({\mathcal{H}}) & < & \min_h \expect_{{\mathcal{S}}} \left[\varphi(y
   h(\bm{x}))\right]\:\:,
\end{eqnarray}
for $\varphi$ = square loss.
Let $\mathcal{X} = {\mathbb{R}}^2$, with $\mathcal{S}$ consisting of
$2$ copies of observation $(0,0)$ (positive), $2$ copies of
observation $(1,1)$ (positive), and $1$ copy of observation $(-1,-1)$
(negative). We enumerate the examples in ${\mathcal{S}}$ in this order.
The $\cp$ satisfies $\matrice{f}^\re \defeq \{x\}$,
$\matrice{f}^\tr \defeq \{y\}$ and
\begin{eqnarray}
\matrice{m} & \defeq & \left[
\begin{array}{ccccc}
0 & 0 & 1 & 0 & 0\\
0 & 0 & 0 & 1 & 0\\
1 & 0 & 0 & 0 & 0\\
0 & 1 & 0 & 0 & 0\\
0 & 0 & 0 & 0 & 1
\end{array}
\right]\:\:,\nonumber
\end{eqnarray}
so that ${\mathcal{S}}^{\process{T}}$ consists of two copies of observation
$(1,0)$ (positive), two copies of $(0,1)$ (positive) and one copy of
$(-1,-1)$ (the same observation as in $\mathcal{S}$). Let $h \defeq \ve{\theta}$, with its coordinates denoted
$x$ and $y$. The square loss $L$ over
${\mathcal{S}}$ equals:
\begin{eqnarray*}
L & = & \frac{1}{5}\cdot \left( 2 + 2(1-x-y)^2 + (1-x-y)^2  \right)\:\:,
\end{eqnarray*}
which is minimized for $x=y=1/2$ and yields $L = 2/5$. The square loss $L^{\process{T}}$ over
${\mathcal{S}}^{\process{T}}$ equals:
\begin{eqnarray}
L^{\process{T}} & = & \frac{1}{5}\cdot \left( 2(1-x)^2 + 2(1-y)^2 + (1-x-y)^2  \right)\:\:,
\end{eqnarray}
which is minimized for $x=y=3/4$ and yields $L^{\process{T}} = 1/10$. Assuming all
linear separators have $\ell_\infty$ norm bounded by $3/4$ (which allows
to have both solutions above), the \rcp~is
\begin{eqnarray}
\disc_{\process{T}} ({\mathcal{H}})& \defeq & \expect_{\ve{\sigma} \sim
   \Sigma_m} \left[ \sup_{h \in {\mathcal{H}}_\tr} \left|\frac{1}{m} \sum_{i}
{\sigma_{i} \left( h((\matrice{s}\matrice{f}^\tr)^\top \ve{1}_i) -
    h((\matrice{m}\matrice{s}\matrice{f}^\tr)^\top
    \ve{1}_i)\right)}\right|\right]\nonumber\\
 & = & \frac{1}{5} \cdot \frac{3}{4} \cdot \frac{1}{16}\cdot\sum_{\ve{\sigma} \in
   \Sigma_4}\left| -\sigma_1 -\sigma_2 +\sigma_3 +\sigma_4\right|\nonumber\\
 & = & \frac{3}{20} \cdot \frac{1}{16}\cdot\sum_{\ve{\sigma} \in
   \Sigma_4}\left| \ve{1}^\top \ve{\sigma}\right|\nonumber\\
 & = & \frac{3}{20} \cdot \frac{1}{16}\cdot (4\cdot 2 + 2\cdot 8 + 0
 \cdot 6)\label{rade1}\\
 & = & \frac{3}{20} \cdot \frac{24}{16} = \frac{9}{40}\:\:.
\end{eqnarray}
We then check that
\begin{eqnarray}
L^{\process{T}} + \disc_{\process{T}} ({\mathcal{H}}) = \frac{13}{40}
& < & \frac{2}{5} = L\:\:,
\end{eqnarray}
as claimed.

\subsection{Proof of Theorem \ref{thradcompUU}}\label{proof_thm_thradcompUU}

The proof of Theorem \ref{thradcompUU} follows from the proof of a more general Theorem that we prove here. We say that
$(\matrice{m}, \matrice{k}^\tr)$ satisfies the 
$(\gamma, \delta)$-correlation assumption for some $0<\delta, \gamma
\leq 1$ iff the following two assumptions hold:
\begin{itemize}
\item [\textbf{(a)}] $((\matrice{i}_m-\matrice{m})\matrice{k}^\tr
  (\matrice{i}_m-\matrice{m})^\top)_{ii} \geq (1-\delta) \cdot (1/m) \cdot
  \trace{(\matrice{i}_m-\matrice{m})\matrice{k}^\tr (\matrice{i}_m-\matrice{m})^\top }$, $\forall i\in [m]$;
\item [\textbf{(b)}] $|((\matrice{i}_m-\matrice{m})\matrice{k}^\tr (\matrice{i}_m-\matrice{m})^\top)_{ii'} /
  \sqrt{((\matrice{i}_m-\matrice{m})\matrice{k}^\tr (\matrice{i}_m-\matrice{m})^\top)_{ii} ((\matrice{i}_m-\matrice{m})\matrice{k}^\tr (\matrice{i}_m-\matrice{m})^\top)_{i'i'}}| \geq 1 - \gamma$,
  $\forall i, i' \in [m]$.
\end{itemize}
Regardless of ${\mathcal{S}}$, there always exist $0<\delta, \gamma
\leq 1$ for which this holds, but the bound may be quantitatively
better when
at least one is small.
\begin{theorem}\label{thradcomp_2}
Using notations of Theorem \ref{thradcompUU}, there exists $\epsilon > 0$
such that for any
${\mathcal{S}}$ and $\process{T}$ for which $(\matrice{m}, \matrice{k}^\tr)$ satisfies the 
$(\gamma, \delta)$-correlation assumption, we have
\begin{eqnarray}
\disc_{\process{T}} ({\mathcal{H}}) & \leq & u \cdot \frac{r_\tr}{\sqrt{m}} \cdot
\sqrt{\frac{1}{m} \cdot \cipr{\matrice{i}_m}{\matrice{k}^\tr}{\matrice{m}}}\matrice{i}_m \:\:.\label{feq112}
 \end{eqnarray}
with
\begin{eqnarray}
u & \defeq & \frac{1}{m} + \kappa(\epsilon) \left(1 - \frac{1}{m}\right)\:\:,\label{defuuu}
\end{eqnarray}
and $\kappa(\epsilon) = 1 - ((1-\delta)(1-\epsilon)(1-\gamma))^2 \in (0,1)$.
\end{theorem}
Furthermore, 
\begin{eqnarray}
\frac{1}{m} \cdot \cipr{\matrice{i}_m}{\matrice{k}^\tr}{\matrice{m}}  & = &  2\cdot \sum_{j\in [d_\tr]} \var(\matrice{s}\matrice{f}^\tr\ve{1}_j)(1-\rho(\matrice{s}\matrice{f}^\tr\ve{1}_j,
\matrice{m}\matrice{s}\matrice{f}^\tr\ve{1}_j)) \label{simplinn}\:\:.
\end{eqnarray}
\begin{proof}
We observe that
$\forall \bm{\sigma} \in
\Sigma_m$,
\begin{eqnarray}
  \lefteqn{\arg \sup_{\ve{\theta} \in {\mathbb{R}}^{d_\tr} :\|\ve{\theta}\|_2 \leq r_\tr} 
     \left|\frac{1}{m} \sum_{i}
{\sigma_{i} \ve{\theta}^\top\left(((\matrice{i}_m - \matrice{m})\matrice{s}\matrice{f}^\tr)^\top\ve{1}_i\right)}\right|}\nonumber\\
  & = & \frac{r_\tr}{\left\| \sum_{i}
{\sigma_{i} ((\matrice{i}_m - \matrice{m})\matrice{s}\matrice{f}^\tr)^\top\ve{1}_i}\right\|_2}  \sum_{i}
{\sigma_{i} ((\matrice{i}_m - \matrice{m})\matrice{s}\matrice{f}^\tr)^\top\ve{1}_i}\:\:,\nonumber
\end{eqnarray}
and so:
\begin{eqnarray}
\disc_{\process{T}} ({\mathcal{H}}) & = & \expect_{\ve{\sigma} \sim
   \Sigma_m} \left[ \sup_{h \in {\mathcal{H}}_\tr} \left|\frac{1}{m} \sum_{i}
{\sigma_{i} \left( h((\matrice{s}\matrice{f}^\tr)^\top \ve{1}_i) - h((\matrice{m}\matrice{s}\matrice{f}^\tr)^\top \ve{1}_i)\right)}\right|\right]\\
 & = &
\frac{r_\tr}{m}\cdot \expect_{\ve{\sigma} \sim
   \Sigma_m} \left[ \left\|\sum_{i}
{\sigma_{i} ((\matrice{i}_m - \matrice{m})\matrice{s}\matrice{f}^\tr)^\top\ve{1}_i}\right\|_2\right] \nonumber\\
& = & \frac{r_\tr}{m }\cdot \sqrt{\sum_{i}
{\left\|((\matrice{i}_m - \matrice{m})\matrice{s}\matrice{f}^\tr)^\top\ve{1}_i\right\|_2^2}}  \nonumber\\
 & & \cdot \expect_{\ve{\sigma} \sim
   \Sigma_m} \left[ \sqrt{1 + \frac{\sum_{i\neq i'} {\sigma_{i}
         \sigma_{i'} \ve{1}^\top_i (\matrice{i}_{m} -
\matrice{m}) \matrice{s}\matrice{f}^\tr (\matrice{s}\matrice{f}^\tr)^\top 
(\matrice{i}_{m} -
\matrice{m})^\top\ve{1}_{i'}}}{\sum_{i}
{\left\|((\matrice{i}_m - \matrice{m})\matrice{s}\matrice{f}^\tr)^\top\ve{1}_i\right\|_2^2}}}\right] \nonumber\\
 & \defeq  & \frac{r_\tr}{m}\cdot\sqrt{\sum_{i}
{\left\|((\matrice{i}_m - \matrice{m})\matrice{s}\matrice{f}^\tr)^\top\ve{1}_i\right\|_2^2}}  \cdot \expect_{\ve{\sigma} \sim
   \Sigma_m} \left[\sqrt{1 + u(\ve{\sigma})}\right]\:\:,\label{eqq1}
\end{eqnarray}
with
\begin{eqnarray}
u(\ve{\sigma}) & \defeq & \frac{\sum_{i\neq i'} {\sigma_{i}
         \sigma_{i'} \ve{1}^\top_i (\matrice{i}_{m} -
\matrice{m}) \matrice{s}\matrice{f}^\tr (\matrice{s}\matrice{f}^\tr)^\top 
(\matrice{i}_{m} -
\matrice{m})^\top\ve{1}_{i'}}}{\sum_{i}
{\left\|((\matrice{i}_m -
    \matrice{m})\matrice{s}\matrice{f}^\tr)^\top\ve{1}_i\right\|_2^2}}\:\:. \label{defuuX}
\end{eqnarray}
Let us call for short $\ve{\delta}_{i} \defeq ((\matrice{i}_m -
    \matrice{m})\matrice{s}\matrice{f}^\tr)^\top\ve{1}_i$, so that eq. (\ref{defuuX}) can be simplified
to $u(\ve{\sigma}) = (\sum_i \|\ve{\delta}_{i}\|_2^2)^{-1}\sum_{i\neq i'} \sigma_i \sigma_{i'}
\ve{\delta}^\top_{i} \ve{\delta}_{i'}$. $\forall n\in {\mathbb{N}}_*$, we have 
\begin{eqnarray}
\lefteqn{\expect_{\ve{\sigma} \sim
   \Sigma_m} \left[ u^n(\ve{\sigma})\right]}\nonumber\\
 & = & \frac{1}{\left(\sum_i \|\ve{\delta}_{i}\|_2^2\right)^n} \cdot \expect_{\ve{\sigma} \sim
   \Sigma_m} \left[ \sum_{i_1 \neq i'_1} \sum_{i_2 \neq i'_2} \cdots \sum_{i_n \neq i'_n} \prod_{k=1}^{n} \sigma_{i_k} \sigma_{i'_k}\ve{\delta}_{i_k}^\top \ve{\delta}_{i'_k}\right] \nonumber\\
 & = & \frac{1}{\left(\sum_i \|\ve{\delta}_{i}\|_2^2\right)^n} \cdot \sum_{i_1 \neq i'_1} \sum_{i_2 \neq i'_2} \cdots \sum_{i_n \neq i'_n} \expect_{\ve{\sigma} \sim
   \Sigma_m} \left[\prod_{k=1}^{n} \sigma_{i_k} \sigma_{i'_k}\ve{\delta}_{i_k}^\top \ve{\delta}_{i'_k}\right] \nonumber\\
 & = & \frac{1}{\left(\sum_i \|\ve{\delta}_{i}\|_2^2\right)^n} \cdot \sum_{i_1
\neq i'_1} \sum_{i_2 \neq i'_2} \cdots \sum_{i_n \neq i'_n} \prod_{(i,
i') \in \{(i_k, i'_k)\}_{k=1}^n} \expect_{\ve{\sigma} \sim
   \Sigma_m} \left[(\sigma_{i} \sigma_{i'})^{n(i,i')}\right] \left(\ve{\delta}_{i}^\top \ve{\delta}_{i'}\right)^{n(i,i')}\:\:,
\end{eqnarray}
with $n(i,i') \defeq |\{k : (i,i') = (i_k,i'_k)\}|$ satisfying $\sum
n(i,i') = n$. Whenever $n(i,i')$ is odd, $\expect_{\ve{\sigma} \sim
   \Sigma_m} \left[(\sigma_{i} \sigma_{i'})^{n(i,i')}\right] = 0$
 (because $\ve{\sigma}$ is Rademacher), and
 it is 1 otherwise. We get, if $n$ is even:
\begin{eqnarray}
\lefteqn{\expect_{\ve{\sigma} \sim
   \Sigma_m} \left[ u^n(\ve{\sigma})\right]}\nonumber\\
 & = & \frac{1}{\left(\sum_i \|\ve{\delta}_{i}\|_2^2\right)^n} \cdot
\underbrace{\sum_{0<\ell\leq n} \sum_{
\begin{array}{c}
\{n_k\}_{k=1}^{\ell}
\subset {\mathbb{N}_*} \\
\mbox{s.t. } 2 \sum n_k = n
\end{array}} \sum_{
\begin{array}{c}
\{(i_k, i'_k)\}_{k=1}^\ell \\
\mbox{s.t. } i_k
\neq i'_k, \forall k
\end{array}} \prod_{k=1}^{\ell} \left(\ve{\delta}_{i_k}^\top \ve{\delta}_{i'_k}\right)^{2n_k}}_{\defeq \zeta(n)}\:\:,\label{beqU}
\end{eqnarray}
and $\expect_{\ve{\sigma} \sim
   \Sigma_m} \left[ u^n(\ve{\sigma})\right] = 0$ if $n$ is odd. Since 
\begin{eqnarray}
\sqrt{1+x} & = & 1 + \sum_{n\in {\mathbb{N}}_*} {\frac{1}{2^n n!}
  \cdot \prod_{k=0}^{n-1} (1-2k) x^n}\:\:,
\end{eqnarray}
we get after combining with eqs (\ref{eqq1}) and (\ref{beqU}) and
using the definition of $\zeta(.)$ in eq. (\ref{beqU}):
\begin{eqnarray}
\disc_{\process{T}} ({\mathcal{H}})  & = & \frac{r_\tr}{m}\cdot \sqrt{\sum_i \|\ve{\delta}_{i}\|_2^2} \cdot \left( 1 -
\sum_{n\in {\mathbb{N}}_*} {\frac{\prod_{k=1}^{2n-1} (2k-1) \cdot \zeta(2n)}{ (2n)!\left(2\sum_i
{\left\|\ve{\delta}_{i}\right\|_2^2}\right)^{2n}}}\right)\nonumber\\
 & = & \frac{r_\tr}{m}\cdot \sqrt{\sum_i \|\ve{\delta}_{i}\|_2^2} \cdot \left( 1 -
\sum_{n\in {\mathbb{N}}_*} {\frac{\prod_{k=1}^{2n-1} (2k-1)}{ (2m)^{2n}(2n)!}}
\cdot \tilde{\zeta}(2n)\right)\label{eqdr}\:\:,
\end{eqnarray}
with:
\begin{eqnarray}
\lefteqn{\tilde{\zeta}(2n)}\nonumber\\
 & \defeq & \sum_{0<\ell\leq 2n} \sum_{
\begin{array}{c}
\{n_k\}_{k=1}^{\ell}
\subset {\mathbb{N}_*} \\
\mbox{s.t. } \sum n_k = n
\end{array}} \sum_{
\begin{array}{c}
\{(i_k, i'_k)\}_{k=1}^\ell \\
\mbox{s.t. } i_k
\neq i'_k, \forall k
\end{array}} \prod_{k=1}^{\ell} \left(
\begin{array}{c}
\frac{\left\|\ve{\delta}_{i_k}\right\|_2^2}{\frac{1}{m} \cdot \sum_p
{\left\|\ve{\delta}_{p}\right\|_2^2}} \\
\cdot 
\frac{\left\|\ve{\delta}_{i'_k}\right\|_2^2}{\frac{1}{m} \cdot \sum_p
{\left\|\ve{\delta}_{p}\right\|_2^2}} 
 \cdot
\gamma_{i_k, i'_k}\\
\end{array}\right)^{2n_k}\:\:,
\end{eqnarray}
and $\gamma_{i,i'} \defeq
\cos(\ve{\delta}_{i},\ve{\delta}_{i'})$. Remark that eq. (\ref{eqdr})
is an equality.
We now use assumption \textbf{(a)} and obtain
\begin{eqnarray}
\tilde{\zeta}(2n) & \geq & (1-\delta)^{2n} \sum_{0<\ell\leq 2n} \sum_{
\begin{array}{c}
\{n_k\}_{k=1}^{\ell}
\subset {\mathbb{N}_*} \\
\mbox{s.t. } \sum n_k = n
\end{array}} \sum_{
\begin{array}{c}
\{(i_k, i'_k)\}_{k=1}^\ell \\
\mbox{s.t. } i_k
\neq i'_k, \forall k
\end{array}} \prod_{k=1}^{\ell} \left(
\gamma_{i_k, i'_k}\right)^{2n_k}\:\:.\label{ppp1}
\end{eqnarray}
Denote for short ${\mathcal{U}}(n)$ the set of eligible triples $(n_k,
i_k, i'_k)$ in the summation. We get because of assumption
\textbf{(b)} $\tilde{\zeta}(2n) \geq |{\mathcal{U}}(n)|((1-\delta)
(1-\gamma))^{2n}$, and so, using the shorthand 
\begin{eqnarray}
\mathscr{L} & \defeq & \frac{1}{m} \cdot \sum_i
{\left\|\ve{\delta}_{i}\right\|_2^2}\:\:,\label{defell}
\end{eqnarray} 
we obtain our first upperbound,
\begin{eqnarray}
\disc_{\process{T}} ({\mathcal{H}}) & \leq & \frac{r_\tr}{\sqrt{m}}\cdot \sqrt{\mathscr{L}} \cdot  \left( 1 -
\sum_{n\in {\mathbb{N}}_*} {|{\mathcal{U}}(n)|\left(\frac{(1-\delta)
(1-\gamma)}{m}\right)^{2n} \cdot \frac{\prod_{k=1}^{2n-1} (2k-1)}{
    2^{2n}(2n)!}}\right)\label{lleq2}\:\:.
\end{eqnarray}
\begin{lemma}
There exists a constant
$\epsilon > 0$ such that $\forall n \in {\mathbb{N}}_*$,
\begin{eqnarray}
\frac{\prod_{k=1}^{2n-1} (2k-1)}{
    (2n)!} & \geq & (2(1-\epsilon))^{2n}\:\:.
\end{eqnarray}
\end{lemma}
\begin{proof}
We proceed by induction, letting $g_\epsilon(n) \defeq
(2(1-\epsilon))^{2n}$ and
\begin{eqnarray}
f(n) & \defeq & \frac{\prod_{k=1}^{2n-1} (2k-1)}{
    (2n)!}\:\:,
\end{eqnarray}
and for $\epsilon = \epsilon_* \defeq 1 - 1/(2\sqrt{2})$. We remark
that $g_{\epsilon_*}(1) = f(1)$. 
Furthermore,
\begin{eqnarray}
 f(n+1) = \frac{\prod_{k=1}^{2(n+1)-1} (2k-1)}{
    (2(n+1))!} & = & \frac{(2n-1)(2n+1)}{(n+1)(n+2)} \cdot  \frac{\prod_{k=1}^{2n-1} (2k-1)}{
    (2n)!} \nonumber\\
 & \geq & \frac{(2n-1)(2n+1)}{(n+1)(n+2)} \cdot (2(1-\epsilon))^{2n}\label{useind}\\
 & & = \frac{(2n-1)(2n+1)}{(n+1)(n+2)(2(1-\epsilon))^2} \cdot g_\epsilon(n+1)\nonumber\:\:,
\end{eqnarray}
where ineq. (\ref{useind}) uses the induction hypothesis. We prove the
Lemma once we prove that the factor on the right is at least 1, that
is, for $\epsilon = \epsilon_*$, we need to prove
\begin{eqnarray}
\frac{(2n-1)(2n+1)}{(n+1)(n+2)} & \geq & \frac{1}{2}\:\:,
\end{eqnarray}
which is indeed the case since the left function is strictly
increasing over ${\mathbb{N}}$ and equals the right-hand side for $n=1$.
\end{proof}
So we get from ineq. (\ref{lleq2}):
\begin{eqnarray}
\disc_{\process{T}} ({\mathcal{H}}) & \leq & \frac{r_\tr}{\sqrt{m}}\cdot \sqrt{\mathscr{L}} \cdot  \left( 1 -
\sum_{n\in {\mathbb{N}}_*} {|{\mathcal{U}}(n)|\left(\frac{(1-\delta)
(1-\gamma)(1-\epsilon)}{m}\right)^{2n}}\right) \nonumber\\
 & \leq & \frac{r_\tr}{\sqrt{m}}\cdot \sqrt{\mathscr{L}} \cdot  \left( 1 -
u_*(m) \sum_{n\in {\mathbb{N}}_*} {\left(\frac{(1-\delta)
(1-\gamma)(1-\epsilon)}{m}\right)^{2n}}\right)\:\:,
\end{eqnarray}
where $u_*(m)$ satisfies $u_*(m) \leq \min_n |{\mathcal{U}}(n)|$. We
finally get:
\begin{eqnarray}
\disc_{\process{T}} ({\mathcal{H}}) & \leq & \frac{r_\tr}{\sqrt{m}}\cdot \sqrt{\mathscr{L}} \cdot  \left( 1 -
(1-\kappa(\epsilon)) \cdot \frac{u_*(m)}{m^2-(1-\kappa(\epsilon))}\right)\nonumber\\
& \leq & \frac{r_\tr}{\sqrt{m}}\cdot \sqrt{\mathscr{L}} \cdot  \left( 1 -
(1-\kappa(\epsilon)) \cdot \frac{u_*(m)}{m^2}\right)\nonumber\\
 & \leq & \frac{r_\tr}{\sqrt{m}}\cdot \sqrt{\mathscr{L}} \cdot
 \left( \frac{1}{m} + \kappa(\epsilon) \left(1 - \frac{1}{m}\right)\right) \:\:,\label{eqkap}
\end{eqnarray}
with $\kappa(\epsilon) = 1 - ((1-\delta)(1-\epsilon)(1-\gamma))^2 >
0$, since $u_*(m) \geq m(m-1)$ (obtained for $\ell = 2n$ in
eq. (\ref{ppp1})). We finish the proof by remarking that $\mathscr{L}$ in
eq. (\ref{defell}) satisfies
\begin{eqnarray}
\mathscr{L} & = & \frac{1}{m} \cdot \sum_i
{\ve{1}^\top_i (\matrice{i}_{m} -
\matrice{m}) \matrice{s}\matrice{f}^\tr (\matrice{s}\matrice{f}^\tr)^\top 
(\matrice{i}_{m} -
\matrice{m})^\top\ve{1}_{i}}\nonumber\\
 & = & \frac{1}{m} \cdot \trace{(\matrice{i}_{m} -
\matrice{m}) \matrice{s}\matrice{f}^\tr (\matrice{s}\matrice{f}^\tr)^\top 
(\matrice{i}_{m} -
\matrice{m})^\top}\nonumber\\
 & = & \frac{1}{m} \cdot \trace{(\matrice{i}_{m} -
\matrice{m}) \matrice{k}^\tr 
(\matrice{i}_{m} -
\matrice{m})^\top}\nonumber\\
 & = & \frac{1}{m} \cdot \trace{(\matrice{i}_{m} -
\matrice{m})^\top \matrice{i}_m (\matrice{i}_{m} -
\matrice{m}) \matrice{k}^\tr }\nonumber\\
 & = & \frac{1}{m} \cdot \cipr{\matrice{i}_m}{\matrice{k}^\tr}{\matrice{m}}\:\:.\label{lll1z}
\end{eqnarray} 
Ineq. (\ref{eqkap}) and eq. (\ref{lll1z}) allow to conclude the proof
of Theorem \ref{thradcomp_2} with
\begin{eqnarray}
u & \defeq & \frac{1}{m} + \kappa(\epsilon) \left(1 - \frac{1}{m}\right)\:\:,\label{defuuXX}
\end{eqnarray}
which, since $\kappa(\epsilon) < 1$, satisfies indeed $u \in (0,1)$.
This achieves the main part of the proof of Theorem
\ref{thradcomp_2}. To prove eq. (\ref{simplinn}), we just have to
write (letting $\varsigma
: [m] \rightarrow [m]$ represent the corresponding permutation),
\begin{eqnarray}
\lefteqn{\frac{1}{m} \cdot \cipr{\matrice{i}_m}{\matrice{k}^\tr}{\matrice{m}}}\nonumber\\
 &
= & \frac{1}{m} \sum_i
{\left\|\ve{x}_{i}^\tr -
  \ve{x}^\tr_{\varsigma
    (i)}\right\|_2^2} \nonumber\\
 & = & \sum_{j\in [d_\tr]} \frac{1}{m} \cdot \sum_i
{(x_{ij}^\tr -
  x^\tr_{\varsigma
    (i)j})^2} \nonumber\\
 & = & 2\cdot \sum_{j\in [d_\tr]} \left\{\frac{1}{m} \cdot \sum_i
{(x_{ij}^\tr)^2 - \frac{1}{m} \cdot \sum_i
  x_{ij}^\tr x^\tr_{\varsigma
    (i)j}} \right\} \label{diff1}\\
 & = & 2\cdot \sum_{j\in [d_\tr]} \left\{\frac{1}{m} \cdot \sum_i
{(x_{ij}^\tr)^2} - \left( \frac{1}{m} \cdot \sum_i
{x_{ij}^\tr}\right)^2 + \left( \frac{1}{m} \cdot \sum_i
{x_{ij}^\tr}\right)^2 - \frac{1}{m} \cdot \sum_i
  {x_{ij}^\tr x^\tr_{\varsigma
    (i)j}} \right\}\nonumber\\
 & = & 2\cdot \sum_{j\in [d_\tr]} \left\{ \var(\matrice{s}\matrice{f}^\tr\ve{1}_j) - \cov(\matrice{s}\matrice{f}^\tr\ve{1}_j,\matrice{m}\matrice{s}\matrice{f}^\tr\ve{1}_j)\right\}\nonumber\\
 & = & 2\cdot \sum_{j\in [d_\tr]} \var(\matrice{s}\matrice{f}^\tr\ve{1}_j) \left\{ 1 -
   \frac{\cov(\matrice{s}\matrice{f}^\tr\ve{1}_j,\matrice{m}\matrice{s}\matrice{f}^\tr\ve{1}_j)}{\sqrt{\var(\matrice{s}\matrice{f}^\tr\ve{1}_j)}
     \sqrt{\var(\matrice{m}\matrice{s}\matrice{f}^\tr\ve{1}_j)} }\right\}\label{eqqq0}\\ 
 & = & 2\cdot \sum_{j\in [d_\tr]} \var(\matrice{s}\matrice{f}^\tr\ve{1}_j)(1-\rho(\matrice{s}\matrice{f}^\tr\ve{1}_j,
\matrice{m}\matrice{s}\matrice{f}^\tr\ve{1}_j)) \nonumber\:\:,
\end{eqnarray}
where eq. (\ref{eqqq0}) follows from the fact that $\var(\matrice{s}\matrice{f}^\tr\ve{1}_j) =
\var(\matrice{m}\matrice{s}\matrice{f}^\tr\ve{1}_j)$.
\end{proof}

\noindent \textbf{Remark}: it is worthwhile remarking that the proof
of Theorem \ref{thradcompUU} can also be applied to upperbound the
empirical Rademacher complexity of linear functions, without
modifications, except for the handling of $\mathscr{L}$. In this case, the
proof improves the upperbound known \citep{kstOT} (Theorem 1) by factor
$u$ in eq. (\ref{defuuXX}). This is due to the fact that the proof in \citep{kstOT} takes into
account only the maximum norm in the observations of ${\mathcal{S}}$,
and not the angles between the observations. We now state the corresponding Theorem.
\begin{theorem}\label{improveRC}
Following \citep{gftsssAK}, we let $\ii{r}{m}$ denote the
set of $r$-tuples drawn without replacement drawn from $[m]$. Suppose
${\mathcal{S}}$ satisfies the following for some $\delta, \gamma >0$
and $r_x > 0$:
\begin{itemize}
\item [\textbf{(a)}] $\|\ve{x}_{i}\|_2^2 \geq (1-\delta) \cdot (1/m)\sum_{i'}
{\left\|\ve{x}_{i'}\right\|_2^2}$, $\forall i \in [m]$;
\item [\textbf{(b)}] $\expect_{(i,i') \sim
    \ii{2}{m}}[|\cos(\ve{x}_{i}, \ve{x}_{i'})|] \geq 1 - \gamma$;
\item [\textbf{(c)}] $\|\ve{x}_{i}\|_2\leq r_x$, $\forall i \in [m]$.
\end{itemize}
Then, assuming that ${\mathcal{H}}$ contains linear classifiers of the
form $\ve{\theta}^\top \ve{x}$ with $\|\ve{\theta}\|_2\leq r_\theta$,
there exists $\epsilon>0$ such that the
empirical Rademacher complexity of ${\mathcal{H}}$ satisfies:
\begin{eqnarray}
\rad_{{\mathcal{S}}} ({\mathcal{H}}) & \leq & \left(\frac{1}{m} + \kappa(\epsilon) \left(1 - \frac{1}{m}\right)\right) \cdot \frac{r_x r_\theta}{\sqrt{m}}\:\:,\label{feq113}
 \end{eqnarray}
with $\kappa(\epsilon) = 1 - ((1-\delta)(1-\epsilon)(1-\gamma))^2 \in (0,1)$.
\end{theorem}
\citep{kstOT}'s proof relies on \textbf{(c)}. Since \textbf{(a)} and
\textbf{(b)} can always be satisfied for some $\delta, \gamma >0$,
ineq. (\ref{feq112}) holds under their setting as well; however, it
becomes better than theirs as both $\delta, \gamma$ are small, so in
particular in
the case where observations start to be heavily correlated and be of
approximately the same norm.
Indeed, in this case, $\sum_i \sigma_i \ve{x}_i$ will often have
\textit{small} magnitude, because $\Sigma_m \ni \ve{\sigma} \sim \{-1,1\}$ and thus
many vectors will approximately cancel through the sum in many draws
of $\ve{\sigma}$.

\subsection{Proof of Theorem \ref{thradcompSettingB}}\label{proof_thm_thradcompSettingB}

\begin{figure}[t]
\begin{center}
\includegraphics[width=0.4\columnwidth]{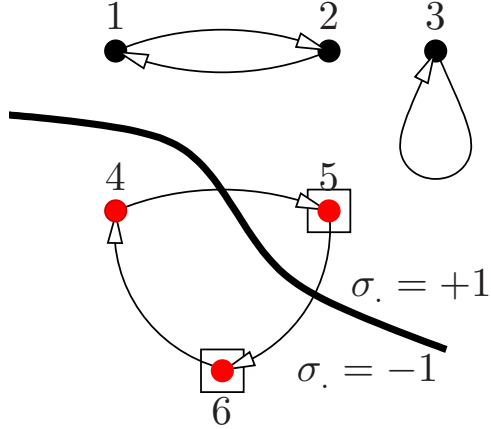}
\caption{A permutation $\varsigma$ defines an oriented graph whose
  vertices are examples and permutation $\varsigma$ defines arcs (here,
  $\varsigma(1) = 2$ for example; black dots are positive examples,
  red dots are negative examples). Each term in the $\sup$ of
  eq. (\ref{defsup2}) is a
  weighted cut (one weighted cut for each $\ve{\sigma} \in \{-1,1\}^m$): the squares depict the examples for which $w(.) \neq
  0$ in Lemma \ref{thradcomp2} for the $\ve{\sigma}$ displayed. In
  this example of $\ve{\sigma}$, only two examples out of the six
  would bring a non-zero weight $w(.)$.}
\label{f-setting2}
\end{center}
\end{figure}

We first start by a Lemma which shows that indeed
$\disc_{\process{T}} ({\mathcal{H}}) $
can be significantly smaller than a Rademacher complexity.
\begin{lemma}\label{thradcomp2}
Suppose setting (B) holds in Theorem \ref{rader}. Then
\begin{eqnarray}
\disc_{\process{T}} ({\mathcal{H}}) &
=&  2\cdot \expect_{\ve{\sigma} \sim
   \Sigma_m} \left[ \sup_{h_\tr \in {\mathcal{H}}_\tr} \left|\frac{1}{m} \sum_{i}
{w (i) h_\tr((\matrice{s}\matrice{f}^\tr)^\top \ve{1}_i)}\right|\right]\:\:, \label{defsup2}
\end{eqnarray}
where $w(i) \defeq (1/2) \cdot (\sigma_{i} -
\sigma_{\varsigma^{-1}(i)}) \in \{-1,0,1\}$.
\end{lemma}
The proof of this Lemma is straightforward. What is interesting is,
since $w(.)$ can take on zero values, to what extent
$\disc_{\process{T}} ({\mathcal{H}})$
can be smaller than the corresponding Rademacher complexity in which
$w(.)$ would be replaced by $\ve{\sigma} \in
\{-1,1\}^m$, and what drives this reduction. Figure \ref{f-setting2}
displays qualitatively this intuition on a simple example. We now
investigate a quantitative derivation of the reduction for \dagr~classifiers.

We let $\cut(\ve{\sigma}) \defeq \{i : \sigma_i \neq
\sigma_{\varsigma^{-1}(i)}\}$. We simplify notations in the proof and drop notation $b$
so that notation $h_\tr^i \defeq h_\tr((\matrice{s}\matrice{f}^\tr)^\top \ve{1}_i)$ where $i\in [m]$. We let
\begin{eqnarray}
\mu \defeq \expect_{\ve{\sigma} \sim\Sigma_m}\left[\sup_{h_\tr \in {\mathcal{H}}^\tr}
\sum_{i \in \cut(\ve{\sigma})} \sigma_i h_\tr^i\right]\:\:,
\end{eqnarray}
which is a generalisation of $m \cdot \disc_{\process{T}}
({\mathcal{H}})$ to any permutation $\varsigma \in S_{m}$, and not just a
block-class permutation in $S^*_{m}$ as assumed in Setting (B). We assume basic knowledge of Massart's
finite class Lemma's proof. Using Jensen's inequality, we arrive, after the same chain of
derivations, for any $t>0$, to:
\begin{eqnarray}
\exp(t \mu) & \leq & \frac{1}{2^m}
\sum_{\ve{\sigma} \in \{-1,1\}^m} \sup_{h_\tr \in {\mathcal{H}}^\tr}\exp\left(
   t \cdot \sum_{i \in
    \cut(\ve{\sigma})} \sigma_i h_\tr^i \right)\label{pdf1}\:\:.
\end{eqnarray}
The proof (of Massart's Lemma) now involves replacing the $\sup$ by a sum. Remark that when
$h$ is \dagr, the $\sup$ implies that each
$h_\tr^j$ is in fact $\in \{\pm K_\tr\}$. So let us use
${\mathcal{H}}^\tr_+ \subseteq {\mathcal{H}}^\tr$, the set of classifiers whose
output is in $\{\pm K_\tr\}$. We get:
\begin{eqnarray}
\exp(t \mu) & \leq & \sum_{h_\tr \in {\mathcal{H}}^\tr_+} \frac{1}{2^m}
\sum_{\ve{\sigma} \in \{-1,1\}^m} \exp\left( t \cdot \sum_{i \in
    \cut(\ve{\sigma})} \sigma_i h_\tr^i \right)\:\:.\label{eqqq1}
\end{eqnarray}
To identify better permutations, we name in this proof $\varsigma\in
S^*_m$ the permutation represented by $\matrice{m}$, so that we also
have
\begin{eqnarray}
\oddcycle(\matrice{m}) & \defeq & \oddcycle(\varsigma)\:\:.\nonumber
\end{eqnarray}
Since each coordinate of $\ve{\sigma}$ is chosen uniformly at random,
cycles in a permutation are disjoint 
and $\exp(a+b) = \exp(a)\exp(b)$,
the inner sigma in ineq. (\ref{eqqq1}) factors over the cycles of
permutation $\varsigma$:
\begin{eqnarray}
\lefteqn{\sum_{h_\tr \in {\mathcal{H}}^\tr_+} \frac{1}{2^m}
\sum_{\ve{\sigma} \in \{-1,1\}^m} \exp\left( t \cdot \sum_{i \in
    \cut(\ve{\sigma})} \sigma_i h_\tr^i \right)}\nonumber\\
 & = & \frac{1}{2^{m}} \sum_{h_\tr \in
  {\mathcal{H}}^\tr_+} \prod_{U \in \cycle(\varsigma)} \left( \sum_{\ve{\sigma} \in \{-1,1\}^{|U|}} \exp\left( t \cdot \sum_{i \in
    \cut(\ve{\sigma}) \cap U} \sigma_i h^{U}_{i} \right) \right)\:\:, \label{pqr1}
\end{eqnarray}
where $h^U$ indicates coordinates of $h$ in $U$ and $\cycle(\varsigma)$ is the
set of cycles \textit{without 1-cycles} (\textit{i.e.} fixed points)
--- we have let $m_\varsigma$ denote the number of fixed points of $\varsigma$. We have used the fact
that cycles define a partition of $[m]$. 

Now, whenever $U$ contains an
\textit{odd} number of indexes, whatever $\ve{\sigma}$, the sum over $\cut(\ve{\sigma}) \cap U$ cannot cover
the sum over $U$: there always remains at least one vertex which does
not belong to the sum (In Figure \ref{f-setting2}, the red dot cycle
displays this fact on an example). Let us denote $i_{\ve{\sigma}} \in [|U|]$ this
vertex. Since $\exp(x) + \exp(-x) \geq 2 + x^2$ and $h_\tr \in
{\mathcal{H}}^\tr_+$, we get:
\begin{eqnarray}
\lefteqn{\sum_{\ve{\sigma} \in \{-1,1\}^{|U|}} \exp\left( t \cdot \sum_{i \in
    \cut(\ve{\sigma}) \cap U} \sigma_i h^{U}_{i} \right)}\nonumber\\
 & \leq &
\frac{1}{2 + t^2 K_\tr^2}\cdot \sum_{\ve{\sigma} \in
  \{-1,1\}^{|U|}} \left\{\exp\left( t \cdot \sum_{i \in
    \cut(\ve{\sigma}) \cap U } \sigma_i h^{U}_{i}
\right) \cdot \left(\exp(t h_{i_{\ve{\sigma}}}^U) + \exp(-t h_{i_{\ve{\sigma}}}^U)\right)\right\}\label{mult11}\\
& \leq &
\frac{2}{2 + t^2 K_\tr^2}\cdot \sum_{\ve{\sigma} \in \{-1,1\}^{|U|}}
\exp\left( t \cdot \sum_{i \in U} \sigma_i h^{U}_{i} \right)\:\:, \label{eqffin}
\end{eqnarray}
since the multiplication in ineq. (\ref{mult11}) duplicates part of
the terms in (\ref{eqffin}).
We now plug this bound in eq. (\ref{eqqq1} --- \ref{pqr1}) and finish the derivation
following Massart's finite class Lemma:
\begin{eqnarray}
\exp(t \mu) & \leq &  \left( \frac{2}{2 +  t^2  K_\tr^2}\right)^{|\oddcycle(\varsigma)|}
\cdot \frac{1}{2^m} \sum_{h_\tr \in
  {\mathcal{H}}^\tr_+} \sum_{\ve{\sigma} \in
  \{-1,1\}^{m}} \exp\left(  t \cdot \sum_{i} \sigma_i h_{i} \right)\nonumber\\
 & = & \left( \frac{2}{2 +  t^2 K_\tr^2}\right)^{|\oddcycle(\varsigma)|}
\cdot \sum_{h_\tr \in
  {\mathcal{H}}^\tr_+} \prod_i \left( \frac{\exp(t h_\tr^i) + \exp(-t
    h_\tr^i)}{2} \right) \nonumber\\
 & \leq & \left( \frac{2}{2 +  t^2 K_\tr^2}\right)^{|\oddcycle(\varsigma)|} \cdot
 |{\mathcal{H}}^\tr_+| \exp \left( \frac{t^2}{2} \cdot \sum_i
   \left(\max_{h\in {\mathcal{H}}^\tr_+}
   h^i_\tr\right)^2\right) \label{eqqqin}\\
 & \leq & \left( \frac{2}{2 +  t^2 K_\tr^2}\right)^{|\oddcycle(\varsigma)|} \cdot
 |{\mathcal{H}}^\tr_+| \exp \left( \frac{t^2 m K_\tr^2}{2}\right) \label{lasqt}\:\:.\nonumber
\end{eqnarray}
Ineq. (\ref{eqqqin}) holds because $(\exp(x)+\exp(-x))/2\leq \exp(x^2/2)$.
Taking logs and rearranging yields:
\begin{eqnarray}
\mu & \leq & \frac{1}{t} \cdot \log \frac{|{\mathcal{H}}^\tr_+| }{\left(1+\frac{t^2K_\tr^2}{2}\right)^{|\oddcycle(\varsigma)|} }+ \frac{t m K_\tr^2}{2}\:\:.\label{eq1112}
\end{eqnarray}
Now, suppose that we can choose $t$ such that
\begin{eqnarray}
\frac{t^2K_\tr^2}{2} & \geq & \varepsilon\:\:,\label{eqvarepsilon}
\end{eqnarray}
for some $\varepsilon > 0$. In this case, ineq. (\ref{eq1112}) implies
\begin{eqnarray}
\mu & \leq & \frac{1}{t} \cdot \log \frac{|{\mathcal{H}}^\tr_+| }{\left(1+\varepsilon\right)^{|\oddcycle(\varsigma)|} }+ \frac{t m K_\tr^2}{2}\:\:,\label{eq1113}
\end{eqnarray}
which is of the form $\mu \leq A/t + Bt$ with $B > 0$. Taking $t =
\sqrt{A/B}$ yields, using the fact that each classifier in
${\mathcal{H}}^\tr_+$ :
\begin{eqnarray}
\mu & \leq & K_\tr \cdot \sqrt{ 2 m \log \frac{|{\mathcal{H}}^\tr_+|}{(1+\varepsilon)^{|\oddcycle(\varsigma)| }}}\:\:.\label{latmu}
\end{eqnarray}
Dividing by $m$ gives the statement of Theorem \ref{thradcompSettingB}. We need however to check that ineq. (\ref{eqvarepsilon})
holds, which, since $t =
\sqrt{A/B}$, yields that we must have, after simplification:
\begin{eqnarray}
\log |{\mathcal{H}}^\tr_+| & \geq & \varepsilon m +
|\oddcycle(\varsigma)| \log (1+\varepsilon)\:\:.
\end{eqnarray}
Since we have excluded fixed points, $|\oddcycle(\varsigma)|\leq m/3$,
and since $\log(1+x)\leq x$, a sufficient condition is $\log
|{\mathcal{H}}^\tr_+| \geq 4\varepsilon m/3$, which is the
Theorem's assumption.\\

We now show a bound on the expected \rcp~when  $\matrice{m}$ in
$\process{T}$ is picked uniformly at random, with or without the class
consistency requirement, as a function of the non-fixed points in the
permutations. Following \citep{gftsssAK}, we let $\ii{r}{m}$ denote the
set of $r$-tuples drawn without replacement drawn from $[m]$. We also
let $\ii{r}{m,b}$ denote the
set of $r$-tuples drawn without replacement drawn from $[m]\cap\{i :
y_i = b1\}$, for $b \in \{-,+\}$.

\begin{theorem}\label{thexpradcomp}
Under the joint Settings of Theorem \ref{thradcompUU} and Setting (B), let $S^k_m$ denotes the
set of permutations with exactly $k$ non-fixed points, and $S^{kb}_m$
its subset of block-class permutations with non-fixed points in class $b\in
\{-,+\}$. Then for $S \in\{S^k_m, S^{k-}_m, S^{k+}_m\}$ the following holds
over the uniform sampling of permutations:
\begin{eqnarray}
\expect_{\matrice{m} \sim S}
[\disc_{\process{T}} ({\mathcal{H}})] & \leq & u\cdot \frac{r_\tr}{\sqrt{m}} \cdot
\sqrt{
  \frac{k}{m}\cdot \mathscr{Q}}\:\:,\label{pprc}
\end{eqnarray}
where $\mathscr{Q} \defeq \expect_{(i,i')\sim \ii{2}{m}}
[\|\ve{x}^\tr_i-\ve{x}^\tr_{i'}\|_2^2]$ if $S =
S^k_m$, and $\mathscr{Q} \defeq \expect_{(i,i')\sim \ii{2}{m,b} }
[\|\ve{x}^\tr_i-\ve{x}^\tr_{i'}\|_2^2]$ if $S = S^{kb}_m$ ($b\in \{-,+\}$).
\end{theorem}
\begin{proof}
We make the proof for $S = S^k_m$. The two other cases follow in the
same way. 
We have from Theorem
\ref{thradcompUU}, because of Jensen inequality:
\begin{eqnarray}
\expect_{\matrice{m} \sim S^{k}_m} \left[\disc_{\process{T}} ({\mathcal{H}})\right] & \leq & u \cdot \frac{r_\tr}{\sqrt{m}} \cdot
\expect_{\matrice{m} \sim S^{k}_m} \left[\sqrt{\frac{1}{m} \cdot
\cipr{\matrice{i}_m}{\matrice{k}^\tr}{\matrice{m}}}
\right]\nonumber\\
 & \leq & u \cdot \frac{r_\tr}{\sqrt{m}} \cdot
\sqrt{\frac{1}{m} \cdot
  \expect_{\matrice{m} \sim S^{k}_m} \left[\cipr{\matrice{i}_m}{\matrice{k}^\tr}{\matrice{m}} \right]}
\:\:.
 \end{eqnarray}
We now decompose the expectation inside and first condition on the set of
permutations whose set of non fixed points are the same set of $k$
examples, say for $i \in [k]$. Let us call $S^{k*}_m$ this subset
of $S^{k}_m$. In this case, we obtain:
\begin{eqnarray}
\lefteqn{\expect_{\matrice{m} \sim S^{k*}_m} \left[\cipr{\matrice{i}_m}{\matrice{k}^\tr}{\matrice{m}}\right]}\nonumber\\
 & = & \expect_{\matrice{m} \sim S^{k*}_m} \left[\trace{(\matrice{i}_m -
  \matrice{m}) \matrice{k}^\tr (\matrice{i}_m - \matrice{m})^\top
}\right]\nonumber\\
 & = & \trace{\matrice{k}^\tr} +  \expect_{\matrice{m} \sim S^{k*}_m}
 \left[\trace{\matrice{m} \matrice{k}^\tr \matrice{m}^\top }\right] -
 2\cdot \expect_{\matrice{m} \sim S^{k*}_m}
 \left[\trace{\matrice{m} \matrice{k}^\tr}\right]\label{prop11X}\\
& = & 2 \cdot \left(\trace{\matrice{k}^\tr} - \trace{\expect_{\matrice{m} \sim S^{k*}_m}
 \left[\matrice{m}\right] \matrice{k}^\tr}\right)\nonumber\\
 & = & 2 \cdot \left(\sum_{i\in [m]} \matrice{k}^\tr_{ii} - \sum_{i\in [k]}
   \frac{1}{k} \cdot \sum_{i'\in [k]} {\matrice{k}^\tr_{ii'}} -
   \sum_{i\in [m]\backslash [k]} \matrice{k}^\tr_{ii}\right)\label{prop13}\\
 & = & 2 \cdot \left(\frac{k-1}{k} \cdot \sum_{i\in [k]}
   \matrice{k}^\tr_{ii} - \frac{1}{k} \cdot \sum_{(i,i') \in \ii{2}{k}}
   {\matrice{k}^\tr_{ii'}}\right)\nonumber \\
 & = & \frac{2}{k} \cdot \left(\sum_{(i,i') \in \ii{2}{k}}
   {\frac{\matrice{k}^\tr_{ii}+\matrice{k}^\tr_{i'i'}}{2} - \matrice{k}^\tr_{ii'}}\right)\nonumber \\
 & = & \frac{1}{k} \cdot \sum_{(i,i') \in \ii{2}{k}}
   {\|\ve{x}^\tr_i-\ve{x}^\tr_{i'}\|_2^2}\label{prop14}\:\:.
\end{eqnarray}
In eq. (\ref{prop11X}) we use the fact that $\matrice{k}^\tr$ is
symmetric. Eq. (\ref{prop13}) uses the fact that
\begin{eqnarray}
\expect_{\matrice{m} \sim S^{k*}_m}
 \left[\matrice{m}\right] & = & 
\left[\begin{array}{ccc}
\matrice{u}_k & | & \matrice{0}\\
\matrice{0} & | & \matrice{i}_{m-k}
\end{array}
\right]\:\:,\nonumber
\end{eqnarray}
where we recall that $\matrice{u}_k \defeq \frac{1}{k} \cdot
\ve{1}\ve{1}^\top$ (main file, Definition \ref{defhs}).
There remains to average eq. (\ref{prop14}) over the set of all
permutations whose set of fixed points is a different $(m-k)$-subset of
$[m]$ and the statement of Theorem \ref{thexpradcomp} is proven for $S
= S^k_m$.
\end{proof}
The key
point in the bound is factor $k/m$, which implies that when
permutations have lots of fixed points, say $(1-\Omega(1))m$, then the
\rcp~may just vanish (as $m$ increases) wrt the
Rademacher complexity, whose dependency on $m$ is $\Omega(1/\sqrt{m})$
\citep{kstOT}.

\subsection{Proof of Theorem \ref{thspectrA}}\label{proof_thm_thspectrA}

The proof stems from the following Theorem, which just assumes that
$\matrice{k}^u$ and $\matrice{k}^v$ can be diagonalized (hence, it is
applies to a more general setting than kernel functions).
\begin{theorem}\label{thspectr1}
Let $\matrice{k}^u$ and $\matrice{k}^v$ be two diagonalisable matrices with respective eigendecomposition $\{\lambda_i,
\ve{u}_i\}_{i\in [d]}$ and $\{\mu_i,
\ve{v}_i\}_{i\in [d]}$, with eigenvalues eventually duplicated up to
their algebraic multiplicity. Letting $\overline{a}\defeq (1/m)
\ve{1}^\top\ve{a}$ denote the average coordinate in $\ve{a}$, the
difference in Hilbert-Schmidt Independence Criterion with respect to shuffling
$\matrice{m}$ satisfies:
\begin{eqnarray}
\hsic(\matrice{k}^u, \matrice{k}^v) - \hsic(\matrice{k}^u,
\matrice{m}\matrice{k}^v \matrice{m}^\top) & = & - 2 m \cdot \left(\sum_{i}\lambda_i\overline{u}_i
 \ve{u}_i\right)^\top(\matrice{i}_m - \matrice{m}) \left(\sum_{i} \mu_i \overline{v}_i\ve{v}_i\right)\:\:.
\end{eqnarray}
Hence, if $\matrice{m}\in S^{\mbox{\tiny{e}}}_m$ permutes $\ell$ and
$\ell'$ in $[m]$, then $\hsic(\matrice{k}^u,
\matrice{m}\matrice{k}^v \matrice{m}^\top) > \hsic(\matrice{k}^u,
\matrice{k}^v)$ iff:
\begin{eqnarray}
\left( \sum_i \lambda_i\overline{u}_i
 (u_{i\ell}-u_{i\ell'})\right) \left( \sum_i \mu_i\overline{v}_i
 (v_{i\ell}-v_{i\ell'})\right) & > & 0 \:\:.\label{ppp3}
\end{eqnarray}
\end{theorem}
\begin{proof}
Being symmetric, $\matrice{k}^u$ and $\matrice{k}^v$ can be
diagonalized as $\matrice{k}^u = \sum_{i} \lambda_i
\ve{u}_i\ve{u}_i^\top$ and $\matrice{k}^v = \sum_{i} \mu_i
\ve{v}_i\ve{v}_i^\top$. We use definition $\matrice{u}_m \defeq (1/m)\ve{1}\ve{1}^\top$ for short. The
following is folklore or can be can be checked after analytic
derivations that we omit:
\begin{eqnarray}
\trace{\matrice{u}_m \matrice{k}^u
  \matrice{k}^v} & = & m \left(\sum_i{\lambda_i \overline{u}_i
    \ve{u}_i}\right)\left(\sum_i{\mu_i \overline{v}_i \ve{v}_i}\right)
\nonumber \\
 & = & \trace{\matrice{k}^u \matrice{u}_m 
  \matrice{k}^v}\nonumber\\
\trace{\matrice{u}_m \matrice{k}^u\matrice{u}_m \matrice{k}^v} & = & m^2 \left(\sum_i
\lambda_i (\overline{u}_i)^2\right) \left(\sum_i
\mu_i (\overline{v}_i)^2\right) \nonumber\\
\trace{\matrice{k}^u \matrice{k}^v} & = & \sum_i \lambda_i \mu_i\:\:.
\end{eqnarray}
We thus get
\begin{eqnarray}
 \hsic(\matrice{k}^u, \matrice{k}^v) & = & \cipr{\matrice{k}^u}{\matrice{k}^v}{\matrice{u}_m}\nonumber\\
 & = & \trace{(\matrice{i}_m - \matrice{u}_m) \matrice{k}^u (\matrice{i}_m -
  \matrice{u}_m) \matrice{k}^v} \nonumber\\
 & = & \trace{\matrice{k}^u
  \matrice{k}^v} - \trace{\matrice{u}_m \matrice{k}^u
  \matrice{k}^v}- \trace{\matrice{k}^u \matrice{u}_m 
  \matrice{k}^v} + \trace{\matrice{u}_m \matrice{k}^u\matrice{u}_m \matrice{k}^v}\nonumber\\
 & = & \sum_i \lambda_i \mu_i - 2 m \cdot
\left(\sum_i{\lambda_i \overline{u}_i
    \ve{u}_i}\right)^\top \left(\sum_i{\mu_i \overline{v}_i
    \ve{v}_i}\right) \nonumber\\
 & &+ m^2 \cdot \left(\sum_i
\lambda_i (\overline{u}_i)^2\right) \cdot\left(\sum_i
\mu_i (\overline{v}_i)^2\right)\nonumber\\ 
 & = & m^2 \left( \frac{1}{m^2}\cdot \sum_i \lambda_i \mu_i - \frac{2}{m}\sum_{i,j}\lambda_i \overline{u}_i
    \ve{u}^\top_i \mu_j \overline{v}_j
    \ve{v}_j + \sum_{i,j}\lambda_i \mu_j  (\overline{u}_i)^2
    (\overline{v}_j)^2\right) \nonumber\\
 & = & m^2\left( \frac{1}{m^2}\cdot \sum_i \lambda_i \mu_i -
   \frac{1}{m^2}\cdot \sum_{i,j} \lambda_i \mu_j(\ve{u}_i^\top\ve{v}_j)^2  +
   \sum_{i,j}\lambda_i\mu_j\left(\frac{1}{m}\ve{u}_i^\top\ve{v}_j - \overline{u}_i\overline{v}_j\right)^2\right) \nonumber\\
 & = & m^2\left( \frac{1}{m^2}\cdot \ve{\lambda}^\top \left(
     \matrice{i} - \matrice{c}\right) \ve{\mu} + \ve{\lambda}^\top
   \matrice{j} \ve{\mu}\right)\label{lll2e}\\
 & = & \ve{\lambda}^\top \left(
     \left(\matrice{i} - \matrice{c}\right) + m^2 \cdot
     \matrice{j}\right) \ve{\mu}\:\:.\label{lp1p}
\end{eqnarray}
We have used here the square cosine matrix $\matrice{c}$ with
$\matrice{c}_{ij} \defeq \cos^2(\ve{u}_i, \ve{u}_j)$, and the square
correlation matrix $\matrice{j}$ with
$\matrice{j}_{ij} \defeq ((1/m)\ve{u}_i^\top\ve{v}_j - \overline{u}_i\overline{v}_j)^2$.
Now, suppose we perform \cp~$\process{T}$ with shuffling matrix
$\matrice{m}$. $\matrice{k}^v$ and its eigendecomposition become after
shuffling 
\begin{eqnarray}
\matrice{m}\matrice{k}^v \matrice{m}^\top & = & \sum_{i} \mu_i
(\matrice{m}\ve{v}_i)(\matrice{m}\ve{v}_i)^\top\:\:.
\end{eqnarray}
Remark that shuffling affects the order in the coordinate of
\textit{all} eigenvectors. So the difference between the two Hilbert-Schmidt Independence
Criteria (before - after shuffling) is:
\begin{eqnarray}
\lefteqn{\hsic(\matrice{k}^u, \matrice{k}^v) - \hsic(\matrice{k}^u,
\matrice{m}\matrice{k}^v \matrice{m}^\top)}\nonumber\\
 & = & 
   \sum_{i,j} \lambda_i
   \mu_j((\ve{u}_i^\top
   \matrice{m}\ve{v}_j)^2  - (\ve{u}_i^\top\ve{v}_j)^2)  \nonumber\\ 
& & - \sum_{i,j}\lambda_i\mu_j\left\{\left(\ve{u}_i^\top
    \matrice{m}\ve{v}_j -
    m\overline{u}_i\overline{v}_j\right)^2-\left(\ve{u}_i^\top\ve{v}_j
    - m\overline{u}_i\overline{v}_j\right)^2\right\}\nonumber\\
 & = & 
   \sum_{i,j} \lambda_i
   \mu_j \cdot \ve{u}_i^\top(\matrice{m}-\matrice{i}_m)\ve{v}_j \cdot \ve{u}_i^\top(\matrice{m}+\matrice{i}_m)\ve{v}_j  \nonumber\\ 
& & - \sum_{i,j}\lambda_i\mu_j\left\{
  \ve{u}_i^\top(\matrice{m}-\matrice{i}_m)\ve{v}_j \cdot
  (\ve{u}_i^\top(\matrice{m}+\matrice{i}_m)\ve{v}_j - 2 m\overline{u}_i\overline{v}_j) \right\}\nonumber\\
 & = & 2 m \cdot \sum_{i,j}(\lambda_i\overline{u}_i)
 \ve{u}_i^\top(\matrice{m}-\matrice{i}_m) (\mu_j \overline{v}_j) \nonumber\\
 & = & 2 m \cdot \left(\sum_{i}\lambda_i\overline{u}_i
 \ve{u}_i\right)^\top(\matrice{m}-\matrice{i}_m) \left(\sum_{i} \mu_i \overline{v}_i\ve{v}_i\right)\:\:.
\end{eqnarray}
We now remark that whenever $\matrice{m}\in S^{\mbox{\tiny{e}}}_m$, if
it permutes $\ell$ and $\ell'$ in $[m]$, then
$\ve{a}(\matrice{m}-\matrice{i}_m) \ve{b} = a_\ell(b_{\ell'} -
b_{\ell}) + a_{\ell'}(b_{\ell} -
b_{\ell'}) = - (a_{\ell} -
a_{\ell'}) (b_{\ell} -
b_{\ell'})$, so we get:
\begin{eqnarray}
\lefteqn{\hsic(\matrice{k}^u, \matrice{k}^v) - \hsic(\matrice{k}^u,
\matrice{m}\matrice{k}^v \matrice{m}^\top)}\nonumber\\
 & = & - 2m \left( \sum_i \lambda_i\overline{u}_i
 (u_{i\ell}-u_{i\ell'})\right) \left( \sum_i \mu_i\overline{v}_i
 (v_{i\ell}-v_{i\ell'})\right)\:\:,
\end{eqnarray}
and we get ineq. (\ref{ppp3}).
\end{proof}
This ends the proof of Theorem \ref{thspectrA}.

\subsection{Proof of Theorem \ref{thspectrB}}\label{proof_thm_thspectrB}

The Theorem is a direct consequence of the following Theorem.
\begin{theorem}\label{thhsic}
Let $\matrice{k}^u$ and $\matrice{k}^v$ be two kernel functions over
${\mathcal{S}}$. Then for any elementary permutation $\matrice{m} \in
S^{\mbox{\tiny{e}}}_m$ that permutes $\ell$ and $\ell'$ in $[m]$,
\begin{eqnarray}
\hsic(\matrice{k}^u, \matrice{m}\matrice{k}^v \matrice{m}^\top) -
\hsic(\matrice{k}^u, \matrice{k}^v) & = & - 2m\cdot \cov(\ve{\updelta}_{\ell\ell'}^u,
\ve{\updelta}_{\ell\ell'}^v)  + \mathscr{R}^{u,v}_{\ell\ell'}
\:\:,\label{fff1}
\end{eqnarray}
with 
\begin{eqnarray}
\mathscr{R}^{u,v}_{\ell\ell'} & \defeq & (\matrice{k}^u_{\ell \ell} - \matrice{k}^u_{\ell' \ell} )(\matrice{k}^v_{\ell \ell} -
 \matrice{k}^v_{\ell' \ell}) +  (\matrice{k}^u_{\ell' \ell'} - \matrice{k}^u_{\ell' \ell} )(\matrice{k}^v_{\ell' \ell'} -
 \matrice{k}^v_{\ell' \ell}) \:\:.
\end{eqnarray} 
Furthermore, the uniform
sampling of elementary permutations in $S^{\mbox{\tiny{e}}}_m$ satisfies:
\begin{eqnarray}
\expect_{\matrice{m} \sim S^{\mbox{\tiny{e}}}_m}
  \left[\hsic(\matrice{k}^u, \matrice{m}\matrice{k}^v
    \matrice{m}^\top) \right] & = & \left(1 - \frac{8}{m-1}\right)
  \cdot \hsic(\matrice{k}^u, \matrice{k}^v) + \frac{8}{m-1}\cdot
  \mathscr{R}^{u,v}\:\:,\label{fff2}
\end{eqnarray}
with
\begin{eqnarray}
\mathscr{R}^{u,v} & \defeq & \sum_{i}\matrice{k}^u_{i i}
\matrice{k}^v_{i i}  - \frac{1}{m}\cdot \left(\frac{\sum_{i}\matrice{k}^u_{i i}
\matrice{k}^v_{. i} +\sum_{i}\matrice{k}^u_{. i}
\matrice{k}^v_{i i}}{2}\right)\:\:.\label{defrrXX}
\end{eqnarray}
Here, when replacing an index notation by a point,
``.'', we denote a sum over all possible values of this index.
\end{theorem}
\begin{proof}
We first decompose $\hsic(\matrice{k}^u, \matrice{k}^v)$:
\begin{eqnarray}
\hsic(\matrice{k}^u, \matrice{k}^v) & = & \sum_{i,i'}
\matrice{k}^u_{ii'}\matrice{k}^v_{ii'} - \frac{2}{m}\cdot \sum_i
\matrice{k}^u_{i.}\matrice{k}^v_{i.} +
\frac{1}{m^2}\matrice{k}^u_{..}\matrice{k}^u_{..}\:\:.
\end{eqnarray}
for any $(\ell, \ell')\in \ii{2}{m}$. For any $\matrice{m} \in S^{\mbox{\tiny{e}}}_m$ denoting an elementary
permutation $\varsigma$ of the features in ${\mathcal{F}}^\tr$ such that
${\mathcal{V}} \subseteq {\mathcal{F}}^\tr$ and $\varsigma(\ell) =
\ell'$, $\varsigma(\ell') = \ell$, we obtain:
\begin{eqnarray}
\lefteqn{\hsic(\matrice{k}^u, \matrice{m}\matrice{k}^v \matrice{m}^\top) -
\hsic(\matrice{k}^u, \matrice{k}^v)}\nonumber\\
 & = & 2\cdot \left( \sum_{i \neq
    \ell, \ell'}
\matrice{k}^u_{\ell i}\matrice{k}^v_{\ell' i} +\sum_{i \neq
    \ell, \ell'}
\matrice{k}^u_{\ell' i}\matrice{k}^v_{\ell i} - \sum_{i \neq
    \ell, \ell'}
\matrice{k}^u_{\ell i}\matrice{k}^v_{\ell i} - \sum_{i \neq
    \ell, \ell'}
\matrice{k}^u_{\ell' i}\matrice{k}^v_{\ell' i}\right)\nonumber\\
 & & -\frac{2}{m} \cdot \matrice{k}^u_{\ell .}\matrice{k}^v_{\ell'
   .}-\frac{2}{m} \cdot \matrice{k}^u_{\ell' .}\matrice{k}^v_{\ell
   .}+\frac{2}{m} \cdot \matrice{k}^u_{\ell .}\matrice{k}^v_{\ell
   .}+\frac{2}{m} \cdot \matrice{k}^u_{\ell' .}\matrice{k}^v_{\ell'
   .}\nonumber\\
& = & -2 \left(\sum_{i}
(\matrice{k}^u_{\ell i} - \matrice{k}^u_{\ell' i} )(\matrice{k}^v_{\ell i} - \matrice{k}^v_{\ell'
  i}) - \frac{1}{m}\cdot (\matrice{k}^u_{\ell .}-\matrice{k}^u_{\ell'
  .})(\matrice{k}^v_{\ell .}-\matrice{k}^v_{\ell' .})\right)
\nonumber\\
 & &+  (\matrice{k}^u_{\ell \ell} - \matrice{k}^u_{\ell' \ell} )(\matrice{k}^v_{\ell \ell} -
 \matrice{k}^v_{\ell' \ell}) +  (\matrice{k}^u_{\ell' \ell'} - \matrice{k}^u_{\ell' \ell} )(\matrice{k}^v_{\ell' \ell'} -
 \matrice{k}^v_{\ell' \ell}) \nonumber\\
 & = & - 2m\cdot \cov(\ve{\updelta}_{\ell\ell'}^u,
\ve{\updelta}_{\ell\ell'}^v) + \mathscr{R}^{u,v}_{\ell\ell'}\:\:,\nonumber
\end{eqnarray}
which is eq. (\ref{fff1}). We also have:
\begin{eqnarray}
\expect_{\matrice{m} \sim S^{\mbox{\tiny{e}}}_m} \left[\sum_{i}
(\matrice{k}^u_{\ell i} - \matrice{k}^u_{\ell' i} )(\matrice{k}^v_{\ell i} - \matrice{k}^v_{\ell'
  i})\right] & = & \frac{4}{m}\cdot \sum_{i,i'}
\matrice{k}^u_{ii'}\matrice{k}^v_{ii'} - \frac{4}{m(m-1)}\cdot \sum_i
\matrice{k}^u_{i.}\matrice{k}^v_{i.} \nonumber\\
 & & + \frac{4}{m(m-1)}\cdot \sum_{i,i'}
\matrice{k}^u_{ii'}\matrice{k}^v_{ii'} \nonumber\\
 & = & \frac{4}{m-1}\cdot \sum_{i,i'}
\matrice{k}^u_{ii'}\matrice{k}^v_{ii'} - \frac{4}{m(m-1)}\cdot \sum_i
\matrice{k}^u_{i.}\matrice{k}^v_{i.} \:\:,\nonumber
\end{eqnarray}
\begin{eqnarray}
\expect_{\matrice{m} \sim S^{\mbox{\tiny{e}}}_m} \left[(\matrice{k}^u_{\ell .}-\matrice{k}^u_{\ell'
  .})(\matrice{k}^v_{\ell .}-\matrice{k}^v_{\ell' .})\right] & = &
\frac{4}{m} \cdot \sum_i \matrice{k}^u_{i .}\matrice{k}^v_{i .}-
\frac{4}{m(m-1)}\cdot \sum_{i \in [m]} \matrice{k}^u_{i .}
\sum_{i' \in [m] \backslash\{i\}}  \matrice{k}^v_{i'
  .}\nonumber\\
 & = & \frac{4}{m} \cdot \sum_i \matrice{k}^u_{i
   .}\matrice{k}^v_{i .} -
\frac{4}{m(m-1)}\cdot \sum_{i} \matrice{k}^u_{i
  .}\cdot\left(\matrice{k}^v_{..} - \matrice{k}^v_{i
  .}\right)\nonumber\\
 & = & \frac{4}{m} \cdot \sum_i \matrice{k}^u_{i
   .}\matrice{k}^v_{i .} -
\frac{4}{m(m-1)}\matrice{k}^u_{..}\matrice{k}^v_{..} + \frac{4}{m(m-1)} \cdot \sum_i \matrice{k}^u_{i
   .}\matrice{k}^v_{i .} \nonumber\\
 & = & \frac{4}{m-1} \cdot \sum_i \matrice{k}^u_{i
   .}\matrice{k}^v_{i .} -
\frac{4}{m(m-1)}\matrice{k}^u_{..}\matrice{k}^v_{..}\:\:,\nonumber
\end{eqnarray}
and finally
\begin{eqnarray}
\expect_{\matrice{m} \sim S^{\mbox{\tiny{e}}}_m}
\left[\mathscr{R}^{u,v}_{\ell\ell'}\right] & = & \frac{8}{m-1}\cdot \left[\sum_{i}\matrice{k}^u_{i i}
\matrice{k}^v_{i i}  - \frac{1}{m}\cdot \left(\frac{\sum_{i}\matrice{k}^u_{i i}
\matrice{k}^v_{. i} +\sum_{i}\matrice{k}^u_{. i}
\matrice{k}^v_{i i}}{2} \right)\right]\nonumber\\
 & \defeq & \frac{8}{m-1}\cdot \mathscr{R}^{u,v}\:\:,\label{equ}
\end{eqnarray}
since
\begin{eqnarray}
\lefteqn{\expect_{\matrice{m} \sim S^{\mbox{\tiny{e}}}_m} \left[ (\matrice{k}^u_{\ell \ell} - \matrice{k}^u_{\ell' \ell} )(\matrice{k}^v_{\ell \ell} -
 \matrice{k}^v_{\ell' \ell})\right]}\nonumber\\
 & = &
\frac{4}{m(m-1)}\sum_{i}\sum_{i'\neq i}\matrice{k}^u_{i i}
\matrice{k}^v_{i i} - \frac{2}{m(m-1)}\sum_{i}\sum_{i'\neq i}\matrice{k}^u_{i i}
\matrice{k}^v_{i' i}- \frac{2}{m(m-1)}\sum_{i}\sum_{i'\neq i}\matrice{k}^u_{i' i}
\matrice{k}^v_{i i}\nonumber\\
 & = &
\frac{4}{m}\sum_{i}\matrice{k}^u_{i i}
\matrice{k}^v_{i i} - \frac{2}{m(m-1)}\sum_{i}\matrice{k}^u_{i
  i}(\matrice{k}^v_{. i} - \matrice{k}^v_{i i}) -
\frac{2}{m(m-1)}\sum_{i}(\matrice{k}^u_{i.} - \matrice{k}^u_{ii})
\matrice{k}^v_{i i}\nonumber\\
 & = &
\frac{4}{m}\sum_{i}\matrice{k}^u_{i i}
\matrice{k}^v_{i i} + \frac{4}{m(m-1)}\sum_{i}\matrice{k}^u_{i i}
\matrice{k}^v_{i i} - \frac{2}{m(m-1)}\sum_{i}\matrice{k}^u_{i
  i}\matrice{k}^v_{. i}- \frac{2}{m(m-1)}\sum_{i}\matrice{k}^u_{i
  .}\matrice{k}^v_{i i}\nonumber\\
 & = &
\frac{4}{m-1}\sum_{i}\matrice{k}^u_{i i}
\matrice{k}^v_{i i} - \frac{2}{m(m-1)}\sum_{i}\matrice{k}^u_{i
  i}\matrice{k}^v_{. i}- \frac{2}{m(m-1)}\sum_{i}\matrice{k}^u_{i
  .}\matrice{k}^v_{i i}\nonumber\\
 & = &\expect_{\matrice{m} \sim S^{\mbox{\tiny{e}}}_m} \left[ (\matrice{k}^u_{\ell' \ell'} - \matrice{k}^u_{\ell' \ell} )(\matrice{k}^v_{\ell' \ell'} -
 \matrice{k}^v_{\ell' \ell}) \right]\:\:.\nonumber
\end{eqnarray}
So we obtain
\begin{eqnarray}
\lefteqn{\expect_{\matrice{m} \sim S^{\mbox{\tiny{e}}}_m} \left[\hsic(\matrice{k}^u, \matrice{m}\matrice{k}^v \matrice{m}^\top) -
\hsic(\matrice{k}^u, \matrice{k}^v) \right]}\nonumber\\
 & = &  -\frac{8}{m-1}\cdot \sum_{i,i'}
\matrice{k}^u_{ii'}\matrice{k}^v_{ii'} + \frac{8}{m(m-1)}\cdot \sum_i
\matrice{k}^u_{i.}\matrice{k}^v_{i.} \nonumber\\
 & & + \frac{8}{m(m-1)} \cdot \sum_i \matrice{k}^u_{i
   .}\matrice{k}^v_{i .} -
\frac{8}{m^2(m-1)}\matrice{k}^u_{..}\matrice{k}^v_{..} + \frac{8}{m-1}\cdot \mathscr{R}^{u,v}\nonumber\\
 & = & -\frac{8}{m-1} \cdot \hsic(\matrice{k}^u, \matrice{k}^v)  + \frac{8}{m-1}\cdot \mathscr{R}^{u,v}\:\:.\nonumber
\end{eqnarray}
This ends the proof of Theorem \ref{thhsic}.
\end{proof}
When kernel functions have unit diagonal (such as for the Gaussian
kernel), eq. (\ref{defrrXX}) simplifies to:
\begin{eqnarray}
\mathscr{R}^{u,v} & \defeq & m\left(1 - \frac{1}{m^2}\cdot
  \left(\frac{\matrice{k}^u_{..} + \matrice{k}^v_{..}}{2}\right)\right)\:\:.\label{defrr}
\end{eqnarray}
Hence, provided we perform $T \defeq \epsilon m$ elementary
permutations, there exists a sequence of such permutations such that the
composition $\matrice{m}_* \defeq \matrice{m}_T \matrice{m}_{T-1}
\cdots \matrice{m}_1$ satisfies:
\begin{eqnarray}
\hsic(\matrice{k}^u, \matrice{m}_*\matrice{k}^v \matrice{m}_*^\top) &
\leq & \left(1 - \frac{8}{m-1}\right)^{\epsilon m}
  \cdot \hsic(\matrice{k}^u, \matrice{k}^v) + \left[1 - \left(1 -
      \frac{8}{m-1}\right)^{\epsilon m}\right]\cdot \mathscr{R}^{u,v} \nonumber\\
 & \leq & \alpha(\epsilon)\cdot \hsic(\matrice{k}^u,
 \matrice{k}^v) + \left(1 - \alpha(\epsilon)\right) \cdot \mathscr{R}^{u,v} \:\:,
\end{eqnarray}
with
\begin{eqnarray}
\alpha(\epsilon) & \defeq & \exp\left(-8 \epsilon\right)\:\:,
\end{eqnarray}
as long as $\hsic(\matrice{k}^u, \matrice{k}^v)\geq
\mathscr{R}^{u,v}$. We have used the fact that
\begin{eqnarray}
\left(1 - \frac{8}{m-1}\right)^{\epsilon m} & \leq & \exp\left( -
  \frac{8\epsilon m}{m-1}\right)\nonumber\\
 & \leq & \exp(-8\epsilon)\:\:.
\end{eqnarray}
This achieves the proof of Theorem \ref{thspectrB}.

\subsection{The Cornia-Mooij model and results}\label{pres_dcm}

\begin{figure}[t]
\begin{center}
\includegraphics[width=0.5\columnwidth]{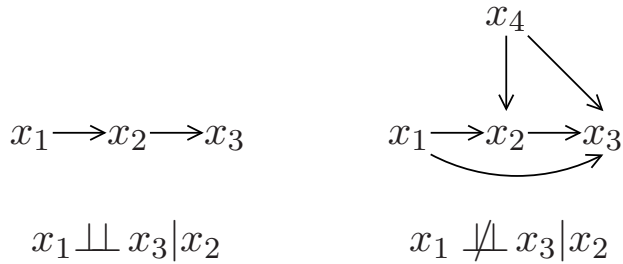}
\caption{The Cornia-Mooij model \citep{cmT2}. Left: belief, right: true model.}
\label{f-settingCM}
\end{center}
\end{figure}

We now show how to trick statistical tests into keeping independence \textit{and}
then incur arbitrarily large errors in estimating causal effects,
\textit{via} the use of \cp s. 
The
model we refer to is the Cornia-Mooij (CM) model \citep{cmT2},
shown in Figure \ref{f-settingCM}. In
the CM model, there are $d=3$
observation variables, and a true model which relies on a weak
conditional dependence $x_1 \notindep x_3 | x_2$. \citep{cmT2} show that \textit{if} one keeps the
independence assumption $H_0$ that $x_1 \indep x_3 | x_2$, this can lead to very high causal
estimation errors, as measured by
  $|\expect[x_3|x_2]-\expect[x_3|\mathrm{do}(x_2)]|/|x_2|$
  \citep{cmT2}. We show that the \cp~is precisely able to trick
  statistics into keeping $H_0$. 

There is a hidden confounder $x_4$, which is assumed
to be independent from $x_1$. The true model makes the following
statistical dependence assumptions:
\begin{itemize}
\item $x_1 \notindep x_2$,
\item $x_2 \notindep x_3$,
\item and the most important one, which we scramble through the \cp,
  $x_1 \notindep x_3 | x_2$.
\end{itemize}
We chose this simple model because (a) it belongs to the few worst-case
models for causality analysis, and (b) it shows, in
addition to jamming (non)linear correlations, how \cp~can
also jam partial correlations. 
In the CM model, there are $d=3$
observation variables, and a true model which relies on a weak
conditional dependence $x_1 \notindep x_3 | x_2$. \citep{cmT2} show that \textit{if} one keeps the
independence assumption $H_0$ that $x_1 \indep x_3 | x_2$, this can lead to very high causal
estimation errors\footnote{As measured by
  $|\expect[x_3|x_2]-\expect[x_3|\mathrm{do}(x_2)]|/|x_2|$
  \citep{cmT2}.}. We show that it is possible, through a \cp, to trick statistics into \textit{keeping} $H_0$ \textit{as well}. In the
following, ${\mathcal{F}}_\tr =
\{x_3\}$. It is shown in \citep{cmT2} that
  $x_1 \indep x_3 | x_2$ iff the partial
    correlation  $\rho_{(13) \cdot2}\defeq (\rho_{13} - \rho_{12}\rho_{23})/\sqrt{(1 -
    \rho_{12}^2)(1 - \rho_{23}^2)}$ vanishes.
Assuming $\rho_{(13) \cdot2}$ is large enough in the dataset we
have (so that we would reject $H_0$ from observing ${\mathcal{S}}$), we show how to reduce it
through a sequence of \cp s, using a similar strategy as in Theorem
\ref{thspectrB}, the main difference being that we rely on 
block-class permutations. For any $\varsigma  \in S_m$, notation $3^\varsigma$ indicates column variable 3 shuffled. We
assume $\rho_{(13) \cdot2} > 0$ (the same
analysis can be done if $\rho_{(13) \cdot2} < 0$). 

\begin{theorem}\label{thmcm1}
Suppose that there exists $\epsilon> 0$ such that $\rho^2_{12} \leq
1-\epsilon$ and $\rho^2_{23^{\varsigma}} \leq
1-\epsilon$ for any $\varsigma
\in S^*_m$. Then there exists $T>0$
and a
sequence of $T$ elementary permutations in $S^*_m$ such that
$\rho_{(13^\varsigma) \cdot2}$ is strictly decreasing in the sequence
and meets at the end $\rho_{(13^\varsigma) \cdot2} \leq
\mathscr{R}$ with 
\begin{eqnarray*}
\mathscr{R} & \defeq & (1-\epsilon)^{-1}\cdot p_+(1-p_+) \cdot\left(\tilde{\mu}_1 -
  \rho_{12} \cdot \tilde{\mu}_2\right)\cdot\tilde{\mu}_3\:\:,
\end{eqnarray*}
and $\tilde{\mu}_{j}\defeq (1/(2\sqrt{v_j})) \cdot \sum_{y'} y'
\expect_{(\ve{x}, y) \sim {\mathcal{S}}}[x_j | y=y']$.
\end{theorem}
\begin{proof}

The Theorem is a direct consequence of the following Lemma.
\begin{lemma}\label{lemperm}
Let $c_{jk}$ denote the covariance between columns $j$ and $k$, $\mu^b_l \defeq (1/m_b)\sum_{i : y_i = b} x_{il}$ and
$p \defeq m_+/m$. Suppose that $j\in {\mathcal{F}}_\re$ and $k\in
{\mathcal{F}}_\tr$. Then as long as 
\begin{eqnarray}
c_{jk} & > & p(1-p) \cdot (\mu^+_j - \mu^-_j) \cdot (\mu^+_k - \mu^-_k)\:\:,
\end{eqnarray}
there always exist $\varsigma \in S^*_m$ such that
$\rho_{jk^\varsigma} < \rho_{jk} $, where $k^\varsigma$ denote column
variable $k$ shuffled according to $\varsigma$ in the corresponding \cp. 
\end{lemma}
\begin{proof}
Let us denote for short $\ve{c}_j \in {\mathbb{R}}^m$ the $j^{th}$ feature
column. We have because of the fact that $\mu_{k^\varsigma}=\mu_{k}$
and $v_{k^\varsigma}=v_{k}$ ($v_.$ being the variance):
\begin{eqnarray}
\rho_{jk^{\varsigma}} - \rho_{jk} & = & 
\frac{1}{m\sqrt{v_jv_k}}\cdot\left(\ve{c}_j^\top
\left(\matrice{m}_{\varsigma} - \matrice{i}_m\right) \ve{c}_k\right)\:\:,
\end{eqnarray}
where $\matrice{m}_{\varsigma}$ is the shuffling matrix of permutation $\varsigma$.
Matrix $\matrice{m}_{\varsigma} - \matrice{i}_{m}$ has only
four non-zero coordinates: in $(\ell,\ell)$ and
$(\ell',\ell')$ (both $-1$), and in $(\ell,\ell')$ and
$(\ell',\ell)$ (both $1$), so we get:
\begin{eqnarray}
\rho_{jk^{\varsigma}} - \rho_{jk} & = & 
\frac{1}{m\sqrt{v_jv_k}}\cdot\left(x_{\ell j}(x_{\ell'k}-x_{\ell k})+x_{\ell'j}(x_{\ell k}-x_{\ell'k})\right)\nonumber\\
 & = &
 - \frac{1}{m\sqrt{v_jv_k}}\cdot\left((x_{\ell j} - x_{\ell'j})(x_{\ell k}-x_{\ell'k})\right)\:\:.
\end{eqnarray}
Hence, $\rho_{jk^{\varsigma}} < \rho_{jk}$
iff the sign of $x_{\ell j} - x_{\ell'j}$ is the same as the sign of
$x_{\ell k}-x_{\ell'k}$. Let $\uppi^b(\ell)$ the predicate $\rho_{jk^{\varsigma}} \geq
\rho_{jk}$, for any elementary permutation $\varsigma \in S^*_m$ that changes $\ell$ to index
$\ell'$ of the same class $b$ ($\varsigma(\ell) = \ell'$, $\varsigma(\ell')=\ell$). If $\uppi^b(\ell)$ is true, then, averaging
over all such permutations, we obtain:
\begin{eqnarray}
 0 & \geq & -
 \frac{1}{m\sqrt{v_jv_k}}\cdot\frac{1}{m_b}\sum_{\ell'}\left((x^b_{\ell j}
   - x^b_{\ell'j})(x^b_{\ell k}-x^b_{\ell'k})\right) \nonumber\\
 & & =-
 \frac{1}{m\sqrt{v_jv_k}}\cdot\left(x^b_{\ell j}x^b_{\ell k}
 - x^b_{\ell j} \mu^b_k- x^b_{\ell k}
 \mu^b_j + \mu^b_{jk}\right) \nonumber\\
 & & =-
 \frac{1}{m\sqrt{v_jv_k}}\cdot\left(c^b_{jk} + (x^b_{\ell j} - \mu^b_j) (x^b_{\ell k}
   - \mu^b_k) \right)\nonumber\:\:,
\end{eqnarray}
\textit{i.e.} we have:
\begin{eqnarray}
(x^b_{\ell j} - \mu^b_j) (x^b_{\ell k}
   - \mu^b_k) & \geq & - c^b_{jk}\label{eqdef1}\:\:.
\end{eqnarray}
Assume now that $\uppi^b(\ell)$ holds over any $\ell \in [m_b]$.
As long as $c^b_{jk}$ is strictly positive, we thus obtain, averaging
ineq. (\ref{eqdef1}) over all $\ell\in [m_b]$,
\begin{eqnarray}
c^b_{jk}  & \defeq & \frac{1}{m_b}\sum_{\ell}(x^b_{\ell j} - \mu^b_j) (x^b_{\ell k}
   - \mu^b_k) \nonumber\\
 & \leq & - c^b_{jk} \nonumber\\
 & < & 0\:\:,\label{eqdef2}
\end{eqnarray}
a contradiction. Hence, as long as $c^b_{jk} > 0$, there must exist
$(\ell,\ell') \in \ii{2}{m_b}$ such that the elementary permutation
$\varsigma(\ell) = \ell'$, $\varsigma(\ell')=\ell$ satisfies
\begin{eqnarray}
c^b_{jk^{\varsigma}} & < & 
c^b_{jk}\:\:,\label{conccor}
\end{eqnarray} 
and this holds for $b\in \{-,+\}$. Now remark that
\begin{eqnarray}
c_{jk} & = & p\mu^+_{jk} + (1-p) \mu^-_{jk} - (p\mu_j^++(1-p)\mu_j^-)(p\mu_k^++(1-p)\mu_k^-)\nonumber\\
 & = & pc_{jk}^++(1-p)c_{jk}^-+ p(1-p)(\mu_j^+-\mu_j^-)(\mu_k^+-\mu_k^-)\:\:,\label{prop11XX}
\end{eqnarray}
and so as long as whichever $c_{jk}^+ > 0$ or $c_{jk}^- > 0$, we can
always find an elementary permutation that decreases the one
chosen. When no more elementary permutations achieve that, $c_{jk}
\leq p(1-p)(\mu_j^+-\mu_j^-)(\mu_k^+-\mu_k^-)$, which yields the
statement of the Lemma.
\end{proof}
To prove the Theorem, remark that
\begin{eqnarray}
\rho_{13^\varsigma} - \rho_{12}\rho_{23^\varsigma} & = & \frac{1}{\sqrt{v_1v_3}}\left(\frac{1}{m}\cdot\sum_i
\left(x_{i1} - \frac{c_{12}}{v_2}\cdot
  x_{i2}\right)x_{\varsigma(i)3}\right) -
\left(\frac{\mu_1\mu_3}{\sqrt{v_1v_3}} -
  c_{12}\frac{\mu_2\mu_3}{v_2\sqrt{v_1v_3}}\right) \nonumber\\
 & = & \frac{1}{\sqrt{v_1v_3}}\left(\frac{1}{m}\cdot\sum_i
\left(x_{i1} - \frac{c_{12}}{v_2}\cdot
  x_{i2}\right)x_{\varsigma(i)3} - \left(\mu_1 - \frac{c_{12}\mu_2}{v_2}\right)\mu_3\right)
\end{eqnarray}
We apply Lemma \ref{lemperm} to linearly transformed column $\ve{c}' \defeq \ve{c}_1 -
\frac{c_{12}}{v_2}\ve{c}_2$ and column $\ve{c}_3$ and obtain that as long as
\begin{eqnarray}
\rho_{13} - \rho_{12}\rho_{23} & > &
\frac{p(1-p)}{\sqrt{v_1v_3}}\left((\mu^+_1 - \mu^-_1) -
  \frac{c_{12}}{v_2} (\mu^+_2 - \mu^-_2)\right)(\mu^+_3 - \mu^-_3)\:\:,
\end{eqnarray}
there always exist a block-class elementary permutation $\varsigma$ that is going
to make $\rho_{13^\varsigma} - \rho_{12}\rho_{23^\varsigma}<\rho_{13}
- \rho_{12}\rho_{23}$. When no such permutation exist anymore, we
have, letting $\varsigma_T$ denote the composition of all elementary
permutations performed so far and $\varsigma_* \defeq \arg \max_{\varsigma \in S^*_m}
    \rho^2_{23^\varsigma}$,
\begin{eqnarray}
\frac{\rho_{13} - \rho_{12}\rho_{23}
}{\sqrt{1-\rho^2_{12}}\sqrt{1-
    \rho^2_{23^{\varsigma_T}}}} & \leq & \frac{\rho_{13} - \rho_{12}\rho_{23}
}{\sqrt{1-\rho^2_{12}}\sqrt{1-\rho^2_{23^{\varsigma_*}}}}\nonumber\\
 & & = \frac{p(1-p)}{\sqrt{v_1v_2 - c^2_{12}}\sqrt{v_2v_3 -
     c^2_{23^{\varsigma_*}}}}\nonumber\\
 & & \cdot \left(v_2 (\mu^+_1 - \mu^-_1) -
  c_{12} (\mu^+_2 - \mu^-_2)\right)(\mu^+_3 - \mu^-_3) \nonumber\\
 & & = \frac{p(1-p)}{1-\epsilon}\left(\frac{\mu^+_1 - \mu^-_1}{\sqrt{v_1}} -
  \rho_{12} \cdot \frac{\mu^+_2 - \mu^-_2}{\sqrt{v_2}}\right)\cdot\frac{\mu^+_3 - \mu^-_3}{\sqrt{v}_3}
\end{eqnarray}
as long as $c^2_{12} \leq (1-\epsilon)v_1v_2$ and
$c^2_{23^{\varsigma_*}} \leq (1-\epsilon)v_2v_3$. We just have to use
the fact that
\begin{eqnarray}
\tilde{\mu}_j & = & \frac{\mu^+_j - \mu^-_j}{\sqrt{v_j}}
\end{eqnarray}
using the main file notation to conclude (End of the proof of Theorem \ref{thmcm1}).
\end{proof}
Remark that the proof also
shows that the conditions to blow up the type-II error (Corollary
2.1 in \citep{cmT2}) are not affected by the \cp. 
To see that it is
possible to still make the Type II error blow, up, in the CM model,
the Type II error can be made at least $K/v_2$ \citep{cmT2} where the
coefficient $K$ does not depend on $\varsigma$ (\citep{cmT2}, Corollary
2.1). Since $v_2$
is also not altered by the permutations, the Type
II error can still be blown up following \citep{cmT2}'s
construction.

We now show that the iterative process is actually not necessary
if one has enough data: sampling $\varsigma \sim
S^*_m$ jams $\rho_{(13^\varsigma) \cdot2}$ up to bounds
competitive with Theorem \ref{thmcm1} with high probability. Such good
concentration
results also hold for \hsic~\citep{ssgbbFS}.

\begin{theorem}\label{thmcm2}
For any $\updelta>0$, provided $m = \Omega((1/\updelta) \log (1/\updelta))$, the uniform
sampling of $\varsigma$ in $S^*_m$ satisfies 
\begin{eqnarray*}
\pr_{\varsigma \sim S^*_m} [\rho_{(13^\varsigma)
  \cdot2} & \leq & \mathscr{R} + \updelta] \geq 1 - \updelta\:\:,
\end{eqnarray*} 
where $\mathscr{R}$ is defined
in Theorem \ref{thmcm1}. 
\end{theorem}
\begin{proof}
We detail first the sampling process of
$\varsigma$. It relies on the fundamental property that a permutation uniquely
factors as a product of disjoint cycles, and so a block-class permutation $\varsigma$ factors uniquely as two
permutations $\varsigma_+$ and $\varsigma_-$, each of
which acts in one of the two classes. Therefore, sampling uniformly
each of $\varsigma_+$ and $\varsigma_-$ results in an uniform sampling
of a block-class $\varsigma$.

The Theorem stems from the following Lemma, whose notations follow
Lemma \ref{lemperm}.
\begin{lemma}\label{lemperm2}
For any $q>0$, as long as
\begin{eqnarray}
m & = & \Omega\left(\frac{1}{q}\log \frac{1}{q}\right)\:\:,
\end{eqnarray}
there is probability $\geq 1 - q$ that a randomly chosen block-class
permutation $\varsigma$ shall bring
\begin{eqnarray}
c_{jk^\varsigma} & \in &
\left[p(1-p)(\mu_j^+-\mu_j^-)(\mu_k^+-\mu_k^-)-q, p(1-p)(\mu_j^+-\mu_j^-)(\mu_k^+-\mu_k^-)+q\right]\:\:.
\end{eqnarray}
\end{lemma}
\begin{proof}
Let $S^{*b}_m \subset S^*_m$ denote the set of block-class
permutations whose set of fixed points contains all examples from
class $\neq b1$, $\forall b\in \{-,+\}$. We have 
\begin{eqnarray}
\expect_{\varsigma \sim S^{*b}_m}\left[\sum_{l : y_l = b}
  {x^b_{lj}x^b_{\varsigma(l)k}}\right] = m_b \mu^b_j \mu^b_k\:\:,\nonumber
\end{eqnarray}
and since $(1/m_b) \sum_{l : y_l = b}
  {x^b_{lj}x^b_{\varsigma(l)k}} - \mu^b_j \mu^b_k =
  c^b_{jk^\varsigma}$, if we sample uniformly at random $\varsigma
  \sim S^{*b}_m$, then we get from \citep{cSM} (Proposition 1.1):
\begin{eqnarray}
\pr_{\varsigma \sim S^{*b}_m} [|c^b_{jk^\varsigma}|\geq t] & = & \pr_{\varsigma \sim S^{*b}_m} \left[\left|\sum_{l : y_l = b}
  {x^b_{lj}x^b_{\varsigma(l)k}} - m_b \mu^b_j \mu^b_k\right|\geq m_b t\right]\nonumber\\
 & \leq & 2\exp\left(-\frac{m_b  t^2}{4 \mu^b_j \mu^b_k + 2t}\right)\:\:.
\end{eqnarray}
We want the right hand side to be no more than some $\delta_b$;
equivalently, we want
\begin{eqnarray}
t^2 - \left(\frac{2}{m_b}\log\frac{2}{\delta_b}\right)t -  \frac{2}{m_b}\log\frac{2}{\delta_b}& \geq & 0\:\:,
\end{eqnarray}
which holds provided
\begin{eqnarray}
t & \geq & \frac{2(1+o(1))}{m_b} \log \frac{2}{\delta_b}\:\:,
\end{eqnarray}
where the little-oh is measured wrt $m_b$. Since a block-class
permutation factors as two fully determined permutations from
$S^{*+}_m$ and $S^{*-}_m$, if we fix $\delta_+ = \delta_- = \delta/2$,
we get that if we sample uniformly at random these two permutations
$\varsigma_+ \sim S^{*+}_m$ and $\varsigma_- \sim S^{*-}_m$, then we shall have simultaneously
\begin{eqnarray}
|c^b_{jk^\varsigma}| & \leq & \frac{2(1+o(1))}{m_b} \log
\frac{4}{\delta} \:\:, \forall b\in \{-,+\}\:\:,
\end{eqnarray}
which implies for the factored permutation $\varsigma$,
\begin{eqnarray}
|c_{jk^\varsigma} - p(1-p)(\mu_j^+-\mu_j^-)(\mu_k^+-\mu_k^-)| & \leq &
\sum_b \frac{m_b}{m}\cdot \frac{2(1+o(1))}{m_b} \log
\frac{4}{\delta} \nonumber\\
 & & = \frac{2(1+o(1))}{m} \log
\frac{4}{\delta} 
\end{eqnarray}
from eq. (\ref{prop11XX}). Hence, if 
\begin{eqnarray}
m & = & \Omega\left(\frac{1}{\updelta} \log\frac{1}{\updelta}\right)\:\:,
\end{eqnarray}
there will be probability $\geq 1 - \updelta$ that $c_{jk^\varsigma}$ is
within additive $\updelta$ from $p(1-p)(\mu_j^+-\mu_j^-)(\mu_k^+-\mu_k^-)$.
\end{proof}
We get that with probability $\geq 1-\updelta$, a randomly chosen block-class
permutation $\varsigma$ shall make
\begin{eqnarray}
\frac{\rho_{13^\varsigma} - \rho_{12}\rho_{23^\varsigma}
}{\sqrt{1-\rho^2_{12}}\sqrt{1-
    \rho^2_{23^{\varsigma}}}} & \leq & \frac{p(1-p)}{1-\epsilon}\left(\frac{\mu^+_1 - \mu^-_1}{\sqrt{v_1}} -
  \rho_{12} \cdot \frac{\mu^+_2 -
    \mu^-_2}{\sqrt{v_2}}\right)\cdot\frac{\mu^+_3 -
  \mu^-_3}{\sqrt{v}_3} + \updelta\:\:,
\end{eqnarray}
and the Theorem is proven (End of the proof of Theorem \ref{thmcm2}).
\end{proof}

\section{Experiments}\label{expes_expes}

\subsection{Domains and setup}\label{exp_domains}

Domain characteristics are described in Table \ref{t-doms}. In particular, the process for train/test split is there given in detail for each dataset. Some domains deserve more comments.

Similarly to \citep{hjmpsNC}, we consider only two features of the
\emph{abalone} datasets; those are \emph{rings} (the age) and
\emph{length}, which are provably causally linked --\emph{age} causes
\emph{length}--, and hence correlated. We predict the attribute
\emph{diameter} (reasonably caused by \emph{age} as well); to turn
this into a binary classification problem, we classify if the
\emph{diameter} is above or below the average one. (We also exclude
abalone examples which have missing \emph{sex} attribute.) For the
experiments, we train with $200$ examples and held out $567$, both
picked at random. The \emph{digoxin} domain \citep{dmzsAP} is already
defined by only two features, \emph{digoxin} and \emph{urine}, which
are conditionally independent given \emph{creatine}. From those, we
predict if the level of \emph{creatine} is above or below
average. Domains Liver disorder, Auto+MPG, Arrhythmia and Diabete are
part of the benchmark of domains of \citep{mpjzsDC}.

The \emph{synthetic} dataset is generated by the function
$\texttt{datasets.make\_classification}$ of the \emph{scikit-learn} python library \citep{scikit-learn}, with $6$ features, $3$ informative for the class prediction, and $3$ more that are linear combinations of the formers. The rational of this toy domain is to craft two feature subspaces highly correlated.

\begin{sidewaystable}[t]
{
    \centering
\begin{center}
\begin{tabular}{c|c|c|c|l}\hline \hline
 name & $m^{(*)}$ & $d$ & source & notes\\ \hline \hline
digoxin & 35 & 2 & \citep{dmzsAP} & features are cond. independent given label\\ \hline
glass & 146 & 9 & UCI &\\ \hline
abalone-2D & 200, 567  & 2 & UCI, \citep{hjmpsNC} & subsample, \{rings, length\} predict diameter\\ \hline
synthetic & 200 & 13 & scikit-learn &\\ \hline
heart & 270 & 13 & UCI &\\ \hline
liver disorders &  345 & 7 & UCI, \citep{mpjzsDC} & predict mcv $<$
$30^{th}$ percentile, task \texttt{pair0034}
\\ \hline 
ionosphere &  351 & 34 & UCI &\\ \hline
\multirow{2}{*}{auto+mpg} &  \multirow{2}{*}{398} & \multirow{2}{*}{8}
& \multirow{2}{*}{UCI, \citep{mpjzsDC}} & predict mpg $<$ mean, task
\texttt{pair0016}\\
 & & & & (feature vectors with missing values removed)
\\ \hline
arrhythmia & 452 & 279 & UCI, \citep{mpjzsDC} &  task
\texttt{pair0023}, missing values replaced by 0\\ \hline
breastw & 683 & 10 & UCI &\\ \hline
australian & 690 & 14 & UCI &\\ \hline
\multirow{2}{*}{diabete(-2)} & \multirow{2}{*}{768} & \multirow{2}{*}{8} & \multirow{2}{*}{UCI} & Pima domain (-2 = task
\texttt{pair0038}, \citep{mpjzsDC},\\
 & & & &  half of the dataset used for training)\\ \hline\hline
\end{tabular}
\caption{Domains considered. $~^{(*)}$ When only one number appears in the column, 1/5 of $m$ was hold out at random for test; when two numbers are present, the first is training set size, while the second is test size, that is fixed by the dataset description.}\label{t-doms}
\end{center}
}
\end{sidewaystable}

All training sets are standardized, and the same transformation is then applied to the respective test sets. The partition of the feature space is defined by the first split $F = \left \lfloor{d/2}\right \rfloor $, and its complement; features are taken in the order defined by the datasets.

Unless stated differently, models are trained with $L_2$
regularisation by \emph{scikit-learn}'s \\
$\texttt{linear\_model.LogisticRegression}$. The hyper-parameter $\lambda$ is optimized by 5-folds cross validation on the grid $\{10^{-5}, 10^{-4}, \dots, 10^{4} \}$.

\subsection{Explanation of the movie}\label{res-movie}

\begin{figure}[h]
    \centering
\begin{center}
\includegraphics[width=0.7\columnwidth]{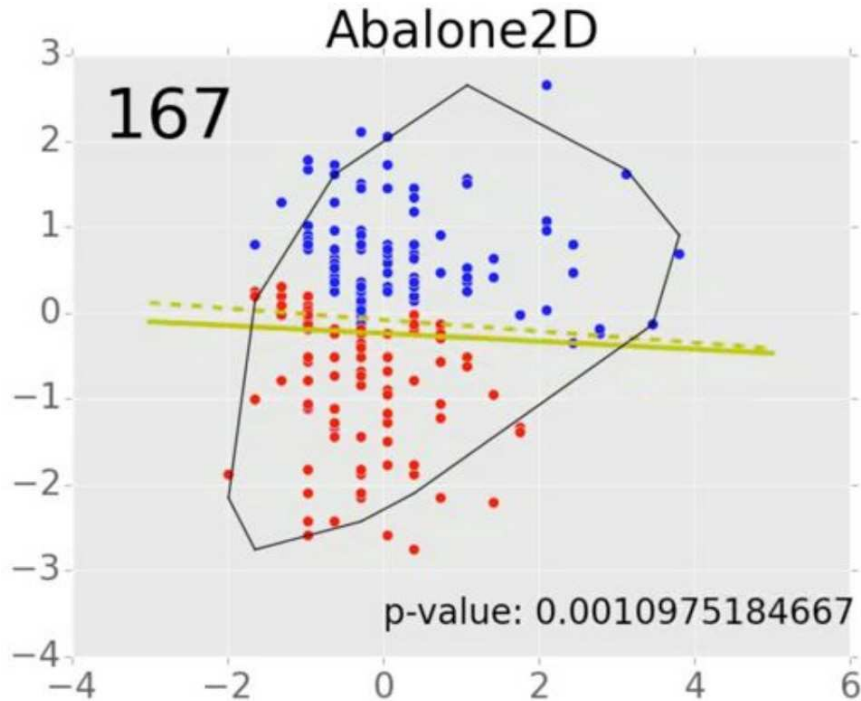}
\caption{Crop of the movie (see text for explanation).}\label{f-res-movie}
\end{center}
\end{figure}

Along with this draft comes a movie displaying the impact on the
$p$-value of \perm, in the context of the decrease of the \hsic. The
movie shows 200 iterations of \perm~on the Abalone2D domain, along with the modification of
the point cloud (classes are red / blue). The $p$-value is
indicated. Notice that it begins at value 0 up to 13 digits before
\perm~starts. The two lines indicate the classifier learnt over the
current data (plain gold line) and compare with the initial classifier
learnt over the data before running \perm~(dashed gold line). The big
number (167 in Figure \ref{f-res-movie}) is the iteration
number. Finally, the polygon displayed is the convex envelope of the
initial data.

\subsection{Main experiments}\label{sec-exp}

\begin{algorithm}[t]
\caption{Crossover Learning (${\mathcal{S}}, T,
  \ve{\theta}_0, \matrice{m}_0, \matrice{f}^\re, \matrice{f}^\tr;
  \mathscr{G}_1 ,[\mathscr{G}_2]$)}\label{algoMETA}
\begin{algorithmic}
\STATE  \textbf{Input} Sample ${\mathcal{S}}$, iterations $T$, classifier $\ve{\theta}_0$,
initial \cp~$\process{T}$ (matrices $\matrice{m}_0\in {S^*_m}$,
$[\matrice{f}^\re| \matrice{f}^\tr] = \matrice{i}_m$);
\STATE  Step 1 : \textbf{for} $t=1, 2, ..., T$
\STATE  \hspace{1.1cm} Step 1.1 : $\matrice{m} \leftarrow \arg\min_{\matrice{m}' \in
  S^{\mbox{\tiny{e}}*}_m} \mathscr{G}_1(\matrice{m}' \circ
\matrice{m}_{t-1} [| \ve{\theta}_{t-1}])$; 
\STATE  \hspace{2.9cm} // finds update of shuffle matrix
\STATE  \hspace{1.1cm} Step 1.2 : $\matrice{m}_t \leftarrow
\matrice{m} \circ \matrice{m}_{t-1}$; 
\STATE  \hspace{2.9cm} // updates shuffle matrix
\STATE  \hspace{1.1cm} [~Step 1.3 : $\ve{\theta}_t \leftarrow
\arg\min_{\ve{\theta} \in {\mathbb{R}^d}} \mathscr{G}_2(\ve{\theta}|
\matrice{m}_t)$;~]
\STATE  \hspace{2.9cm} // (optionally) updates classifier
\STATE \textbf{Return} classifier $\ve{\theta}_T$ and / or \cp'ed dataset
${\mathcal{S}}^{\process{T}(\matrice{m}_T)}$ 
\end{algorithmic}
\end{algorithm}

Our applications use the same meta-level algorithm (Algorithm
\ref{algoMETA}) which operates in Setting (A) $\cap$ Setting (B) ($\matrice{m}$
in $S_m^*$, linear classifiers) (Section \ref{sec-dsp-learn}), iteratively composing block-class elementary
permutations. Here, $S^{\mbox{\tiny{e}}*}_m \subset S^*_m$ is
  the set of block-class elementary permutations. The iteration step minimises a criterion
  $\mathscr{G}_1$ over $S^{\mbox{\tiny{e}}*}_m$ and potentially, after
  the update of the \cp~matrix, a criterion $\mathscr{G}_2$ over
  ${\mathcal{H}}$. The optimization of $\mathscr{G}_1$ is performed by a simple greedy search
  in the space of $S^{\mbox{\tiny{e}}*}_m$.
  The experimental setup (a dozen readily available domains) and 
  results are provided \textit{in extenso} in Section \ref{res-s}; Table
  \ref{t-applis} summarises them. 
  The split step and the choice of 
${\mathcal{F}}_\re$ are highly domain and task dependent: to keep
experiments of reasonable length, unless otherwise stated, we put in ${\mathcal{F}}_\re$ the first half of
  features. In Abalone2D and Digoxin, this
 jams a particular ground truth (see below). Also, $\varphi$=logistic loss. 

\subsubsection{Disrupting dependence and causality}\label{sec-dis}

\noindent\textbf{General experiments} --- We run Algorithm \ref{algoMETA}
without step 1.3, and let $\mathscr{G}_1$ be \hsic. 
As a proof of concept, we show that we can destroy the significance of statistical tests for independence, commonly as base for causal inference; in particular, we measure the change in the $p$-value computed on top of \hsic~as in \citep{gftsssAK}. We use two Gaussian kernels for $\matrice{k}^u$ and
$\matrice{k}^v$, each computed over its full subset of features
(Subsection \ref{exp_domains}) --- hence, we do not seek to alter
specifically the dependence between two features, but between the two
sets of features defined by the anchor and shuffle sets. On
most domains (Table \ref{t-applis}, top row), the $p$-value of the
independence test starts close to zero at the beginning of the \cp,
which implies that in general both anchor and shuffle sets are
(predictably) dependent. In general, we achieve a good control of the \rcp~and manage in several cases to
decrease the true error as well through the process. 

\noindent\textbf{Specific dependences} --- The general experiments
revealed that we manage
to blow-up the $p$-value, even for domains for
which the ground truth clearly \textit{implies the alternative hypothesis $H_1$}. 
To dive into this phenomenon, we have considered two sets
of experiments on which \hsic~is computed over two specific
features that are known to have a causal relationship. To disrupt the
dependence, we thus put one of the features in the anchor set and one
in the shuffle set of features (\textit{i.e.}, after we have split the
feature set in two, if both features belong to the same set, we switch
one with a randomly chosen feature of the other set).  

In the first set of experiments, we consider datasets with $d=2$ features,
so that this surgical disruption embeds kernels measured over the
complete set of features. Experiments are reported in Table
\ref{t-applis} for domains Abalone2D and Digoxin (Subsection \ref{exp_domains}).
In Abalone2D for example, a gold standard for
dependence \citep{hjmpsNC}, the final $p$-values is more than
\textit{ten billion} times the initial value. For Digoxin domain, another popular domain
\citep{dmzsAP} with ground truth, $p$ is very small at the beginning
(which corresponds to the ground truth $D \notindep U$; $D$ = digoxin
clearance, $U$ = urin flow). After shuffling, we obtain $p>0.4$, which easily
brings $D \indep U$, while ground truth is $D \indep U | C$ ($C$ = creatinine
clearance). A rather surprising fact is that in both cases, the effect on test error is 
minimal, considering that Abalone2D and Digoxin have $d=2$ attributes
only.

In the second set of experiments, we consider several domains
with a larger $d$, between 7 and 279. These domains belong to the
benchmarks of \citep{mpjzsDC} in which specific pairs are known to have
specific causal relationships, referred to as "causal tasks", indicated in Subsection
\ref{exp_domains}. We have targeted one causal task for each
domain. Table \ref{res-2features} summarizes the results obtained. In
all domains, the $p$-value is blown up at almost no expense in test
error. Quite remarkably, the initial value is indeed $p=0$ (up to
\textit{sixteen} digits), while we manage at the end of the process to get $p$
that exceeds
1\textperthousand, which would be quite sufficient to raise doubts
about the causal relationships for sensitive domains. For example,
in \texttt{pair0016}, the final $p>2$\textperthousand~might lead us to
keep the independence assumption between horsepower and acceleration. In Arrhythmia, we would keep the
(obviously fake) independence between age and weight. In the Liver disorder domain, our
experiment has the following interesting consequence. Causal task
\texttt{pair0034} is the causality relationship between alcohol
consumption and the measure of alkaline phosphatase (ALP, \citep{mpjzsDC}). It
is known that ALP elevation may be caused by heavy alcohol
consumption. By keeping the independence assumption for such a value
of $p$, one may just discourage specific blood tests related to
alcohol consumption if they were to be designed from this domain. 
 Again, the
variation of the test error is minimal (if any) for all these domains.

\subsubsection{Data optimisation for efficient learning}

In the previous subsection, we showed how a \cp~may be carried out to
target directly the disruption of causal relationships. In this
subsection, we analyse how (and when) it can be devised to improve the
test performances of classifiers.
We perform Algorithm \ref{algoMETA} with $\matrice{m}_0 = \matrice{i}_m$,
$\mathscr{G}_1(\matrice{m}
| \ve{\theta}) = \expect_{{\mathcal{S}}^{\process{T}(\matrice{m})}} \left[\varphi(y
   \ve{\theta}^\top\bm{x})\right] $ (Theorem \ref{rader}) and
 $\mathscr{G}_2(\ve{\theta} | \matrice{m}) =
 \expect_{{\mathcal{S}}^{\process{T}}} \left[\varphi(y
   \ve{\theta}^\top\bm{x})\right] + \lambda\|\ve{\theta}\|_2^2$ where
 $\lambda$ is learnt through cross-validation. The
 algorithm returns classifier $\ve{\theta}_T$. The bottom row in Table
 \ref{t-applis}, and Subsection \ref{res-s} shows how we almost
 always find some permutations that reduce the test error compared to the
 initial data, even when a
 specific data optimisation should care for a risk of over-fitting,
 which seems to occur for problems with a very small number of
 features (see Ionosphere and Abalone2D, Subsection \ref{res-s}).
 It appears also that when the (bound on the) \rcp~flattens, it
 may indicate a regime where substantial reductions can be obtained on test error
 as witnessed by \textit{e.g.} Synthetic, Heart, Glass (see also
 BreastWisc in Table \ref{res-breast}). We also remarked that the
 regime where the (bound on the) \rcp~flattens can be associated with
 a peak or decrease of the number of odd cycles as seen in Table
 \ref{t-applis-oc}, which, interestingly, can be used to provide
 upperbounds on the \rcp~as well (Theorem
 \ref{thradcompSettingB}). Finally, the results on Glass display that
 some domains (predictably) make it possible to kill to birds in one shots at
 little effort: in this case, the \cp~both reduces the test error and
 increases the $p$-value for \hsic~computed as in the general
 experiments of Subsection \ref{sec-dis}.

\begin{table}[t]
    \centering
\begin{center}
{\scriptsize
\begin{tabular}{c}
\multicolumn{1}{c}{\hspace{-0.25cm}\begin{tabular}{>{\centering\arraybackslash} m{0.1cm}ccc||c}\hline\hline
\vspace{-2cm}\begin{sideways}{\footnotesize \hspace{-0.11cm}  \hsic~reduction}\end{sideways} & \hspace{-0.11cm} \includegraphics[height=2.83cm]{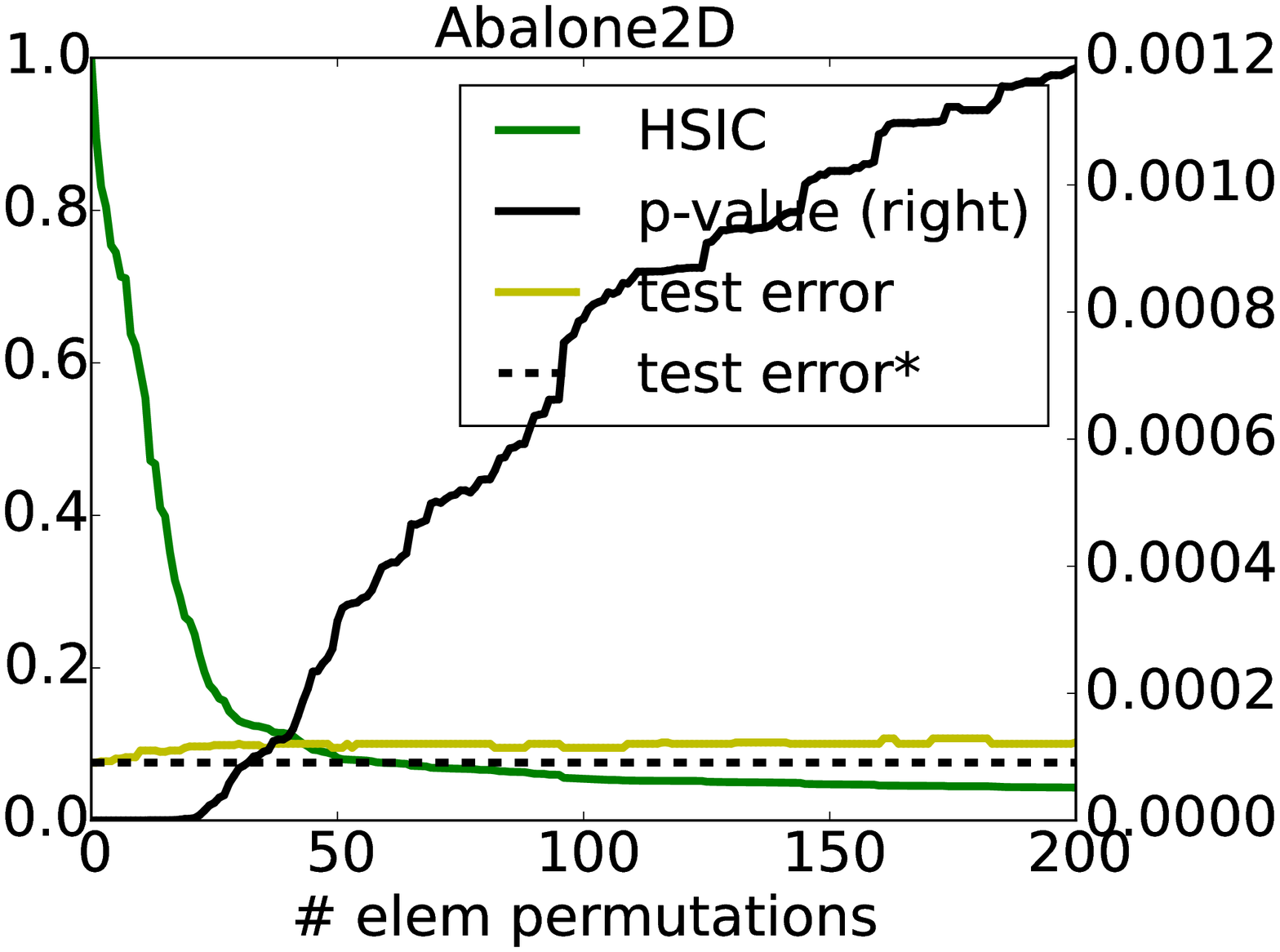}\hspace{-0.11cm} 
& \hspace{-0.11cm} \includegraphics[height=2.83cm]{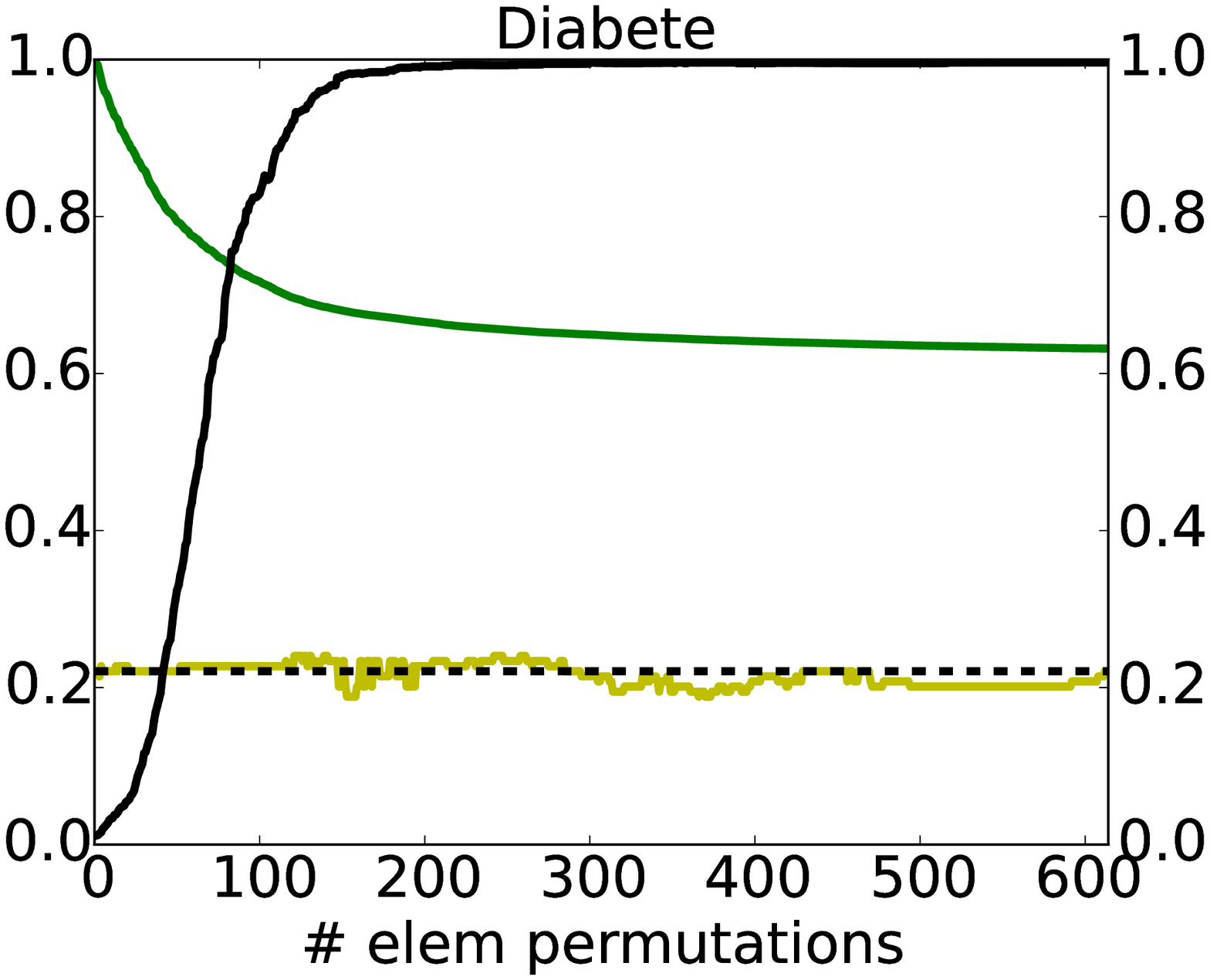}\hspace{-0.11cm} 
&
\hspace{-0.11cm} \includegraphics[height=2.83cm]{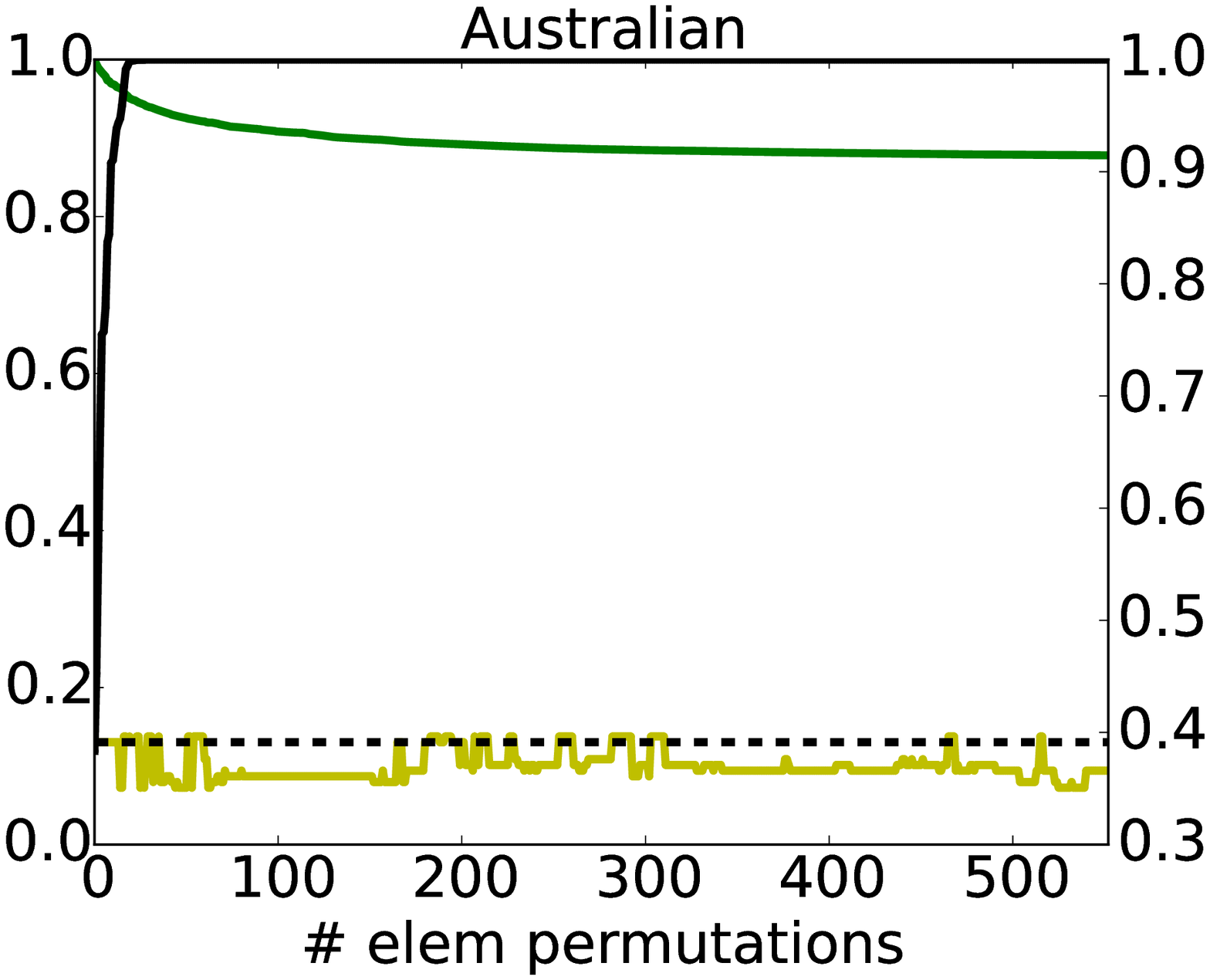}\hspace{-0.11cm}
& \hspace{-0.11cm} \includegraphics[height=2.83cm]{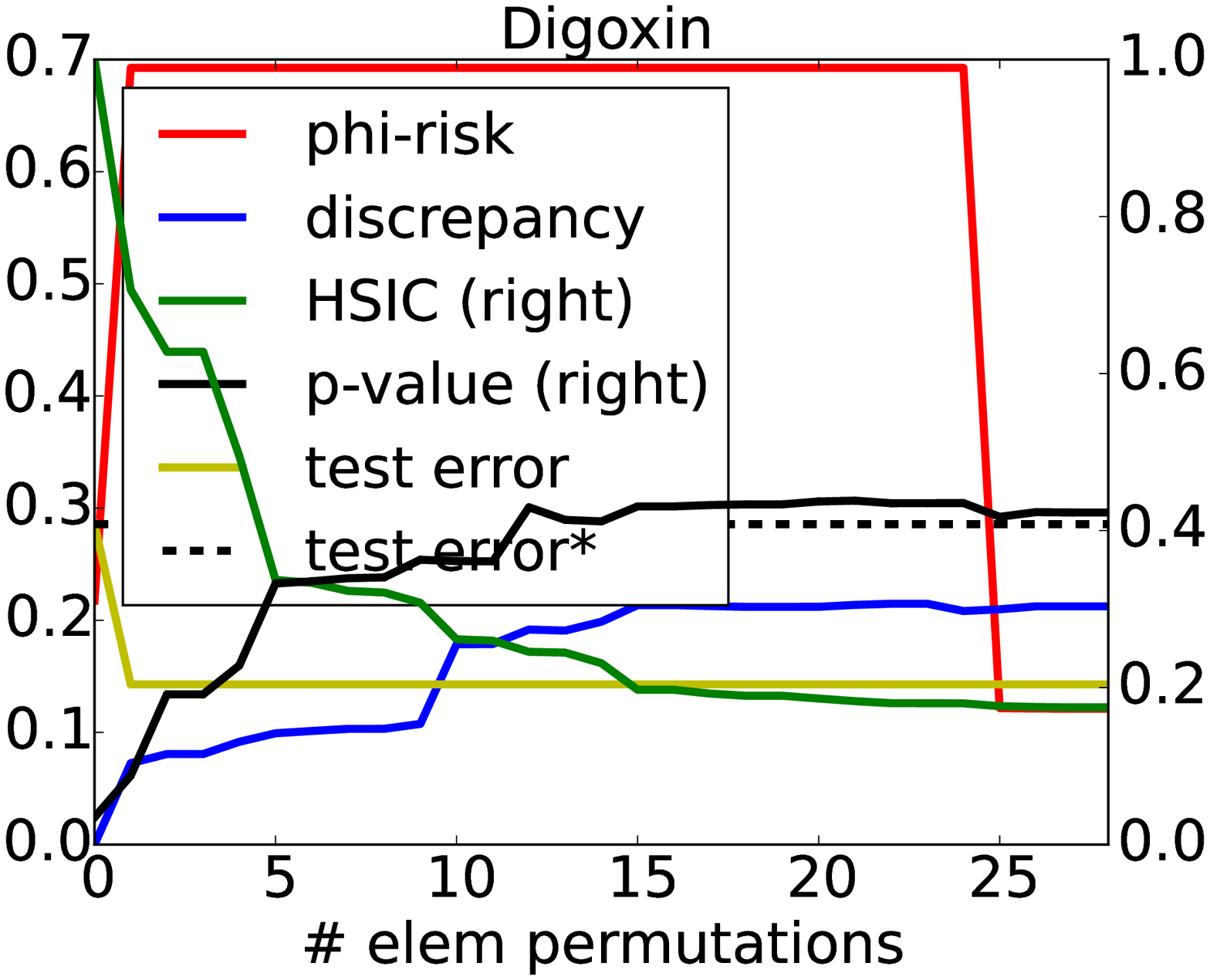}\hspace{-0.11cm}\\ \hline\hline
\vspace{-2cm}\begin{sideways}{\footnotesize \hspace{-0.11cm}  data optimisation}\end{sideways} & \hspace{-0.11cm} \includegraphics[height=2.83cm]{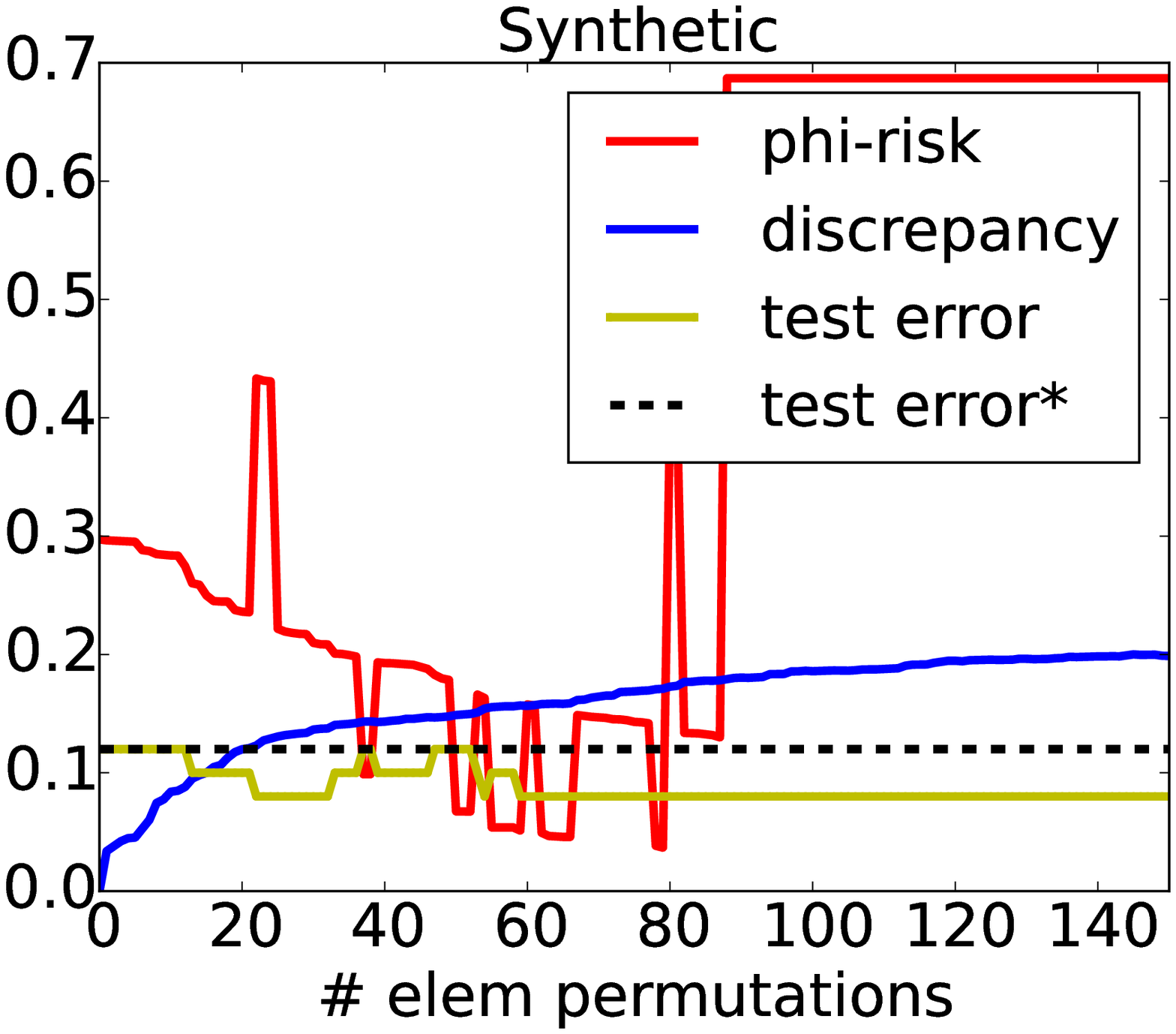}\hspace{-0.11cm} 
& \hspace{-0.11cm} \includegraphics[height=2.83cm]{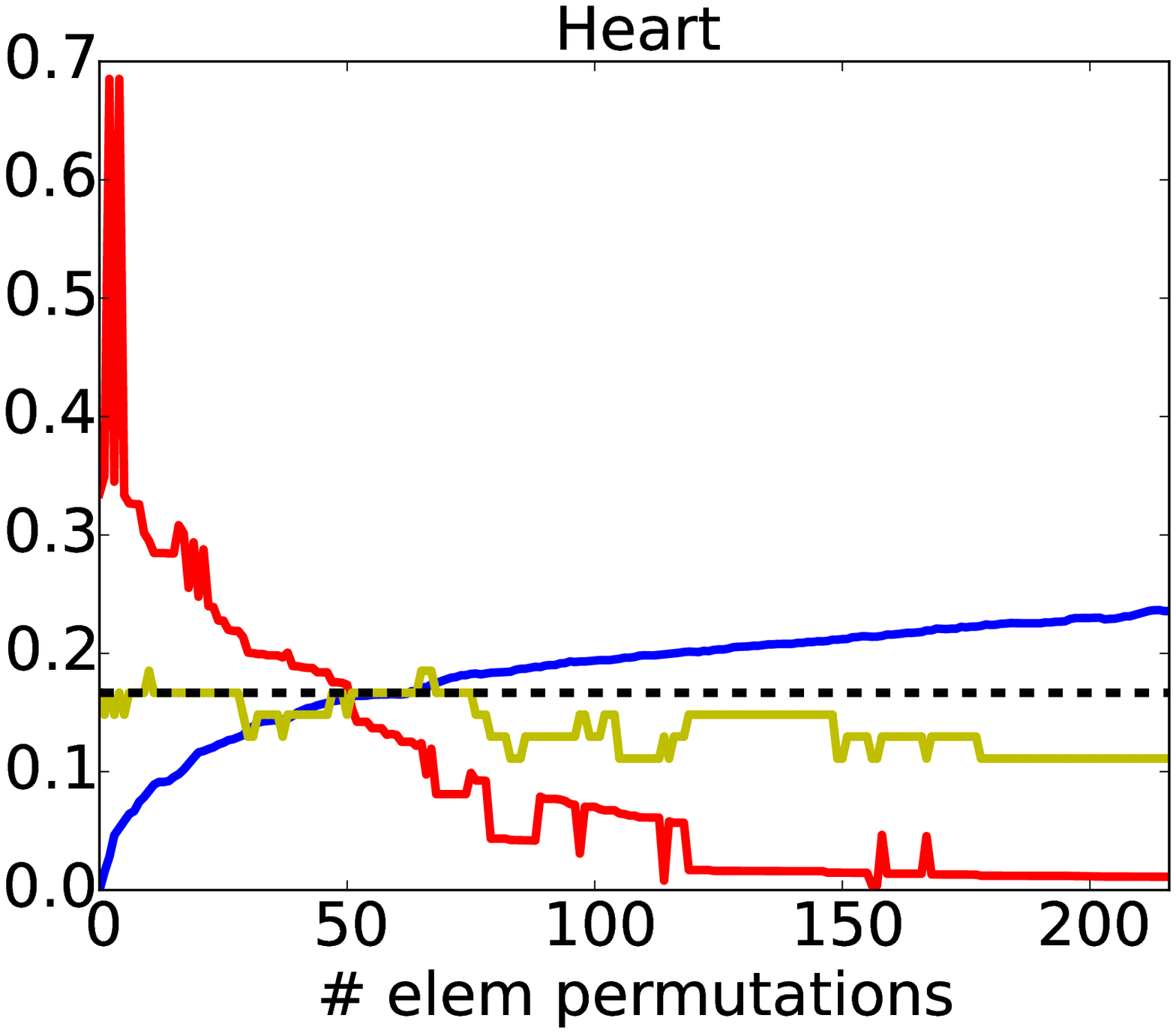}\hspace{-0.11cm} 
&
\hspace{-0.11cm} \includegraphics[height=2.83cm]{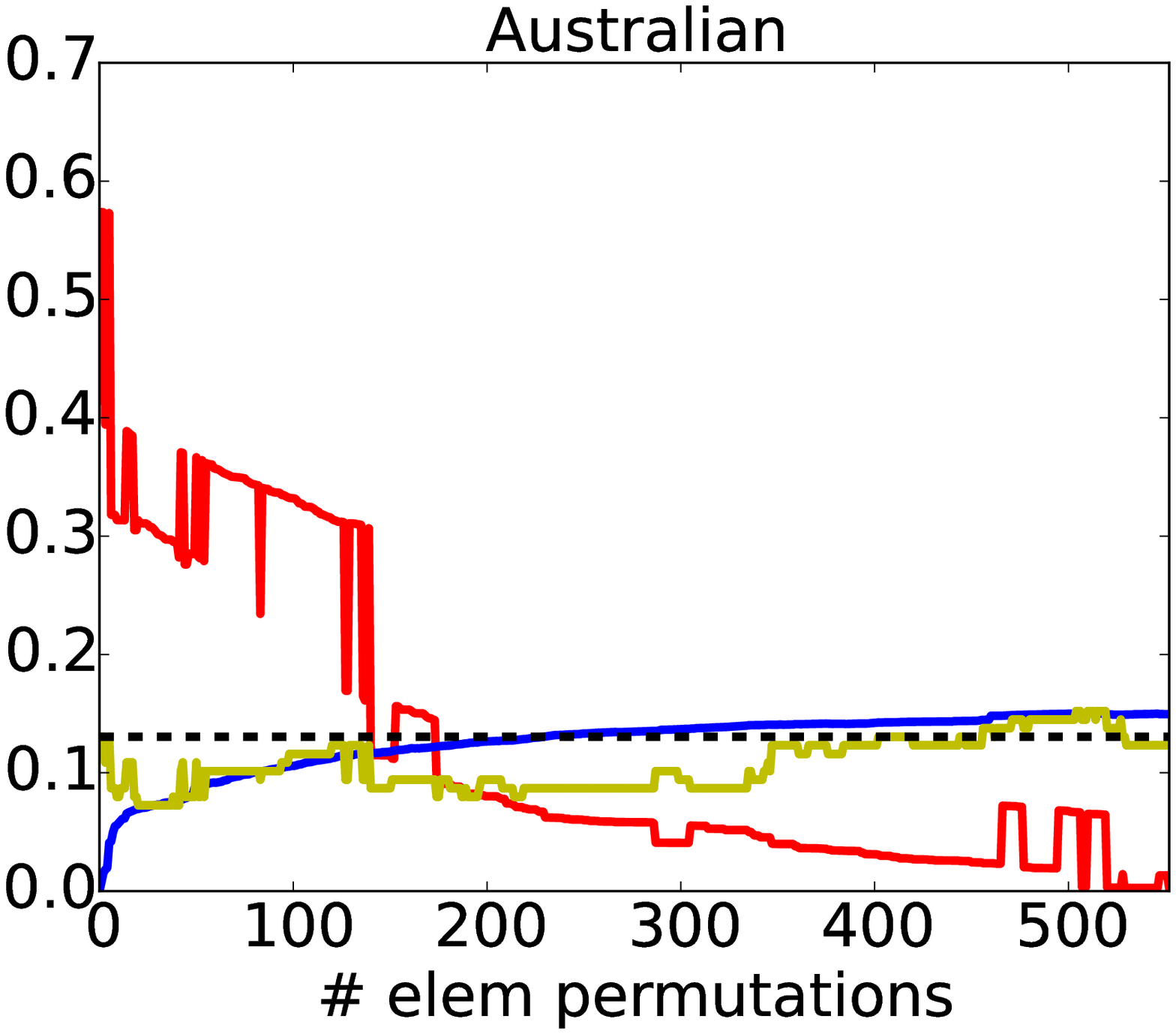}\hspace{-0.11cm} 
&
\hspace{-0.11cm} \includegraphics[height=2.83cm]{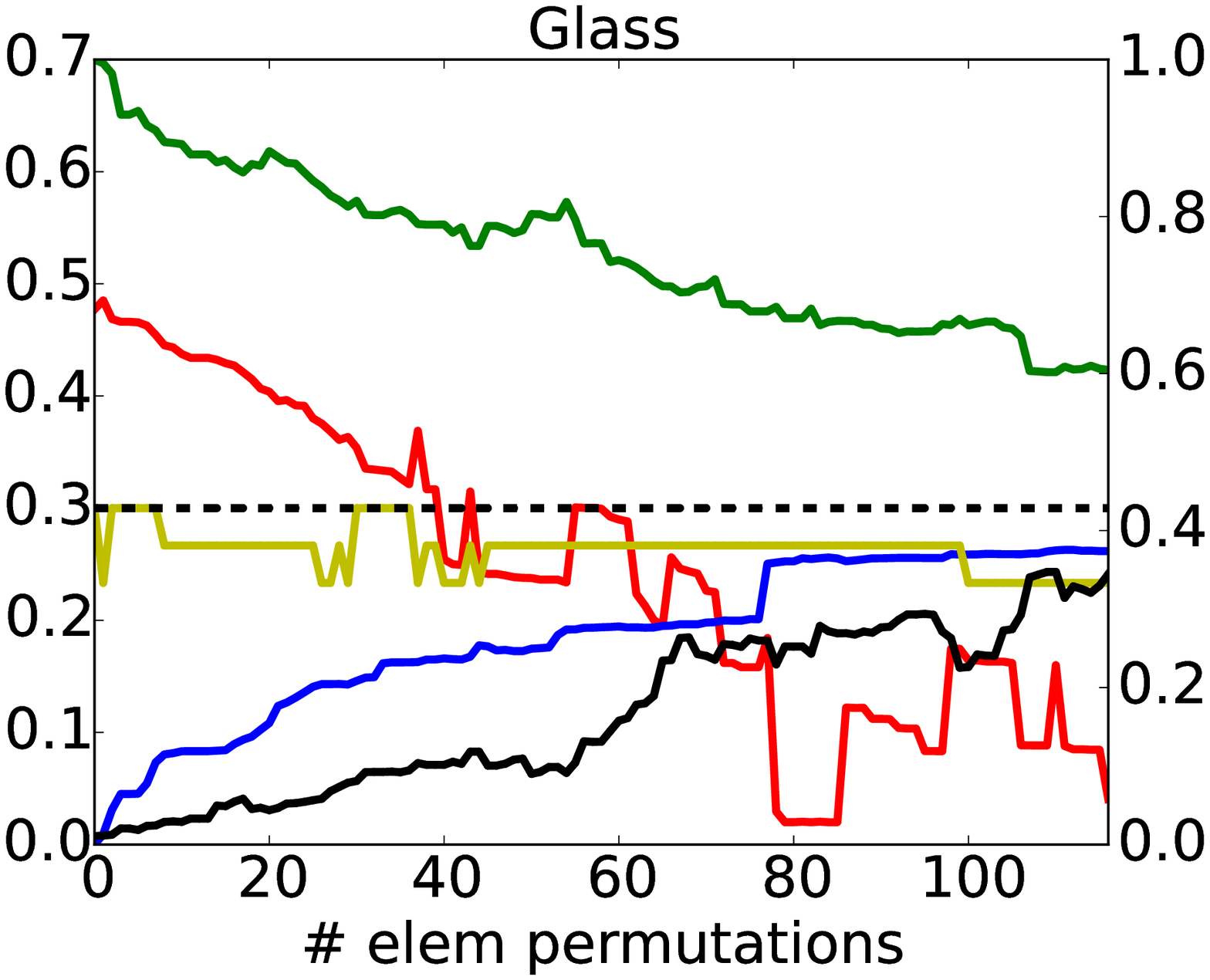}\hspace{-0.11cm} \\ \hline\hline
\end{tabular}
}
\end{tabular}
}
\caption{Experiments performed with \cp. Top row: reduction in \hsic~task;
  bottom row: data optimisation task. References to domain names are provided in
  Subsection \ref{exp_domains}. ``test-error*'' is test error
  over initial, non shuffled data. ``discrepancy'' is an upperbound on the
  \rcp~provided by Theorem \ref{thradcompUU}. The rightmost column aggregates the
  information of all curves for the task in the row, for two
  different domains.}\label{t-applis}
\end{center}
\end{table}

\begin{table}[t]
    \centering
\begin{center}
\begin{tabular}{cc}\hline\hline
\includegraphics[height=5.50cm]{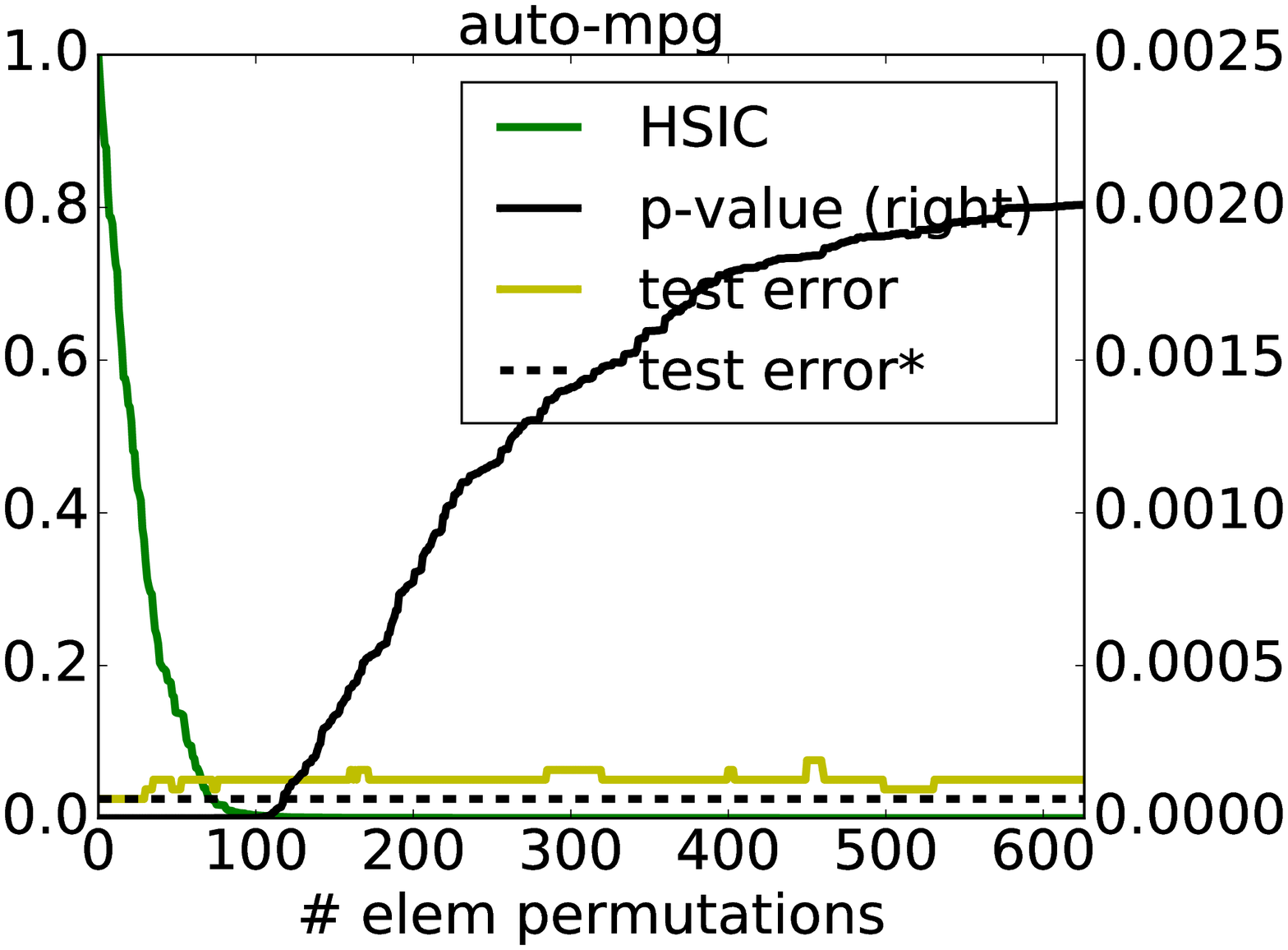}
& \includegraphics[height=5.50cm]{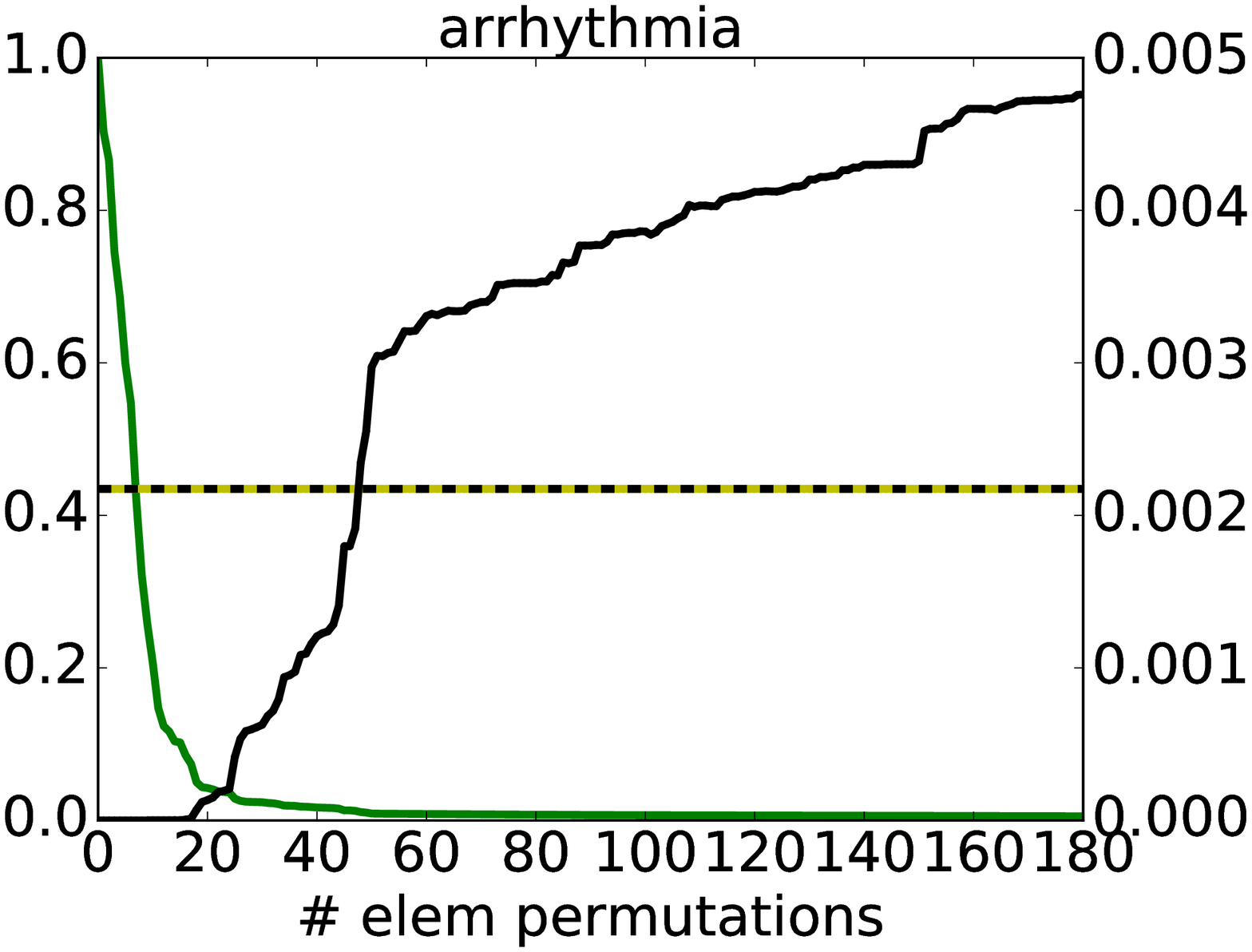}
\\
\texttt{pair0016} ($p_{\mbox{\tiny init}=0_{16}}$) & \texttt{pair0023} ($p_{\mbox{\tiny init}=0_{16}}$)\\\hline
\includegraphics[height=5.50cm]{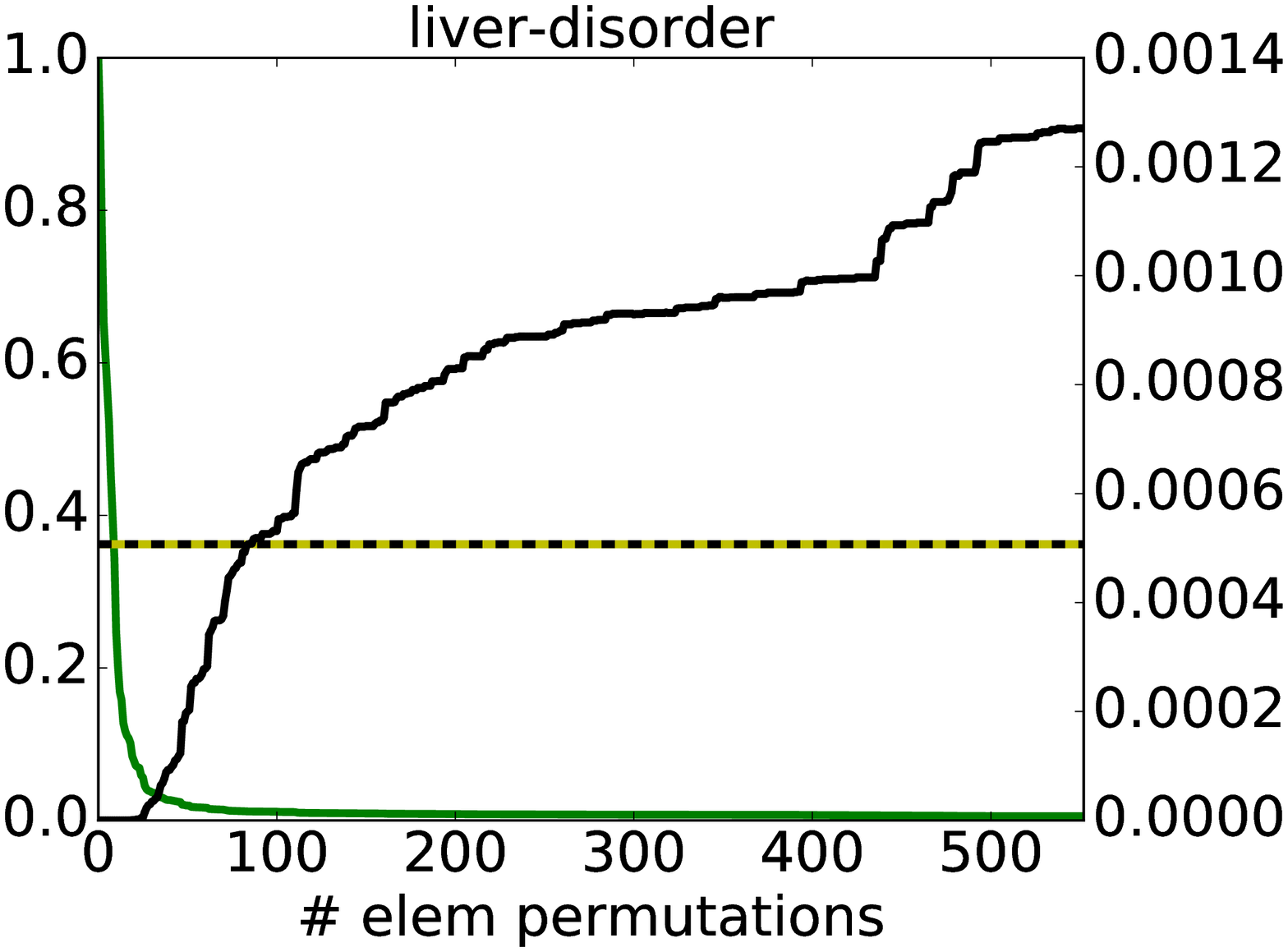}&
\includegraphics[height=5.50cm]{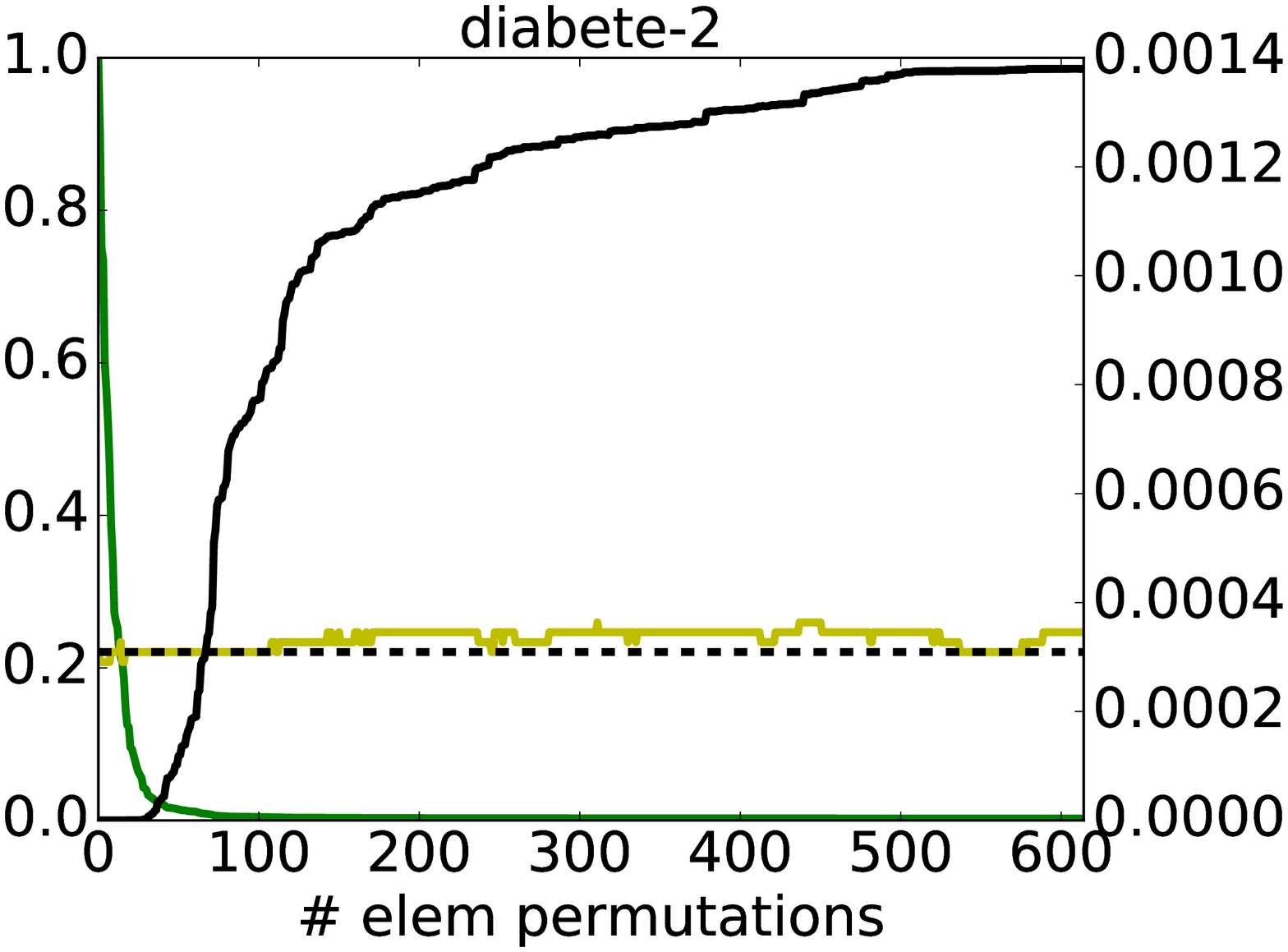}\\
\texttt{pair0034} ($p_{\mbox{\tiny init}=0_{16}}$) & \texttt{pair0038}
($p_{\mbox{\tiny init}=0_{16}}$)\\ \hline\hline
\end{tabular}
\caption{\hsic~reduction on specific causal tasks \citep{mpjzsDC}
  (referred to as \texttt{pair00XX}); datasets indicated on pictures;
  ``$p_{\mbox{\tiny init}}$'' is the initial $p$-value, $0_N$ indicating
  zero up to $N^{th}$ digit (see Table \ref{t-applis} for additional
  notations, and text for details).}\label{res-2features}
\end{center}
\end{table}

\newpage 
\subsection{Complete experimental results}\label{res-s}

\begin{table}[h]
    \centering
\begin{center}
\begin{tabular}{cccc}\hline\hline
\includegraphics[height=4.50cm]{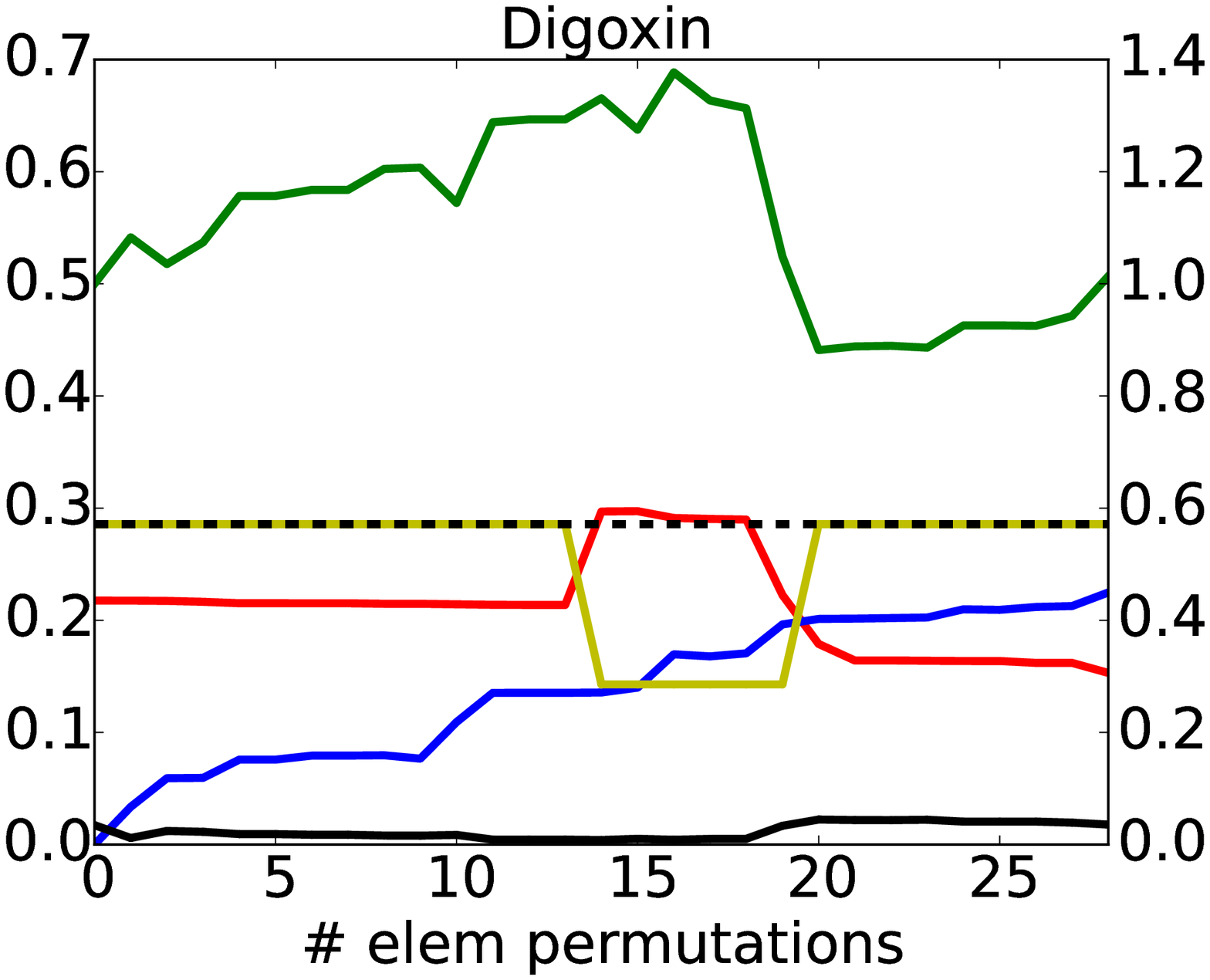}
& \includegraphics[height=4.50cm]{Archive/hsic/complete/hsic_digoxin.eps}\\
Data optimisation & \hsic~reduction\\ \hline\hline
\end{tabular}
\caption{Results on domain Digoxin. Left: Data optimisation;
  right: \hsic~reduction. 
Color codes are the same on
  all plots. See text  for details.}\label{res-digoxin}
\end{center}
\end{table}

\begin{table}[t]
    \centering
\begin{center}
\begin{tabular}{cc}\hline\hline
\includegraphics[height=4.50cm]{Archive/do/complete/do_glass_nolegend.eps}
& \includegraphics[height=4.50cm]{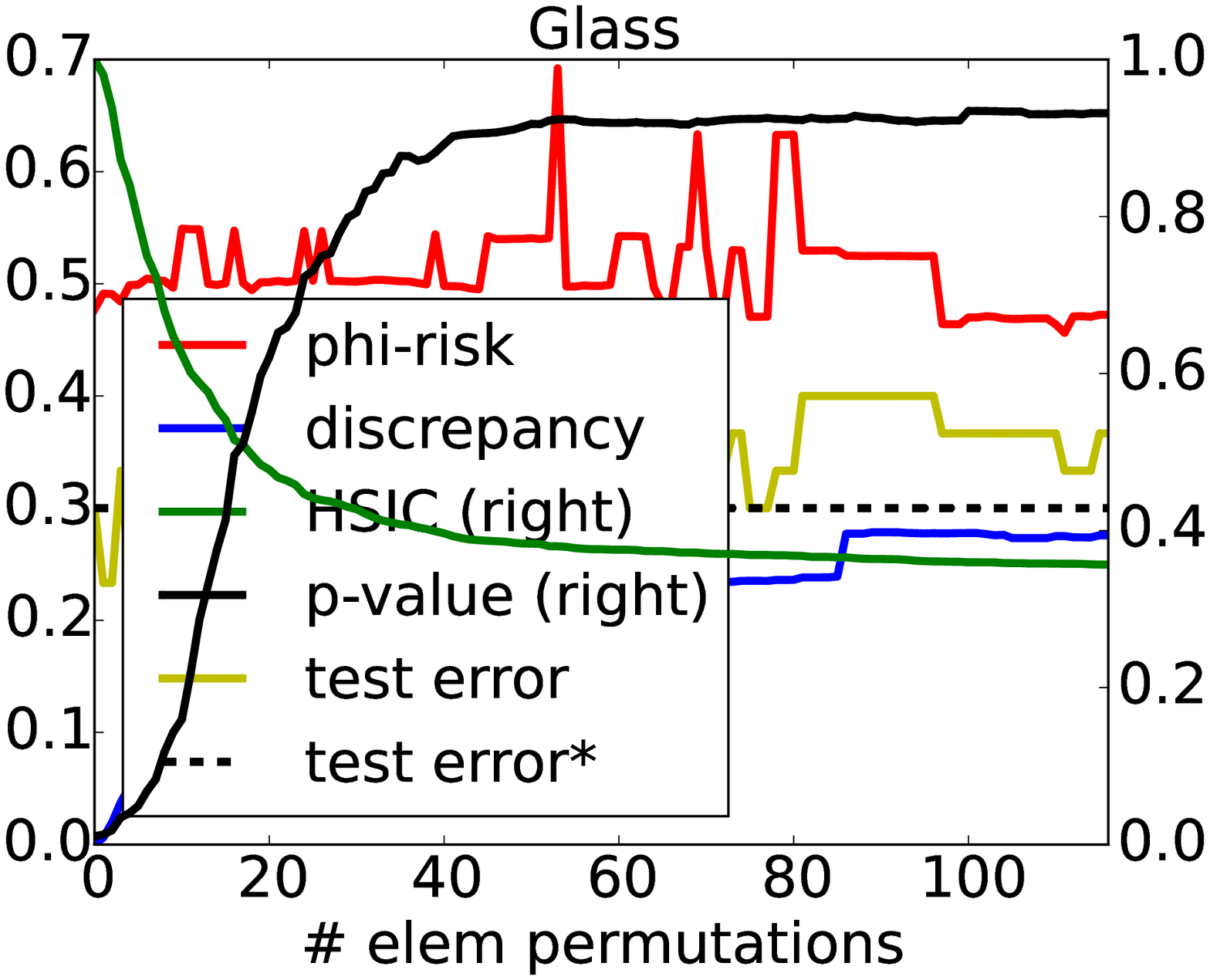}\\
Data optimisation & \hsic~reduction\\ \hline\hline
\end{tabular}
\caption{Results on domain Glass. Left: Data optimisation;
  right: \hsic~reduction. 
Color codes are the same on
  all plots. Color codes are the same on
  all plots. See text  for details.}\label{res-glass}
\end{center}
\end{table}

\begin{table}[t]
    \centering
\begin{center}
\begin{tabular}{cc}\hline\hline
\includegraphics[height=4.50cm]{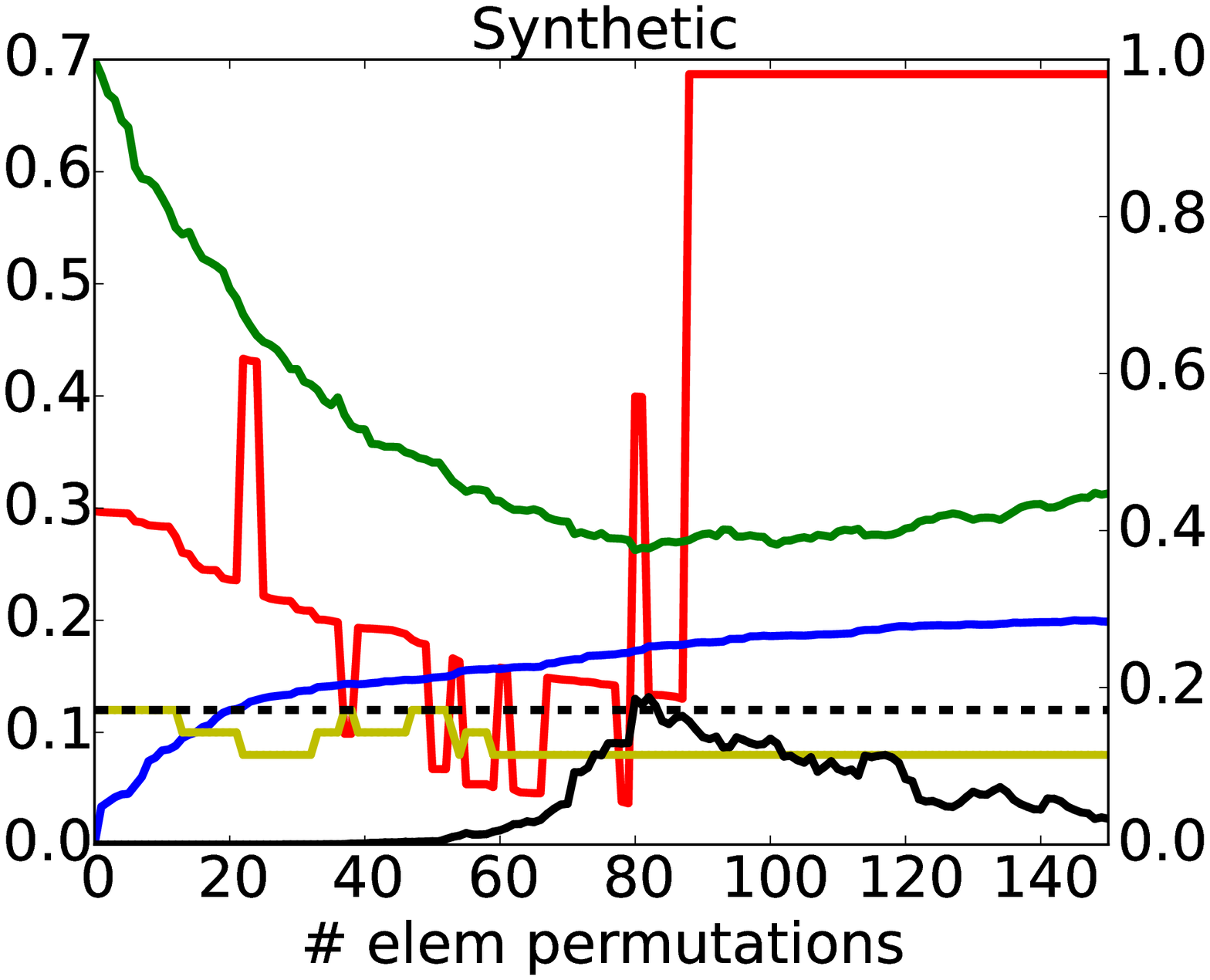}
& \includegraphics[height=4.50cm]{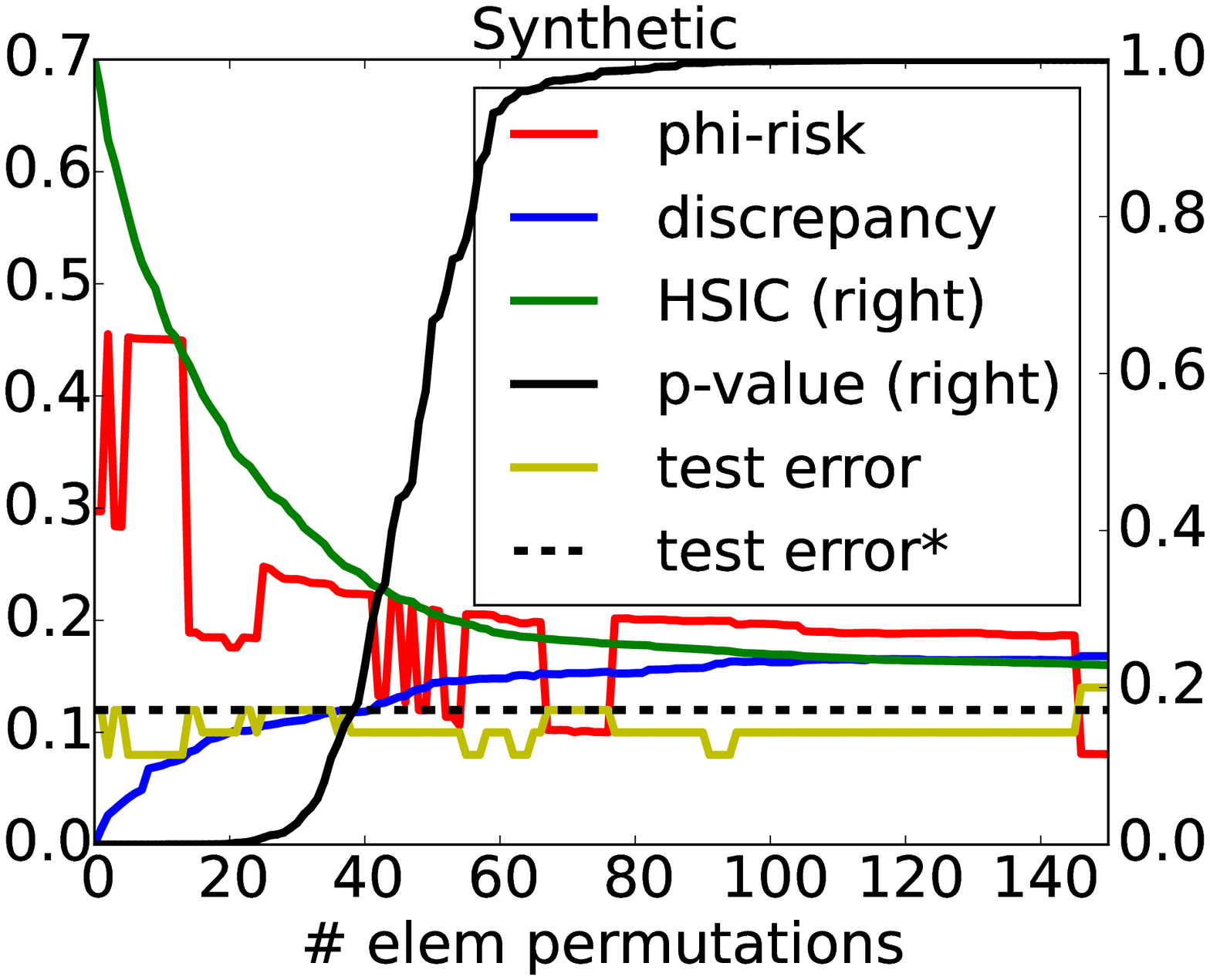}\\
Data optimisation & \hsic~reduction\\ \hline\hline
\end{tabular}
\caption{Results on domain Synthetic. Left: Data optimisation;
  right: \hsic~reduction. 
Color codes are the same on
  all plots. Color codes are the same on
  all plots. See text  for details.}\label{res-synth}
\end{center}
\end{table}

\begin{table}[t]
    \centering
\begin{center}
\begin{tabular}{cc}\hline\hline
\includegraphics[height=4.50cm]{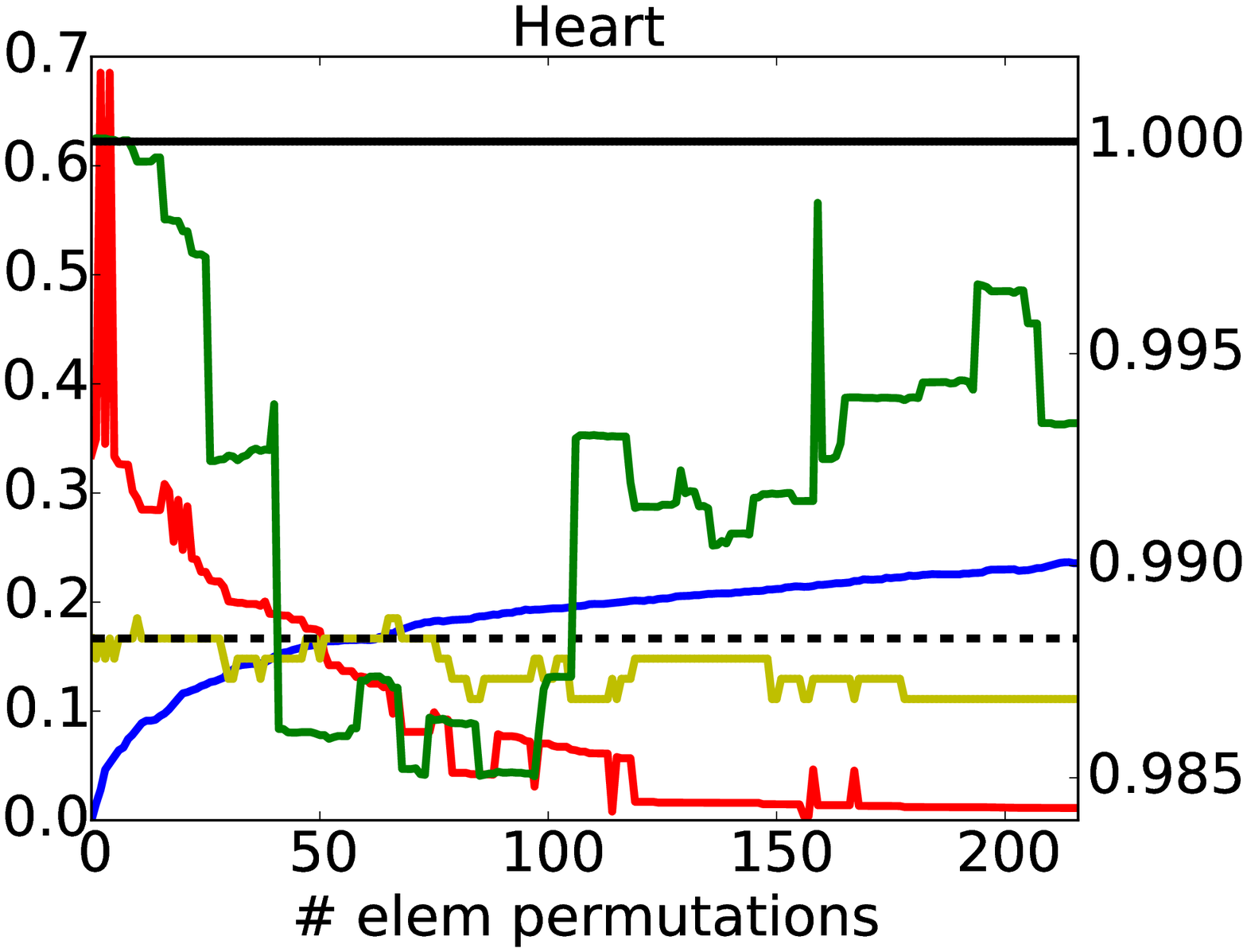}
& \includegraphics[height=4.50cm]{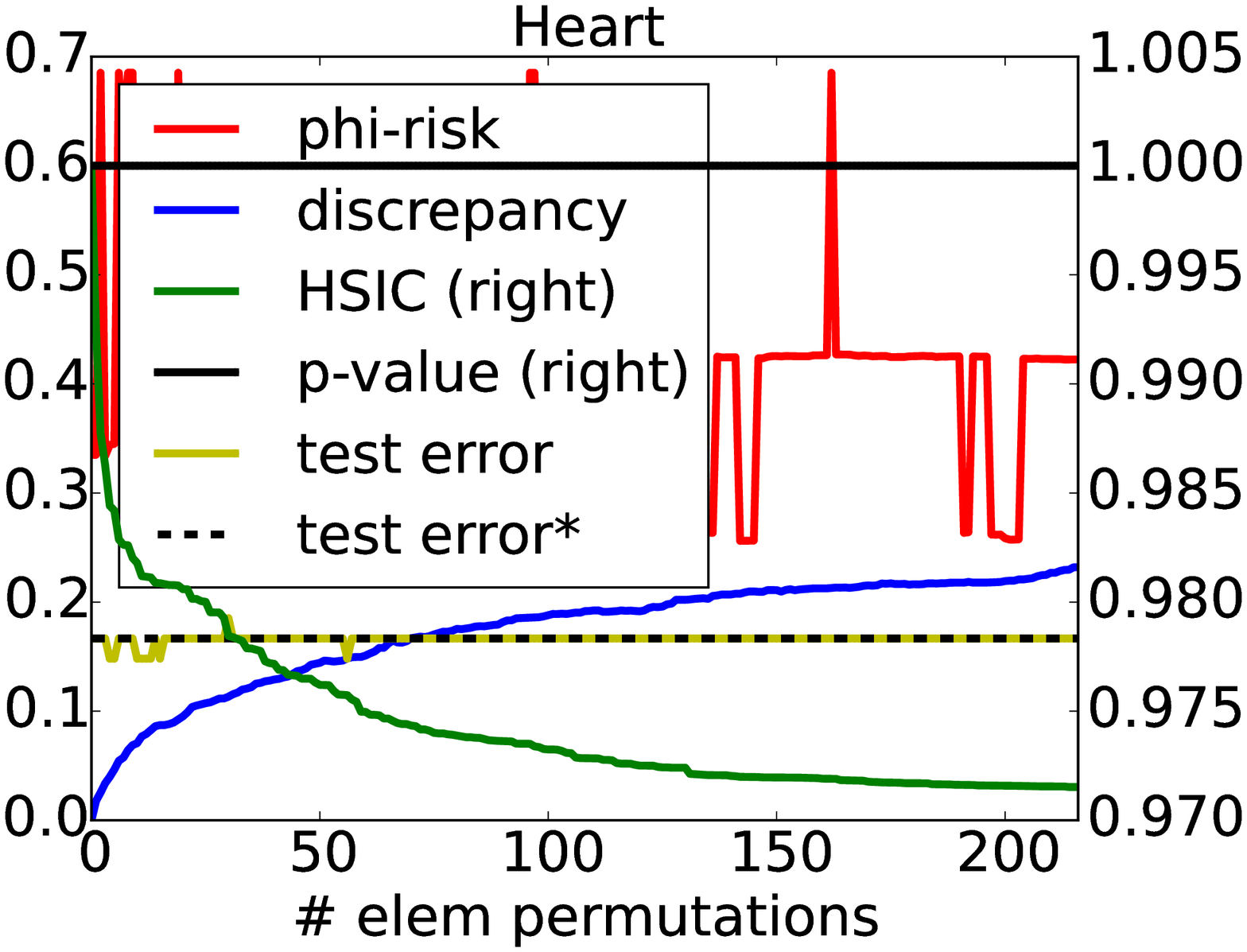}\\
Data optimisation & \hsic~reduction\\ \hline\hline
\end{tabular}
\caption{Results on domain Heart. Left: Data optimisation;
  right: \hsic~reduction. 
Color codes are the same on
  all plots. See text  for details.}\label{res-heart}
\end{center}
\end{table}

\begin{table}[t]
    \centering
\begin{center}
\begin{tabular}{cc}\hline\hline
\includegraphics[height=4.50cm]{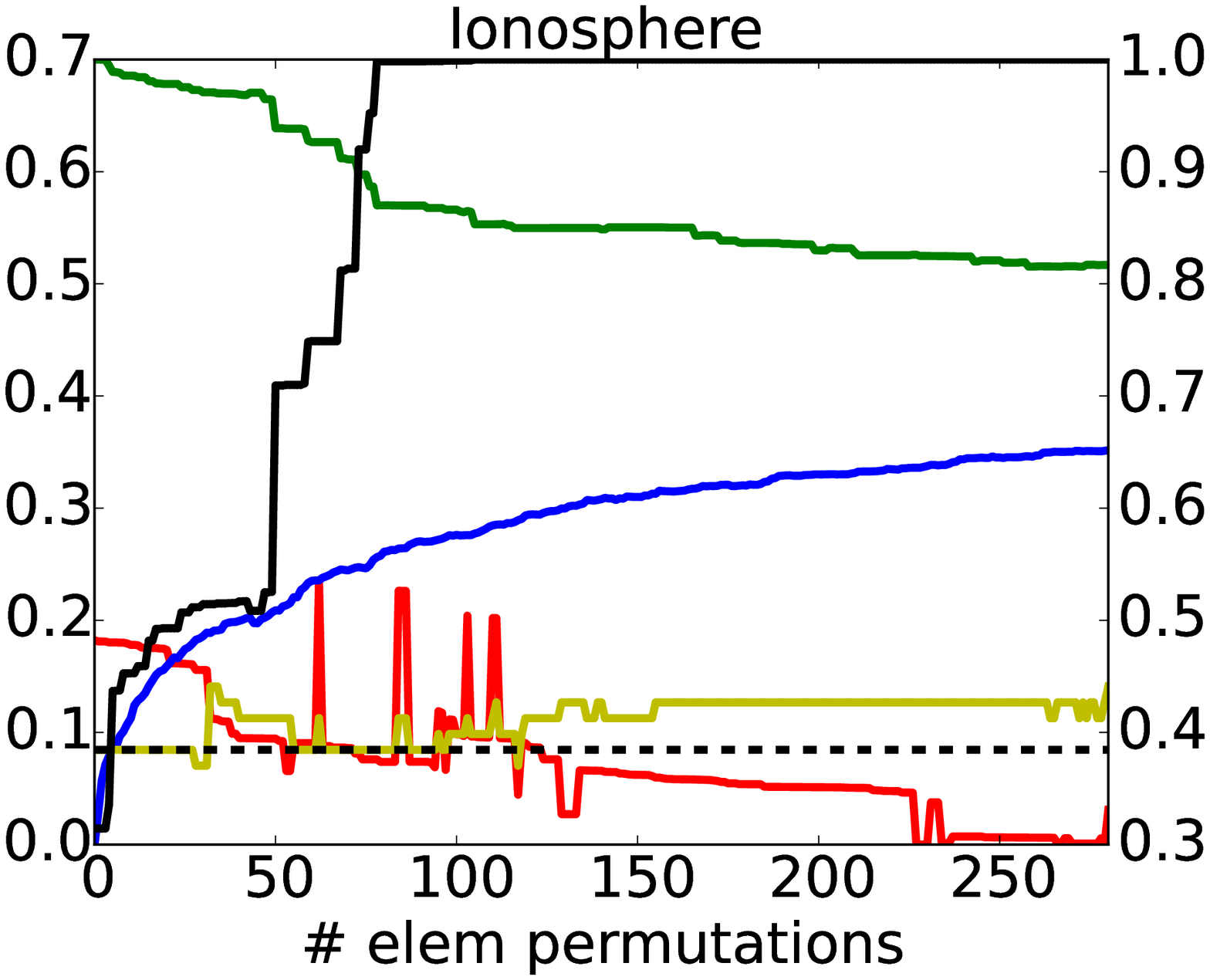} 
& \includegraphics[height=4.50cm]{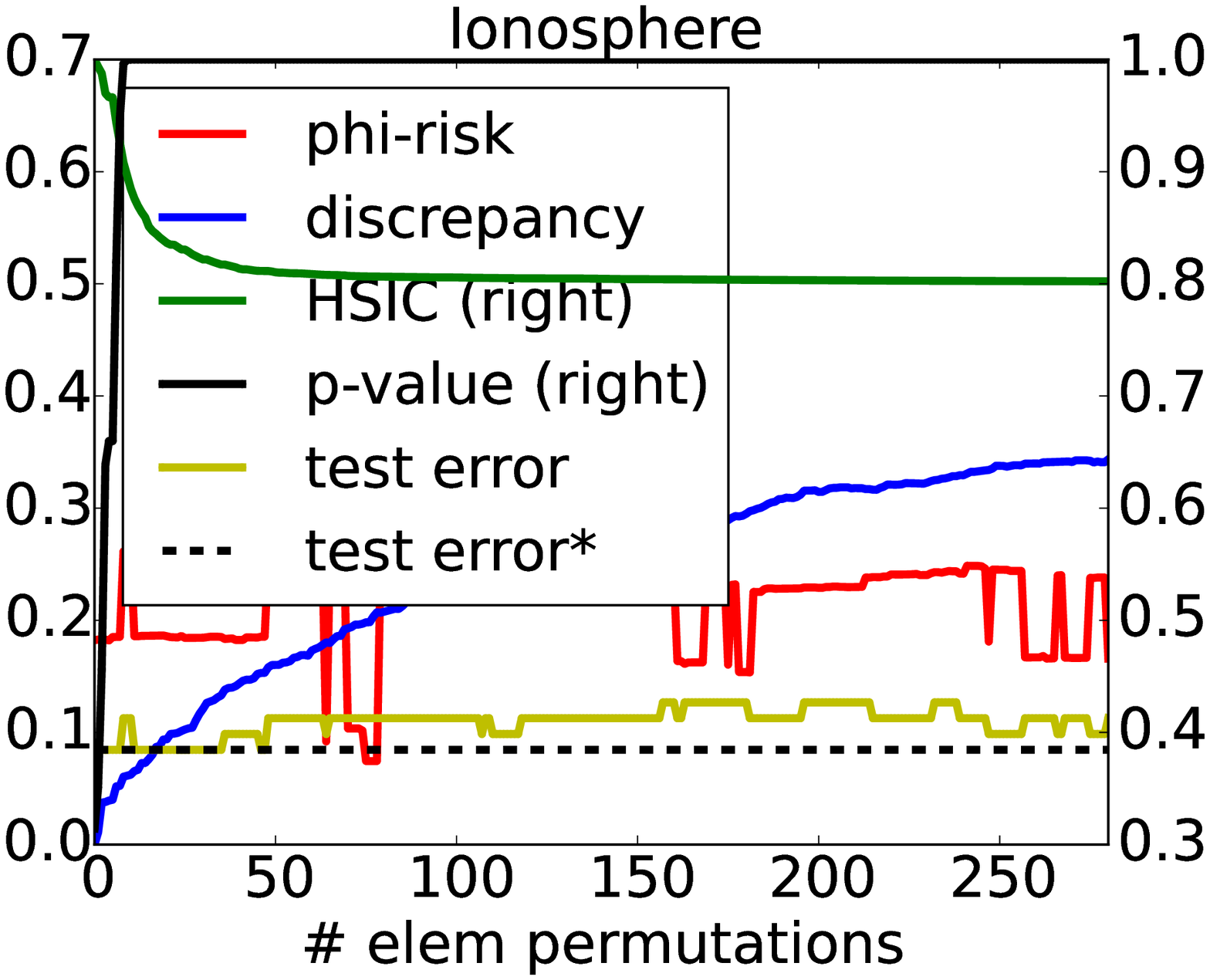}\\
Data optimisation & \hsic~reduction\\ \hline\hline
\end{tabular}
\caption{Results on domain Ionosphere. Left: Data optimisation;
  right: \hsic~reduction. 
Color codes are the same on
  all plots. See text  for details.}\label{res-iono}
\end{center}
\end{table}

\begin{table}[t]
    \centering
\begin{center}
\begin{tabular}{cc}\hline\hline
\includegraphics[height=4.50cm]{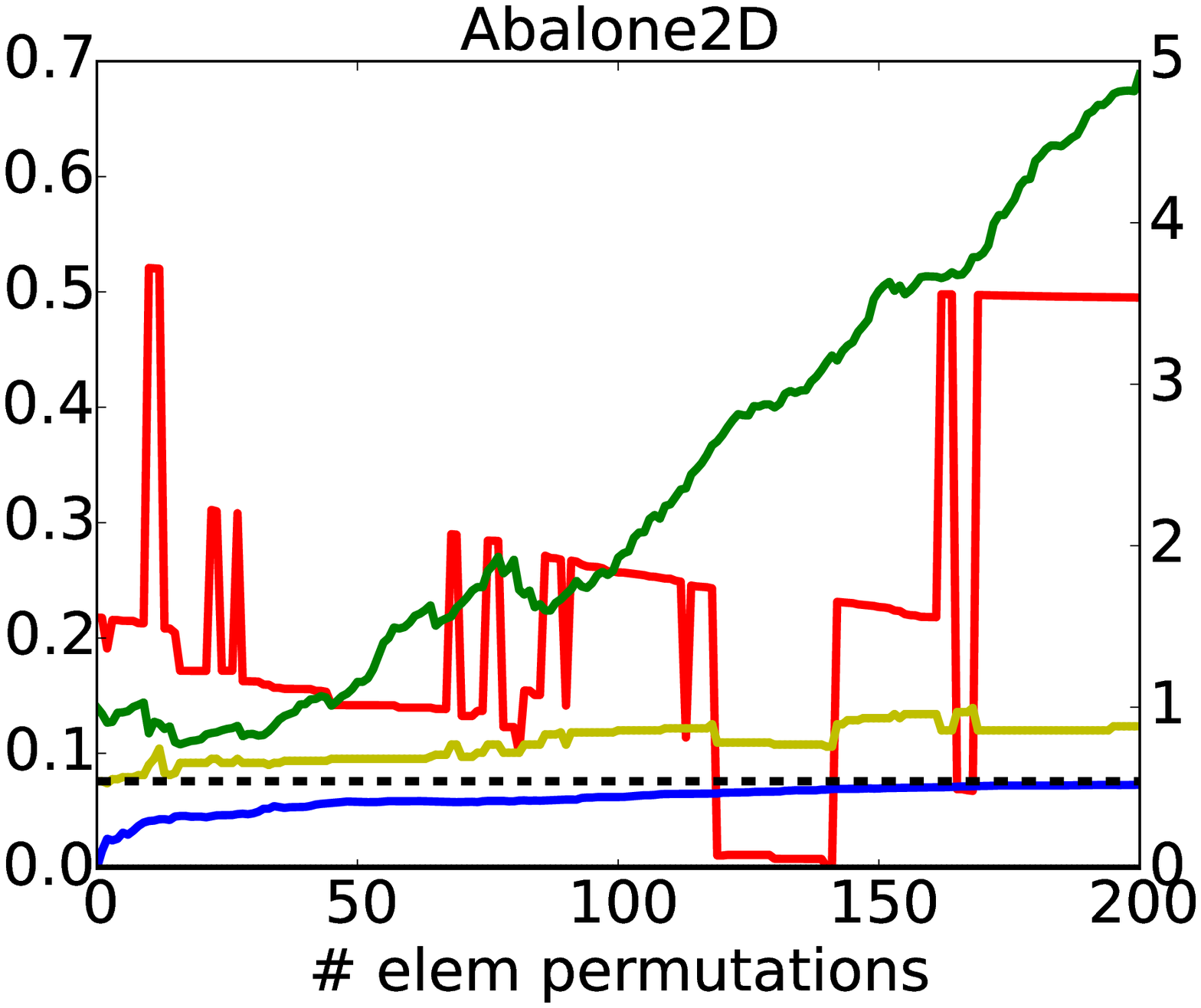}
& \includegraphics[height=4.50cm]{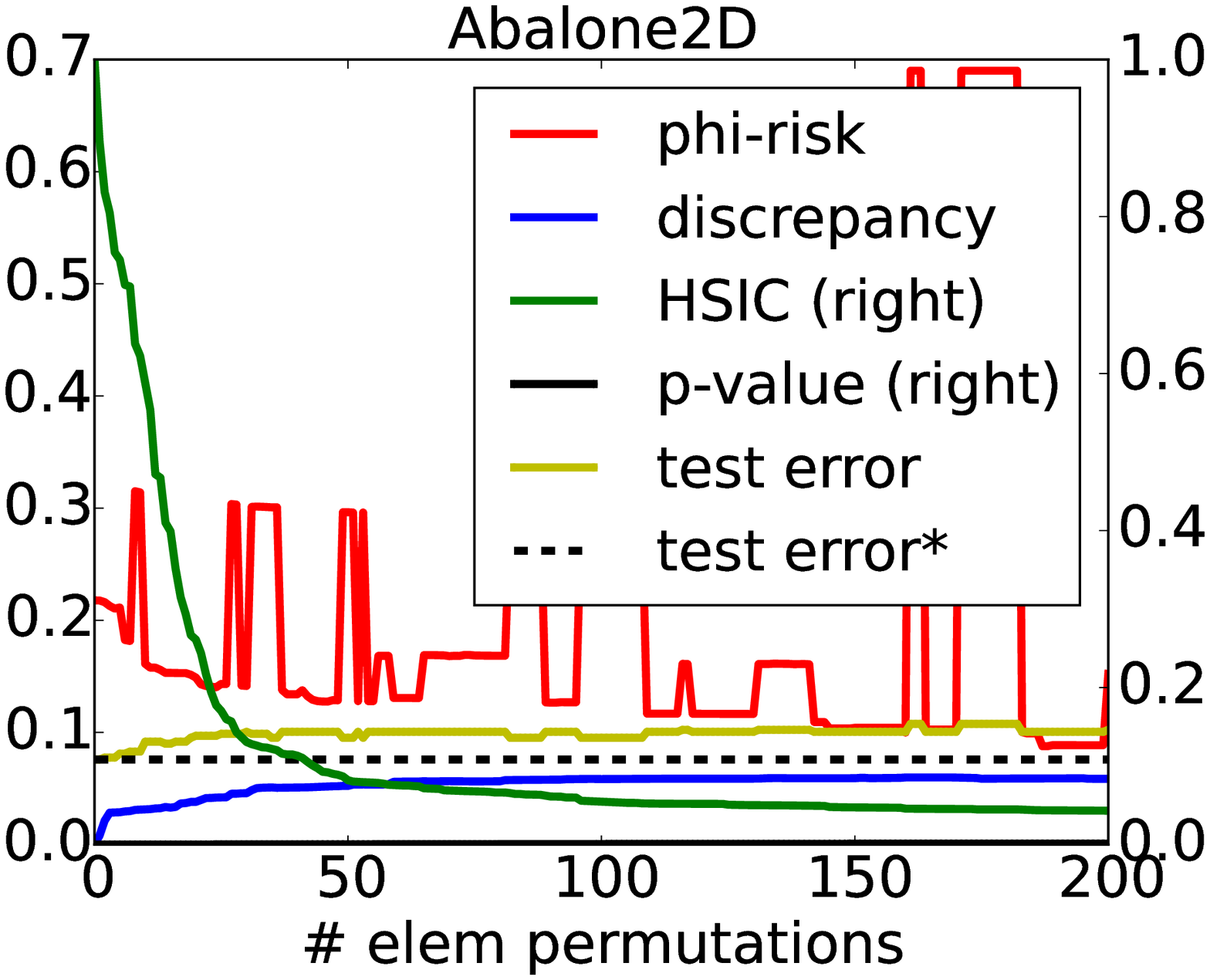}\\
Data optimisation & \hsic~reduction\\ \hline\hline
\end{tabular}
\caption{Results on domain Abalone. Left: Data optimisation;
  right: \hsic~reduction. 
Color codes are the same on
  all plots. The scale of the $p$-value curve is not the same as in
  the main file: here, its scale is the same as for the \hsic~curve,
  which explains why it seems to be flat while the value for the first
  iterations is the zero-machine and the values for the last exceed
  one per thousand. See text  for details.}\label{res-aba}
\end{center}
\end{table}

\begin{table}[t]
    \centering
\begin{center}
\begin{tabular}{cc}\hline\hline
\includegraphics[height=4.50cm]{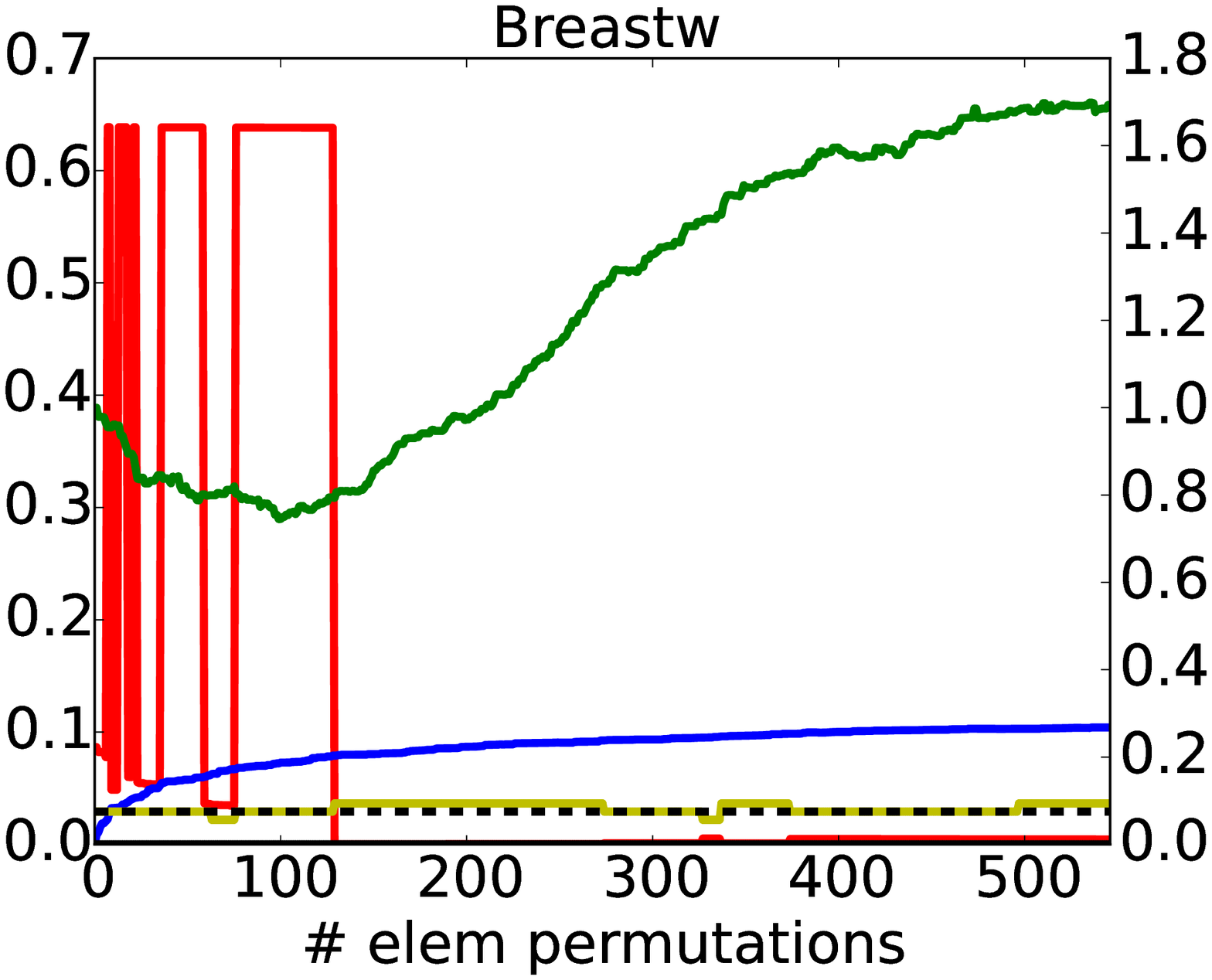}
& \includegraphics[height=4.50cm]{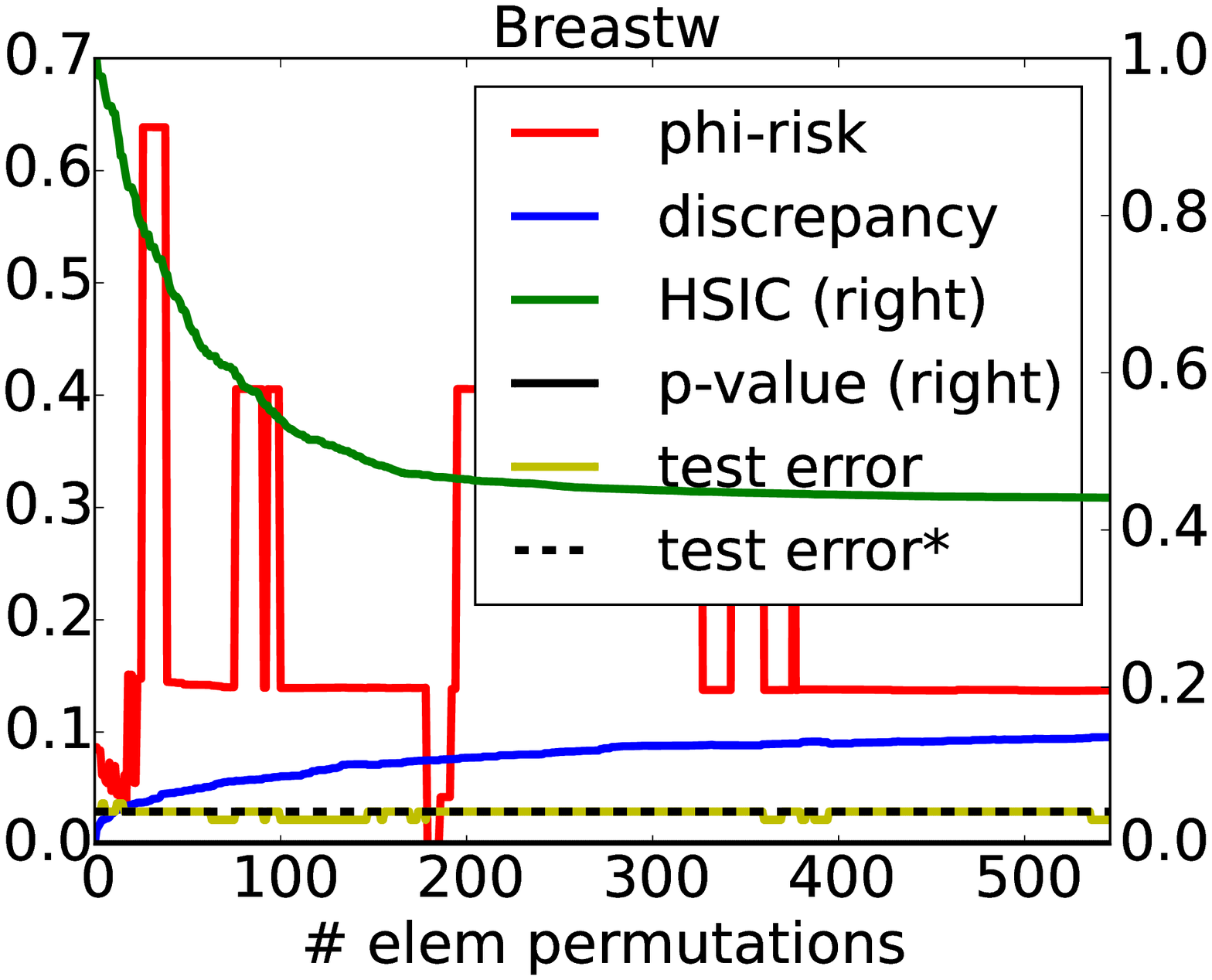}\\
Data optimisation & \hsic~reduction\\ \hline\hline
\end{tabular}
\caption{Results on domain BreastWisc. Left: Data optimisation;
  right: \hsic~reduction. 
Color codes are the same on
  all plots. See text  for details.}\label{res-breast}
\end{center}
\end{table}

\begin{table}[t]
    \centering
\begin{center}
\begin{tabular}{cc}\hline\hline
\includegraphics[height=4.50cm]{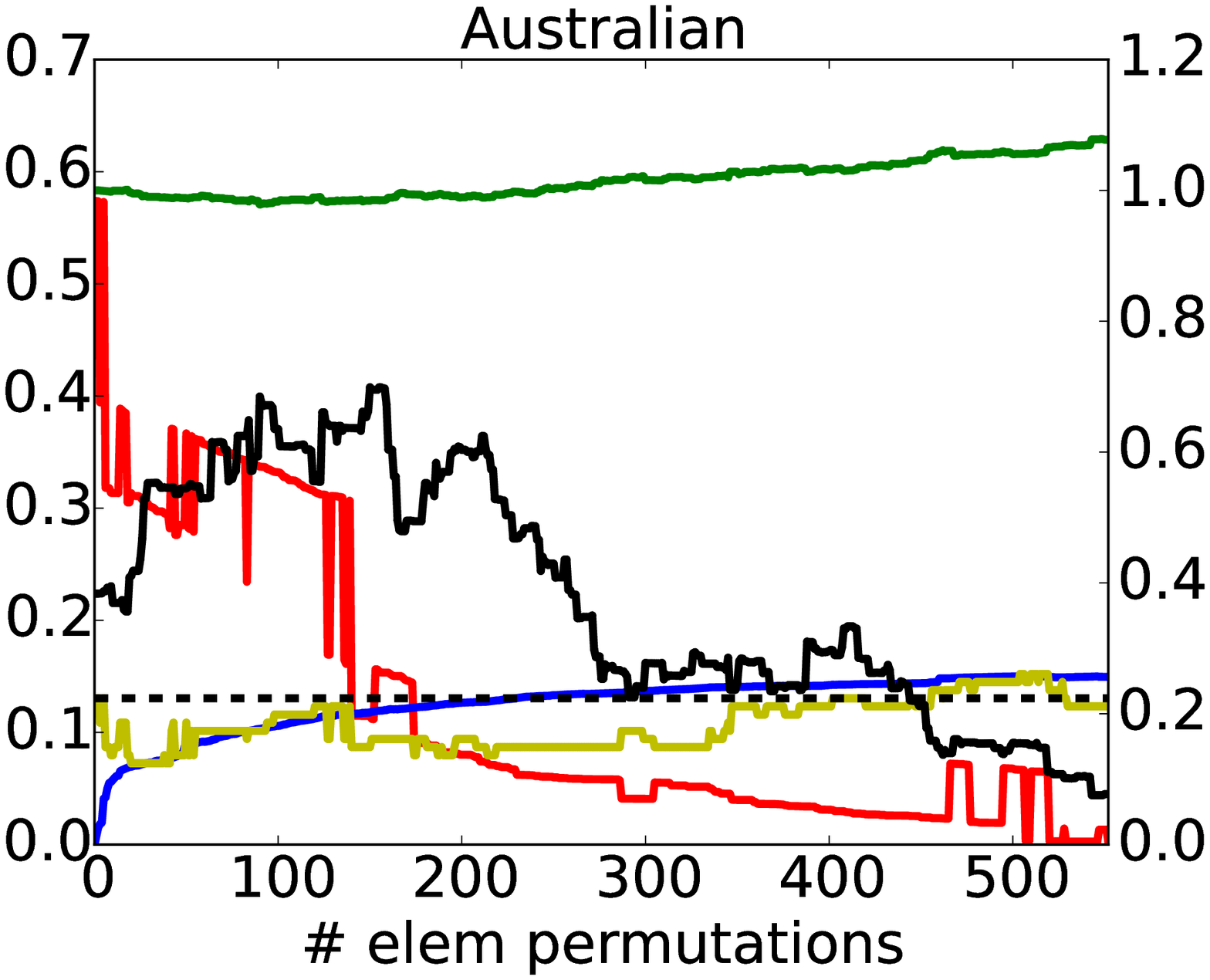}
& \includegraphics[height=4.50cm]{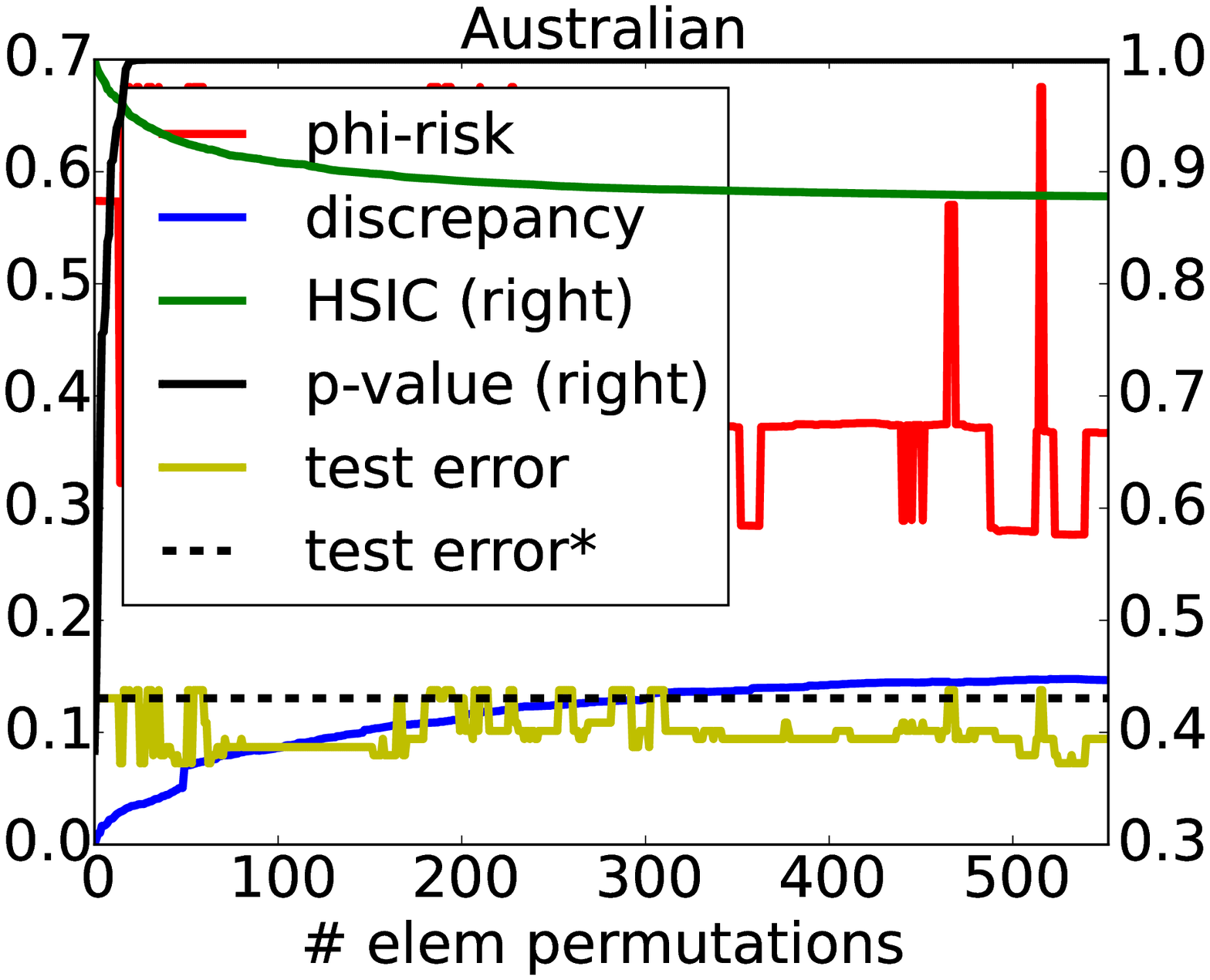}\\
Data optimisation & \hsic~reduction\\ \hline\hline
\end{tabular}
\caption{Results on domain Australian. Left: Data optimisation;
  right: \hsic~reduction. 
Color codes are the same on
  all plots. See text  for details.}\label{res-australian}
\end{center}
\end{table}

\begin{table}[t]
    \centering
\begin{center}
\begin{tabular}{cc}\hline\hline
\includegraphics[height=4.50cm]{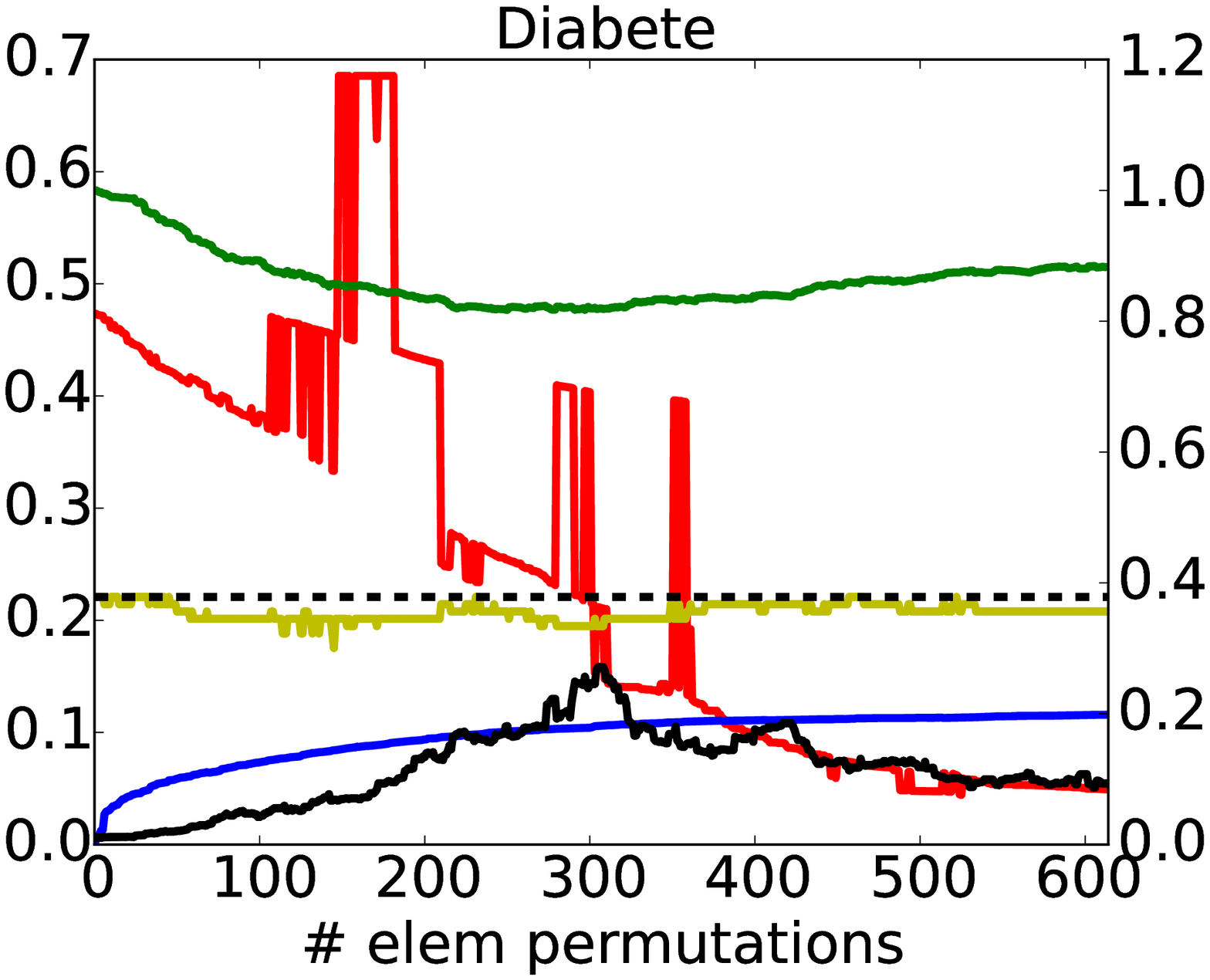}
& \includegraphics[height=4.50cm]{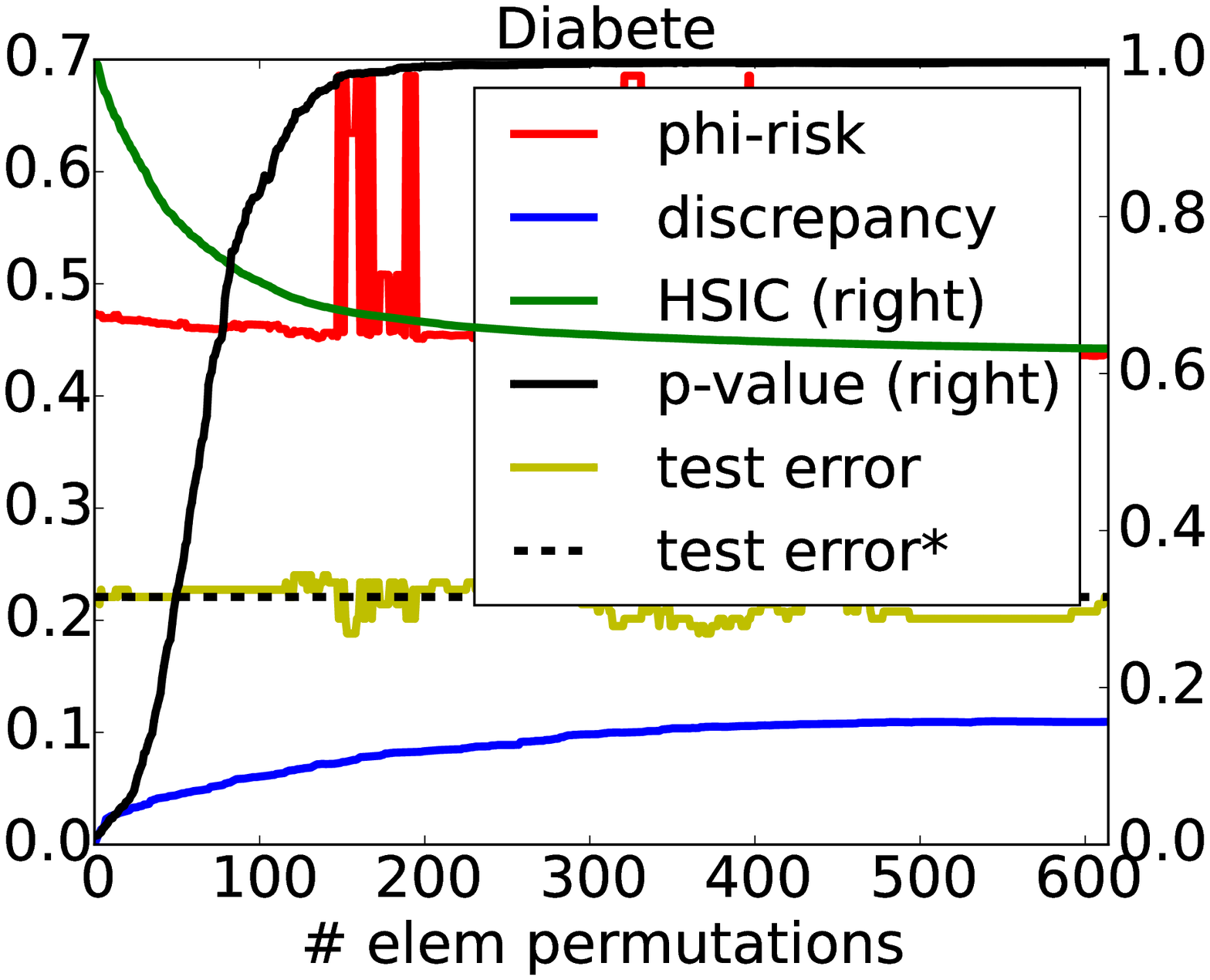}\\
Data optimisation & \hsic~reduction\\ \hline\hline
\end{tabular}
\caption{Results on domain Diabete$\_$scale. Left: Data optimisation;
  right: \hsic~reduction. 
Color codes are the same on
  all plots. See text  for details.}\label{res-diabetes}
\end{center}
\end{table}

\begin{table}[t]
    \centering
\begin{center}
{\scriptsize
\begin{tabular}{c}
\multicolumn{1}{c}{\hspace{-0.25cm}\begin{tabular}{>{\centering\arraybackslash} m{0.1cm}ccc||c}\hline\hline
\vspace{-2cm}\begin{sideways}{\footnotesize \hspace{-0.11cm}  data optimisation}\end{sideways} & \hspace{-0.11cm} \includegraphics[height=2.83cm]{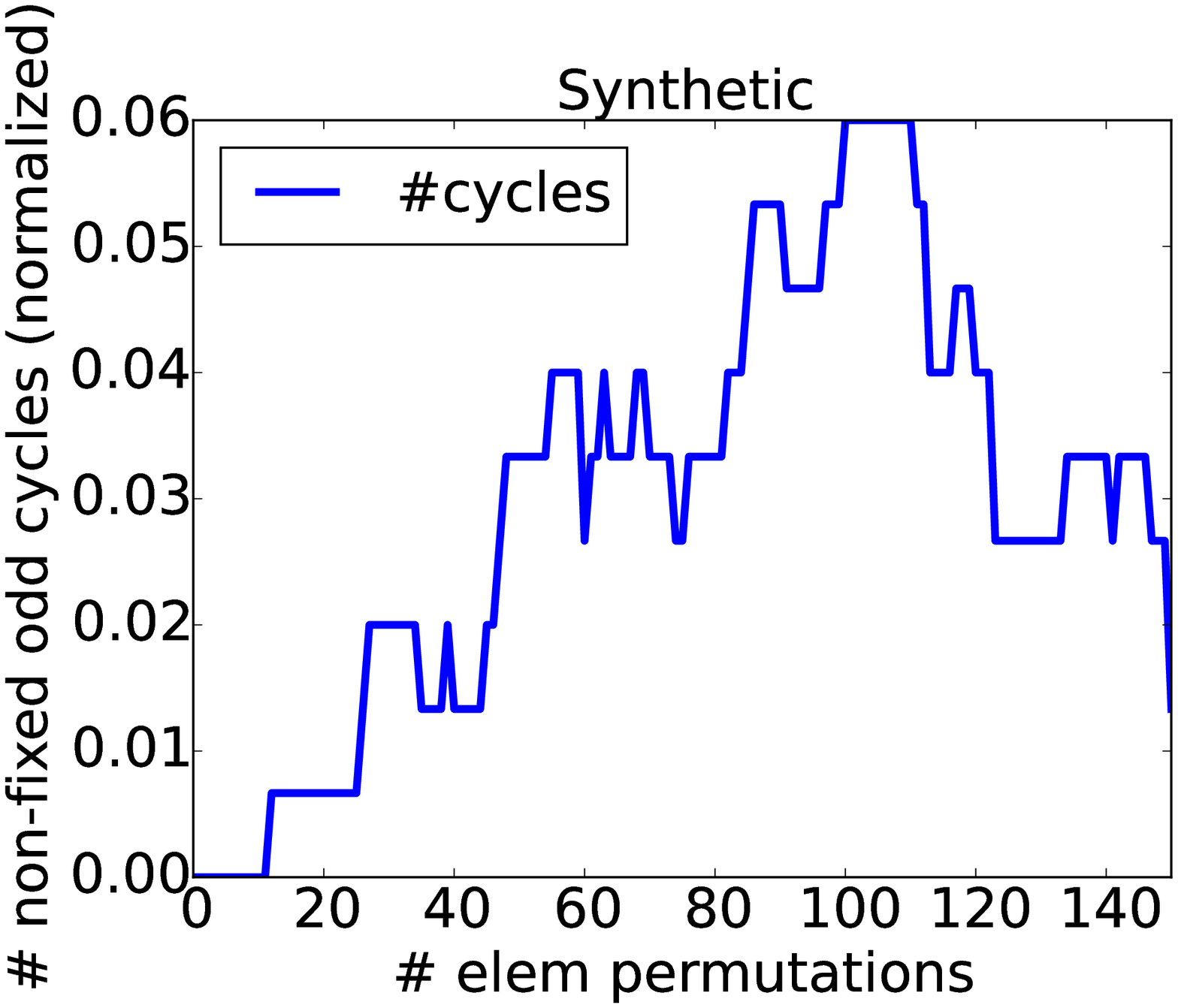}\hspace{-0.11cm} 
& \hspace{-0.11cm} \includegraphics[height=2.83cm]{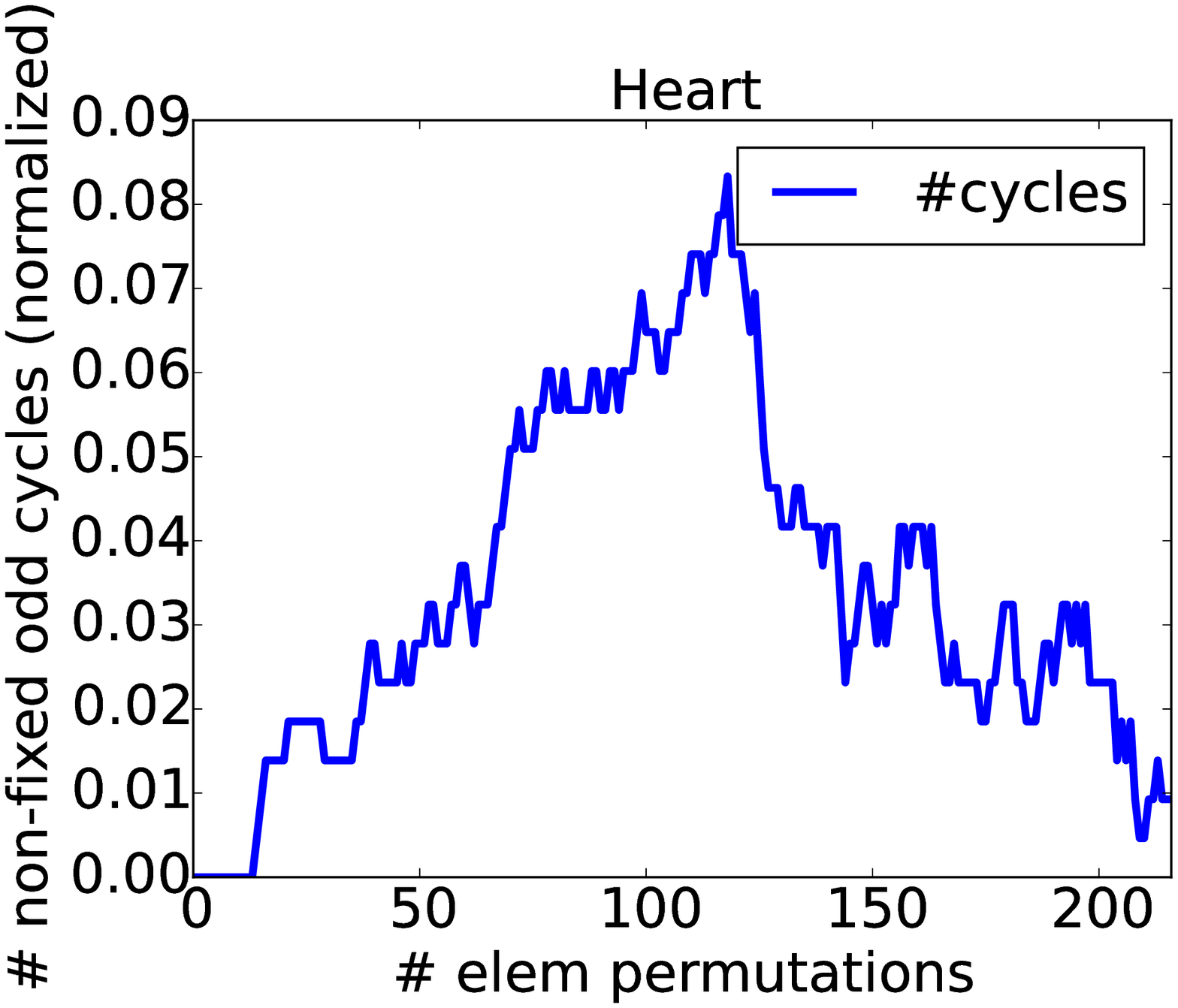}\hspace{-0.11cm} 
&
\hspace{-0.11cm} \includegraphics[height=2.83cm]{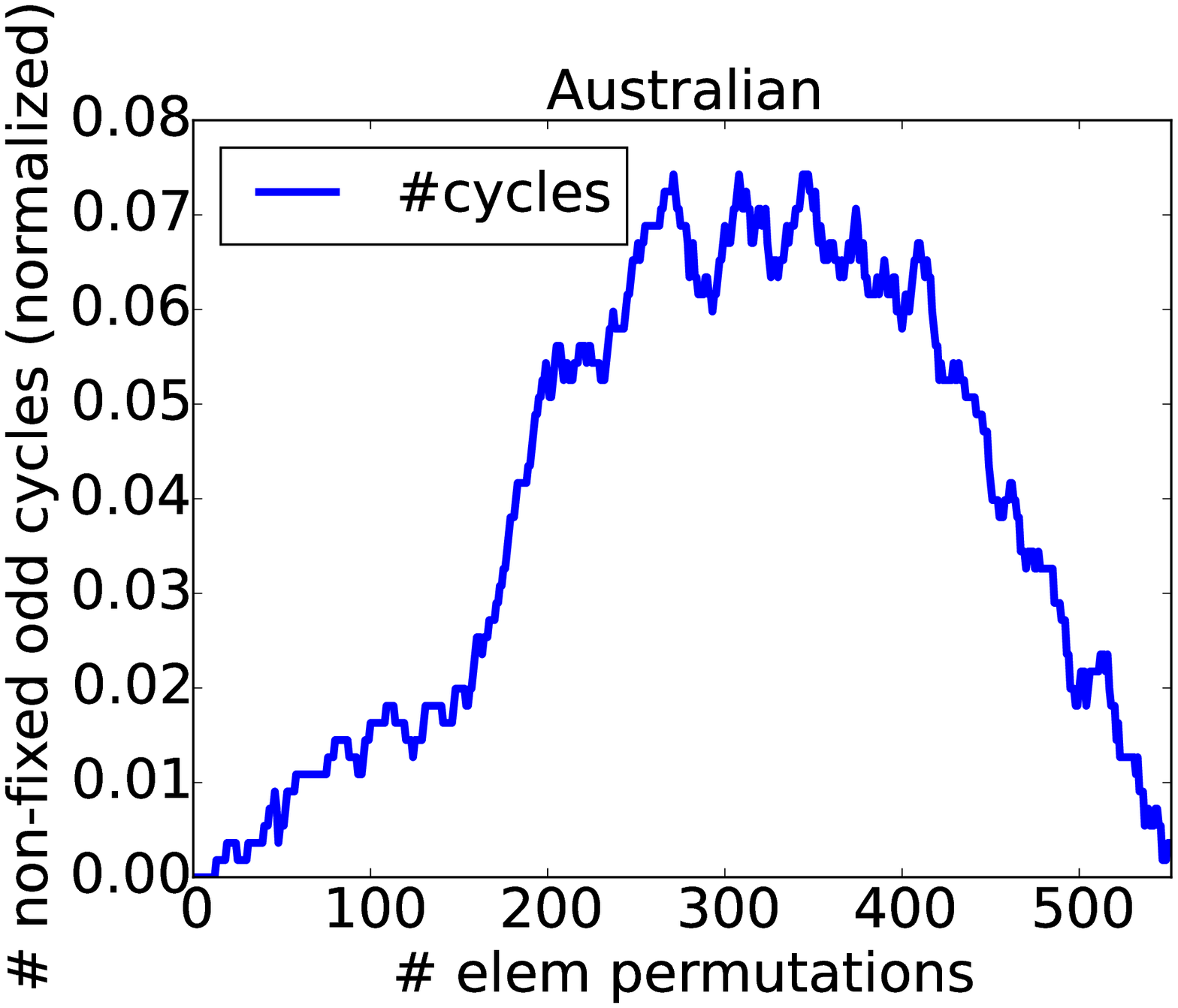}\hspace{-0.11cm} 
&
\hspace{-0.11cm} \includegraphics[height=2.83cm]{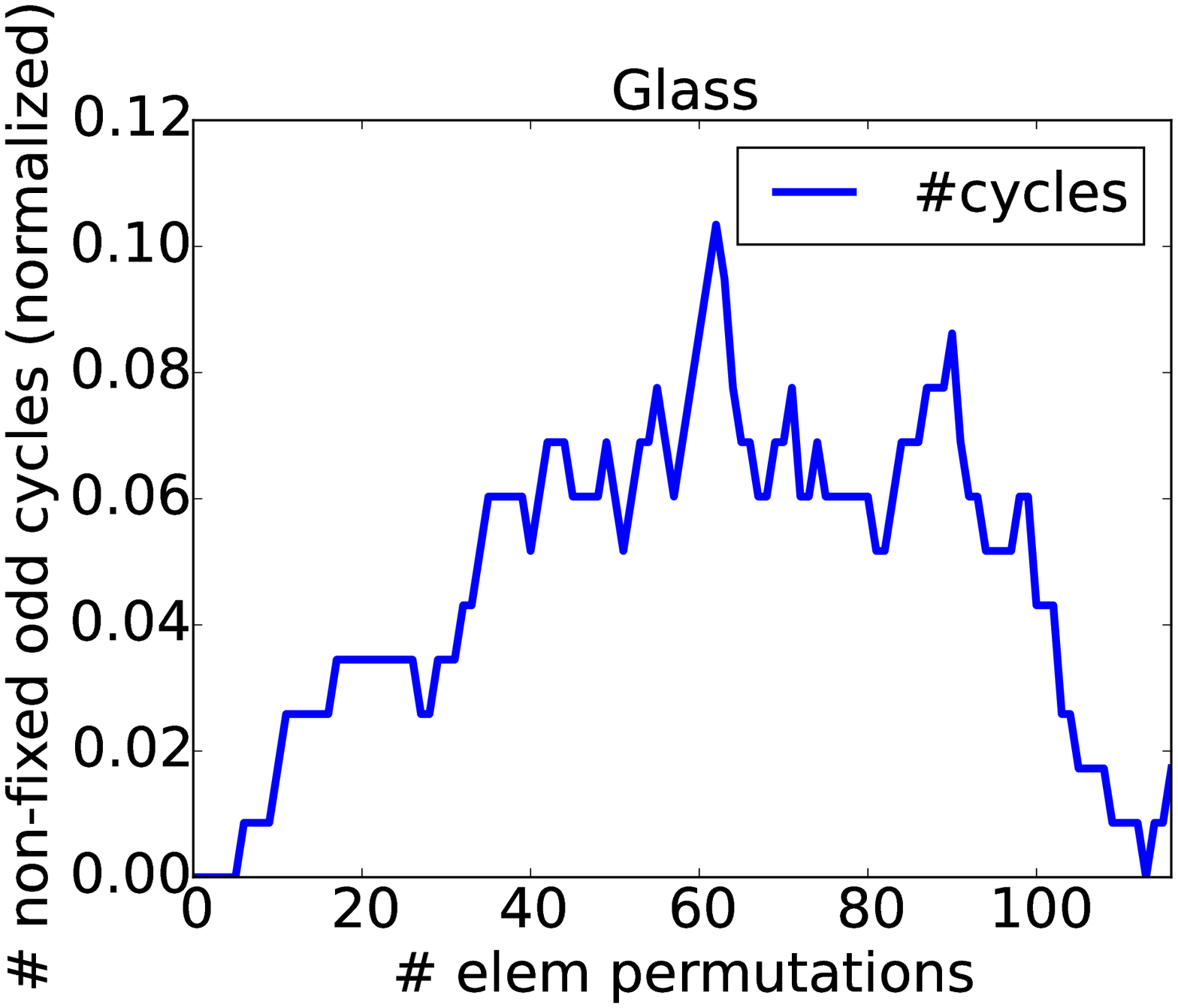}\hspace{-0.11cm} \\ \hline\hline
\end{tabular}
}
\end{tabular}
}

\caption{Number of odd cycles (excluding fixed points, normalized by $m$) for the data
  optimization experiments in Table \ref{t-applis}.}\label{t-applis-oc}
\end{center}
\end{table}

\clearpage
\newpage
\subsection{Comparisons of block-class vs arbitrary permutations}\label{res-comp-c}

We now compare \perm~as in Algorithm
 \ref{algoMETA} to the one where we relax the constraint that
 permutations must be block-class (implying the invariance of the
 mean operator). See Tables \ref{res-comp}, \ref{res-comp2}. The results are a clear
 advocacy for the constraint, as relaxing it brings
 poor results, from both the $\varphi$-risk and test error standpoints. 

\begin{table}[h]
    \centering
\begin{center}
\begin{tabular}{c|c}\hline\hline
\includegraphics[height=4.10cm]{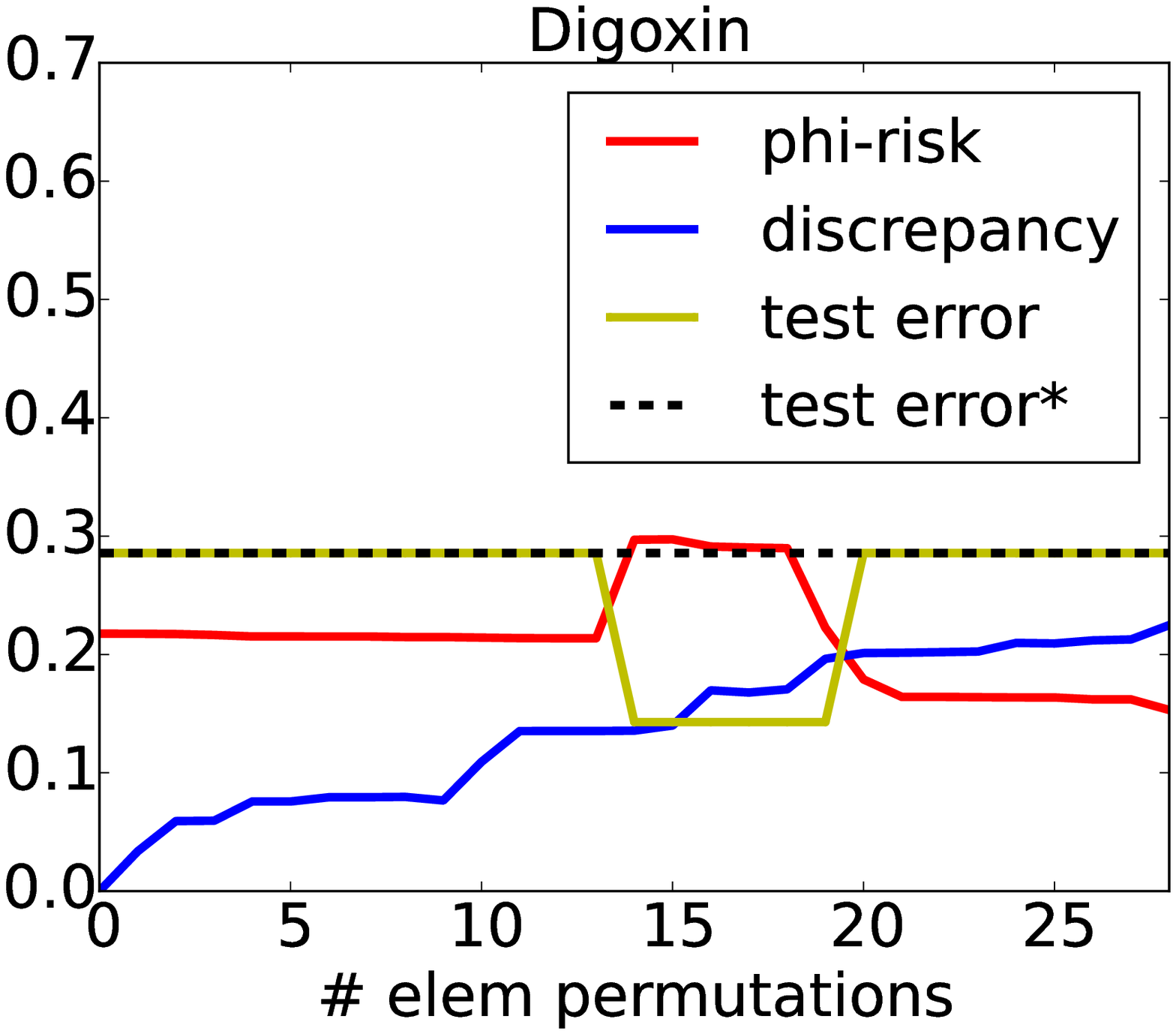}
& \includegraphics[height=4.10cm]{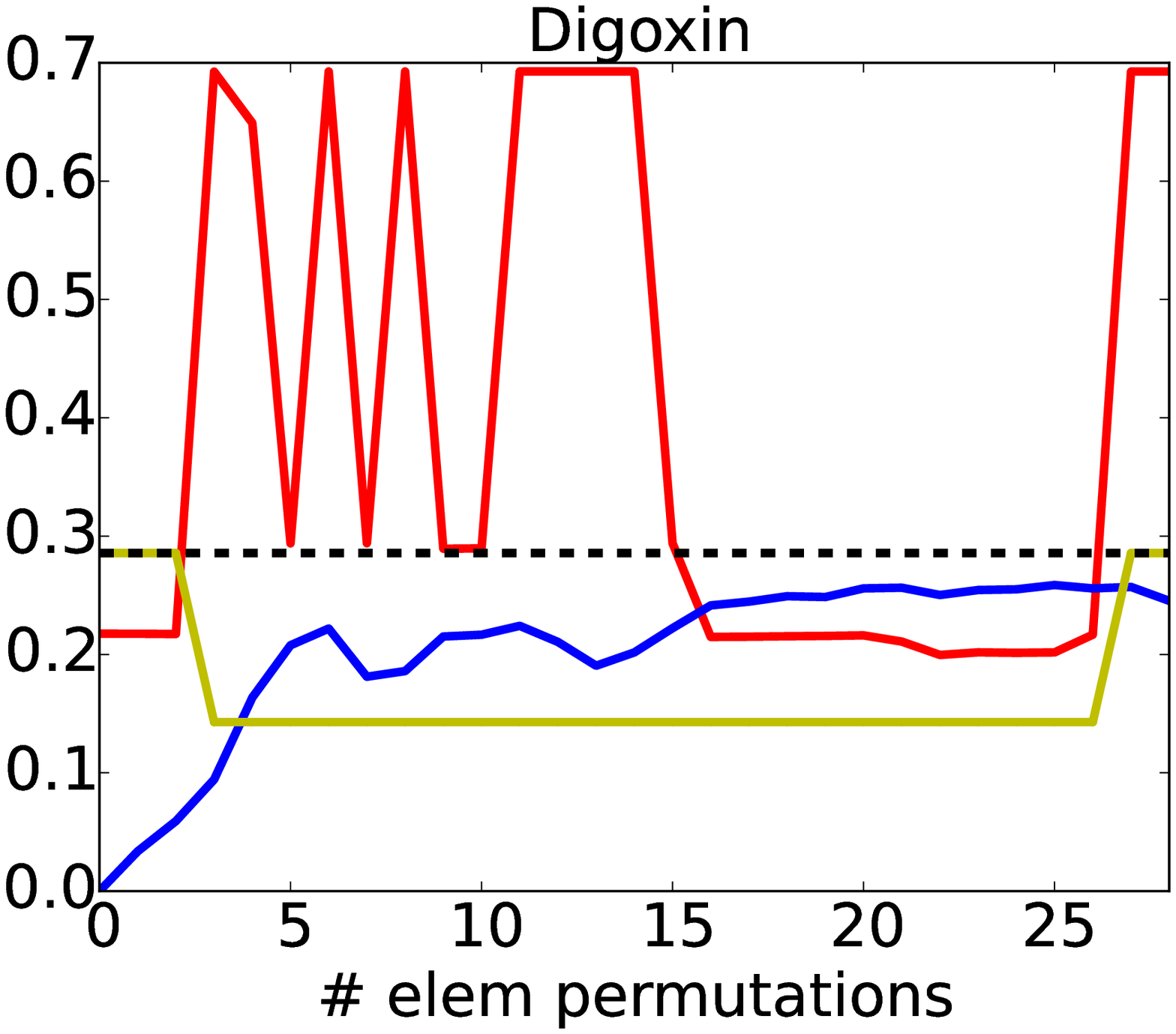}\\ 
$\matrice{m} \in S^*_m$ & $\matrice{m} \in S_m$ \\ \hline
\includegraphics[height=4.10cm]{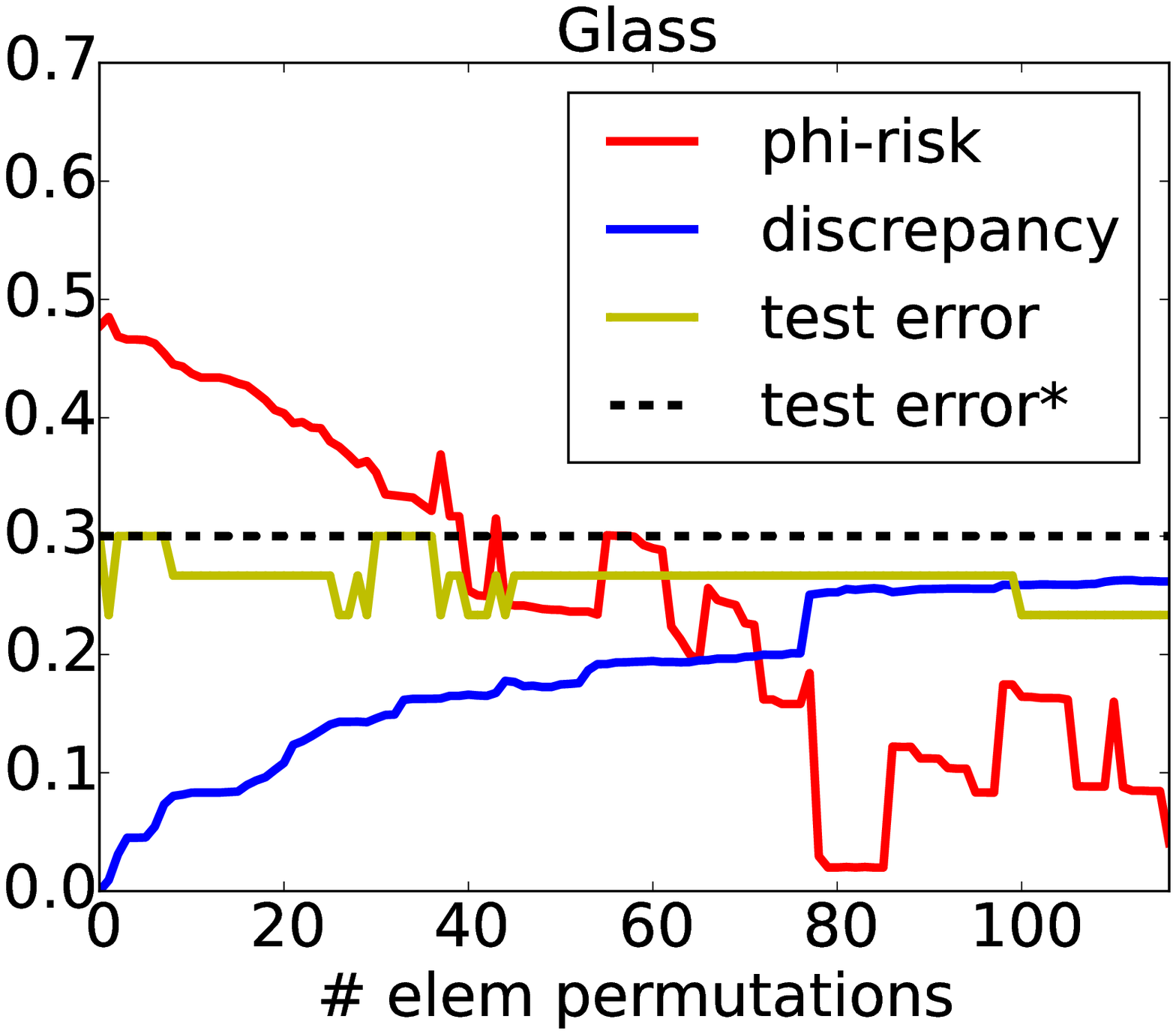}
& \includegraphics[height=4.10cm]{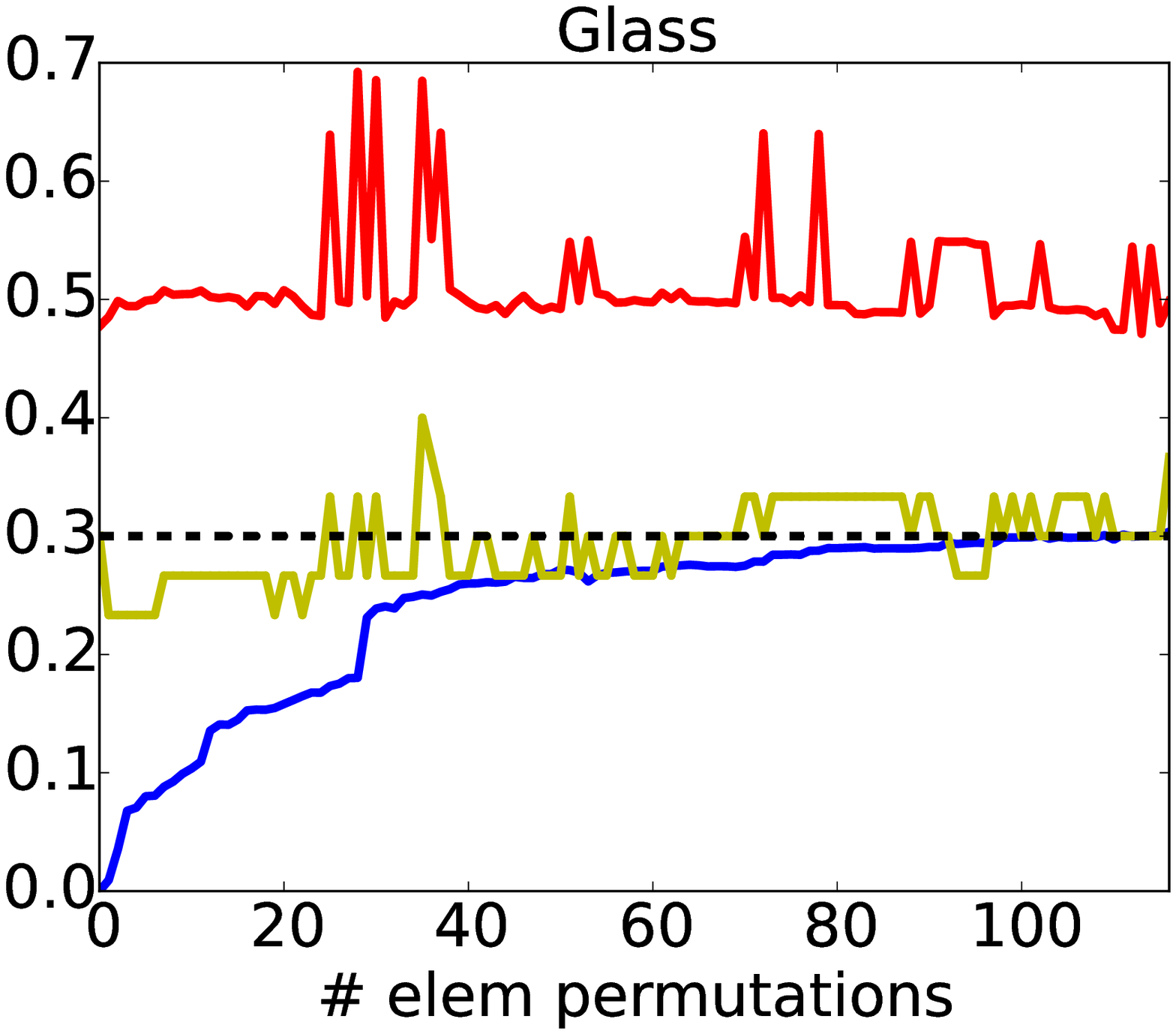}\\ 
$\matrice{m} \in S^*_m$ & $\matrice{m} \in S_m$ \\ \hline
\includegraphics[height=4.10cm]{Archive/do/do_synthetic.eps}
& \includegraphics[height=4.10cm]{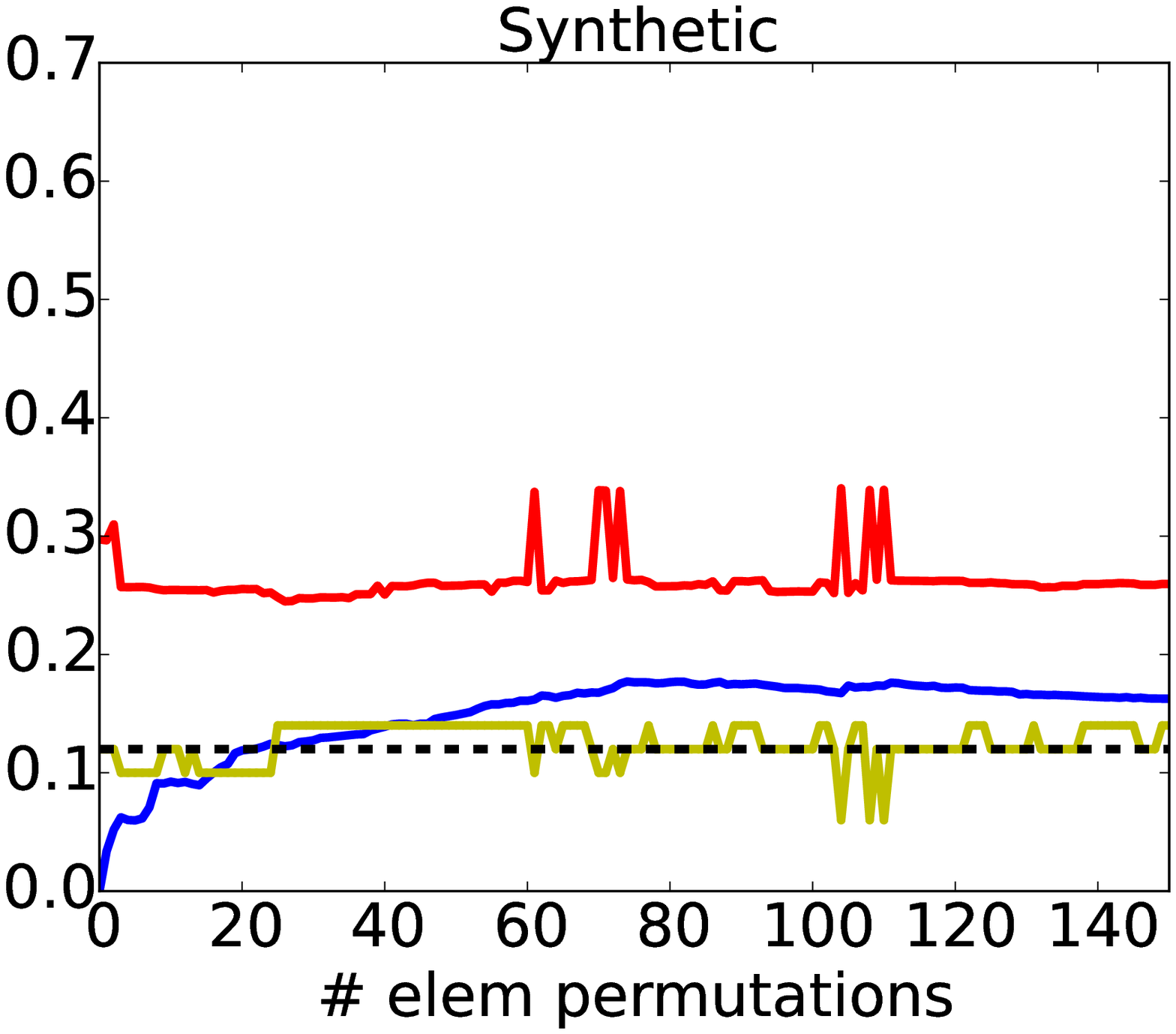}\\ 
$\matrice{m} \in S^*_m$ & $\matrice{m} \in S_m$ \\ \hline
\includegraphics[height=4.10cm]{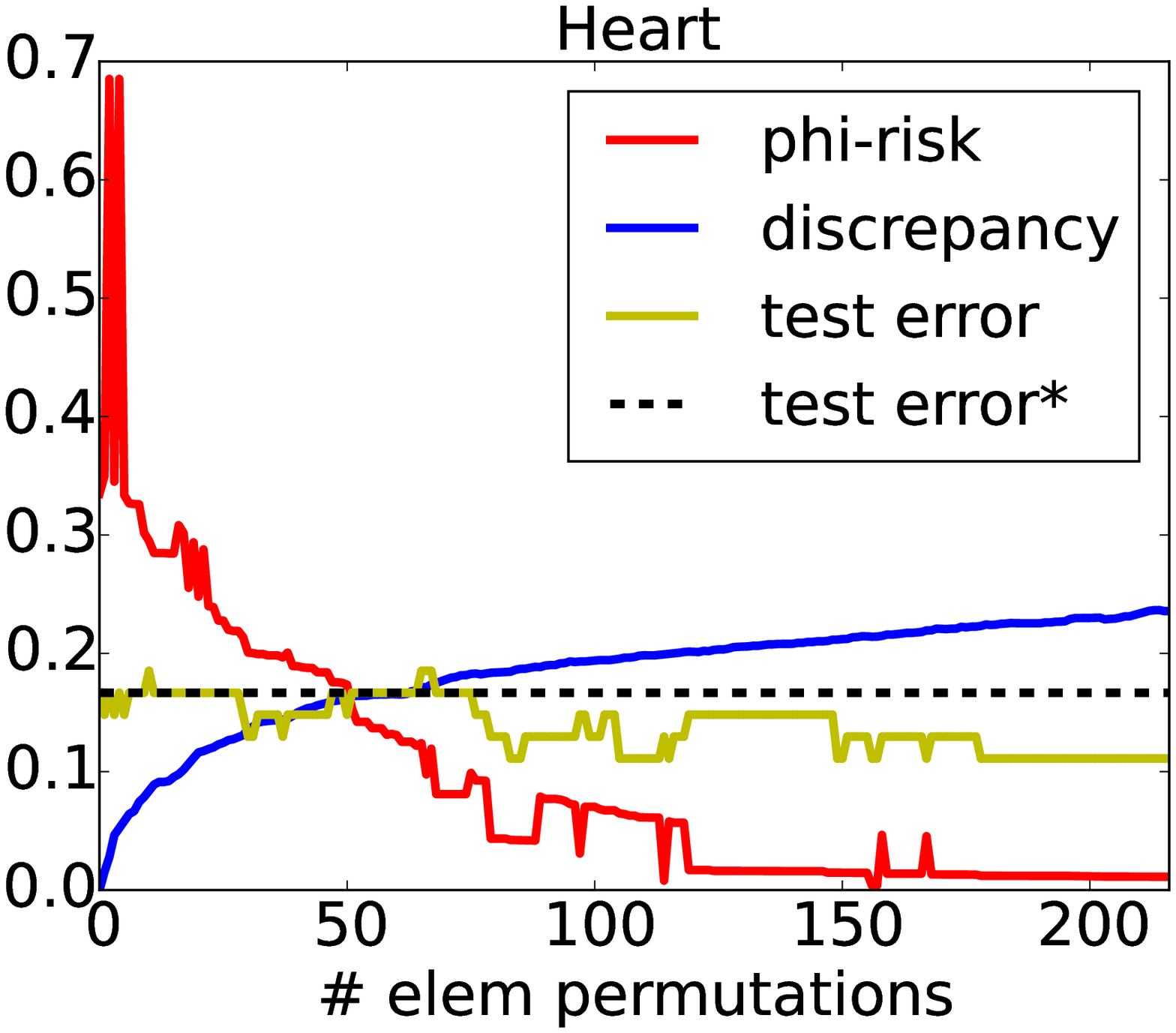}
& \includegraphics[height=4.10cm]{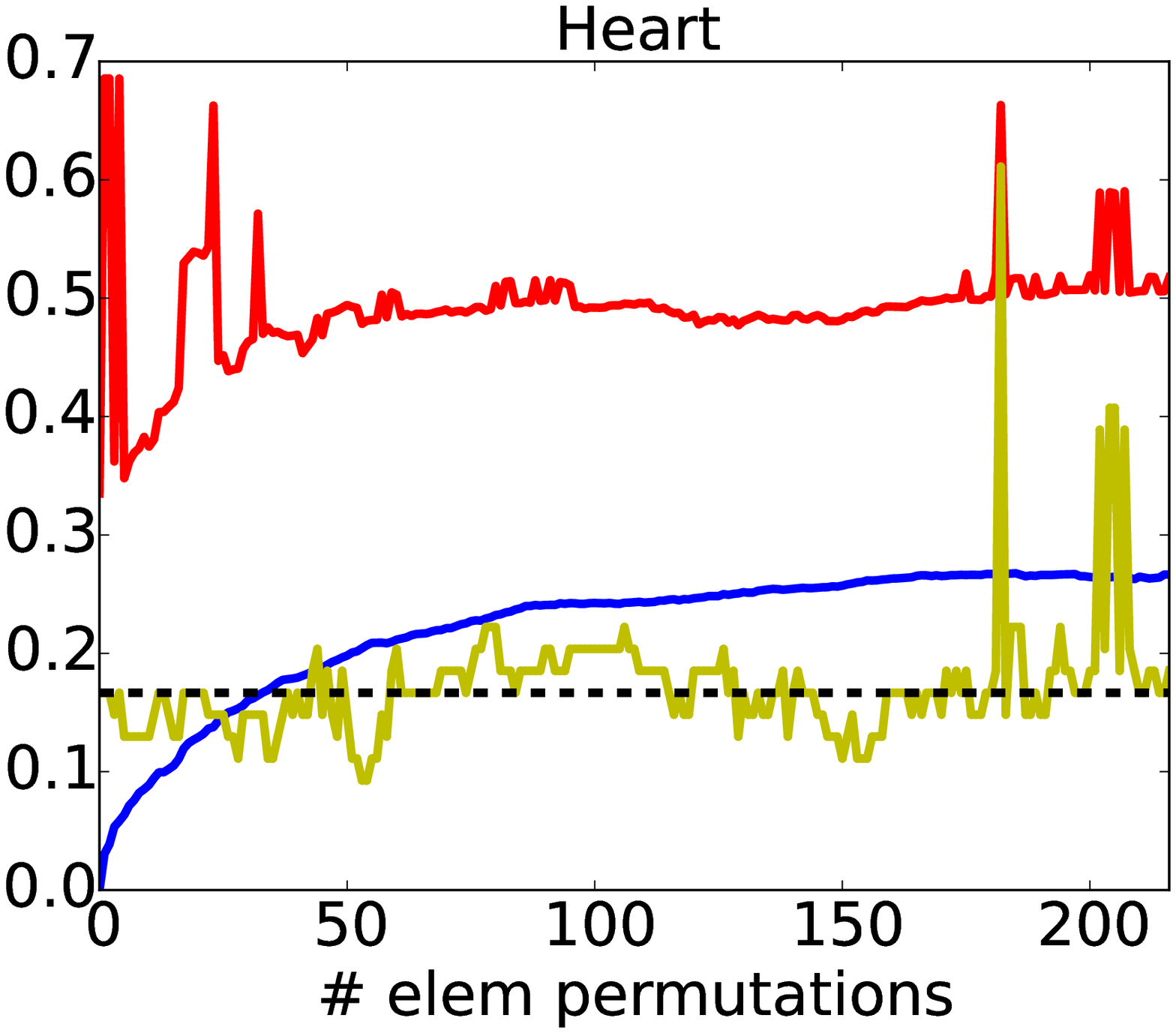}\\ 
$\matrice{m} \in S^*_m$ & $\matrice{m} \in S_m$ \\ \hline
 & \\ \hline\hline
\end{tabular}
\caption{Comparison, for data optimisation, of algorithm \perm~in
  which elementary permutation matrices are constrained to be block-class (left), and \textit{not} constrained to be block-class
  (right).}\label{res-comp}
\end{center}
\end{table}

\begin{table}[h]
    \centering
\begin{center}
\begin{tabular}{c|c}\hline\hline
\includegraphics[height=4.10cm]{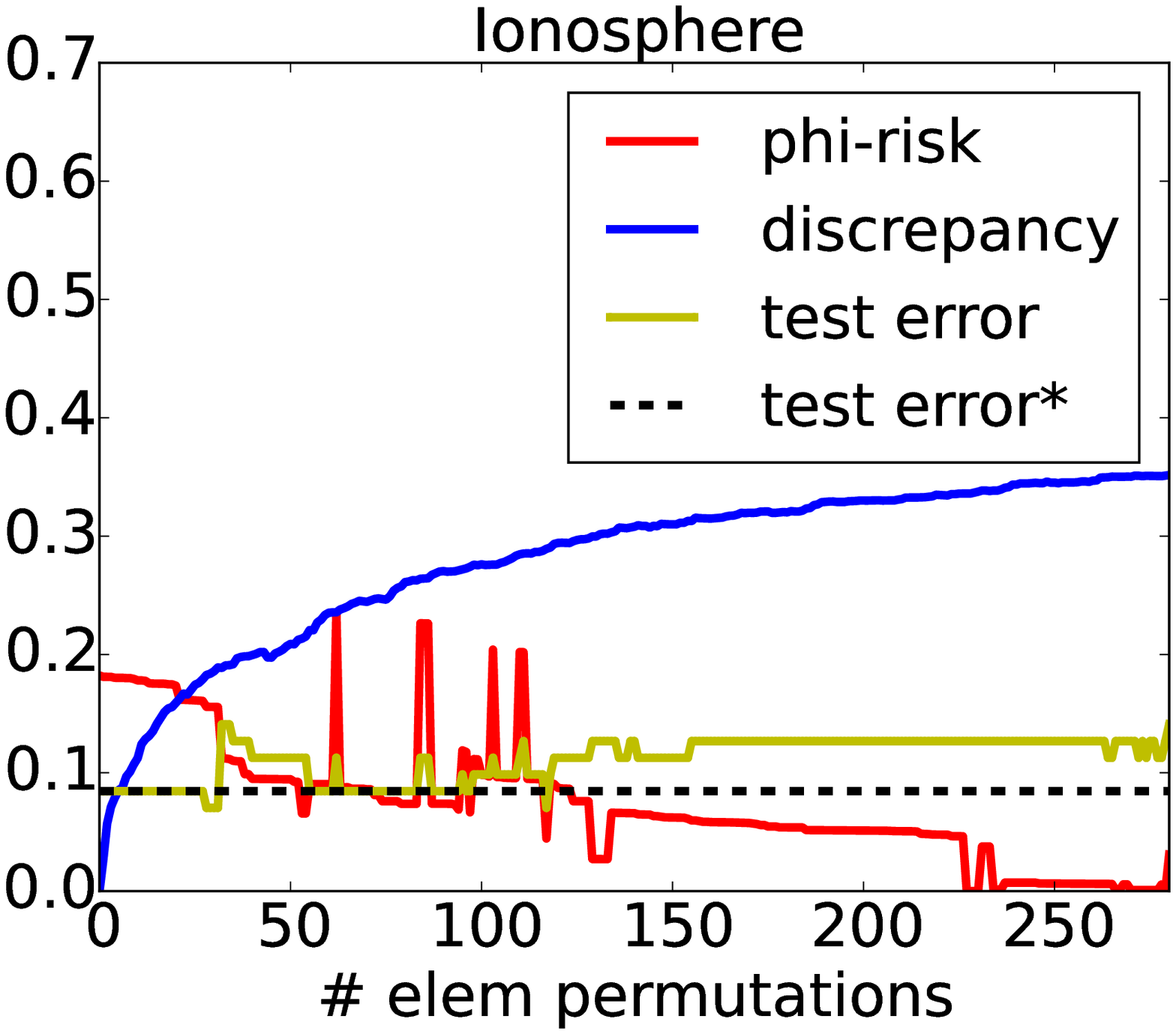}
& \includegraphics[height=4.10cm]{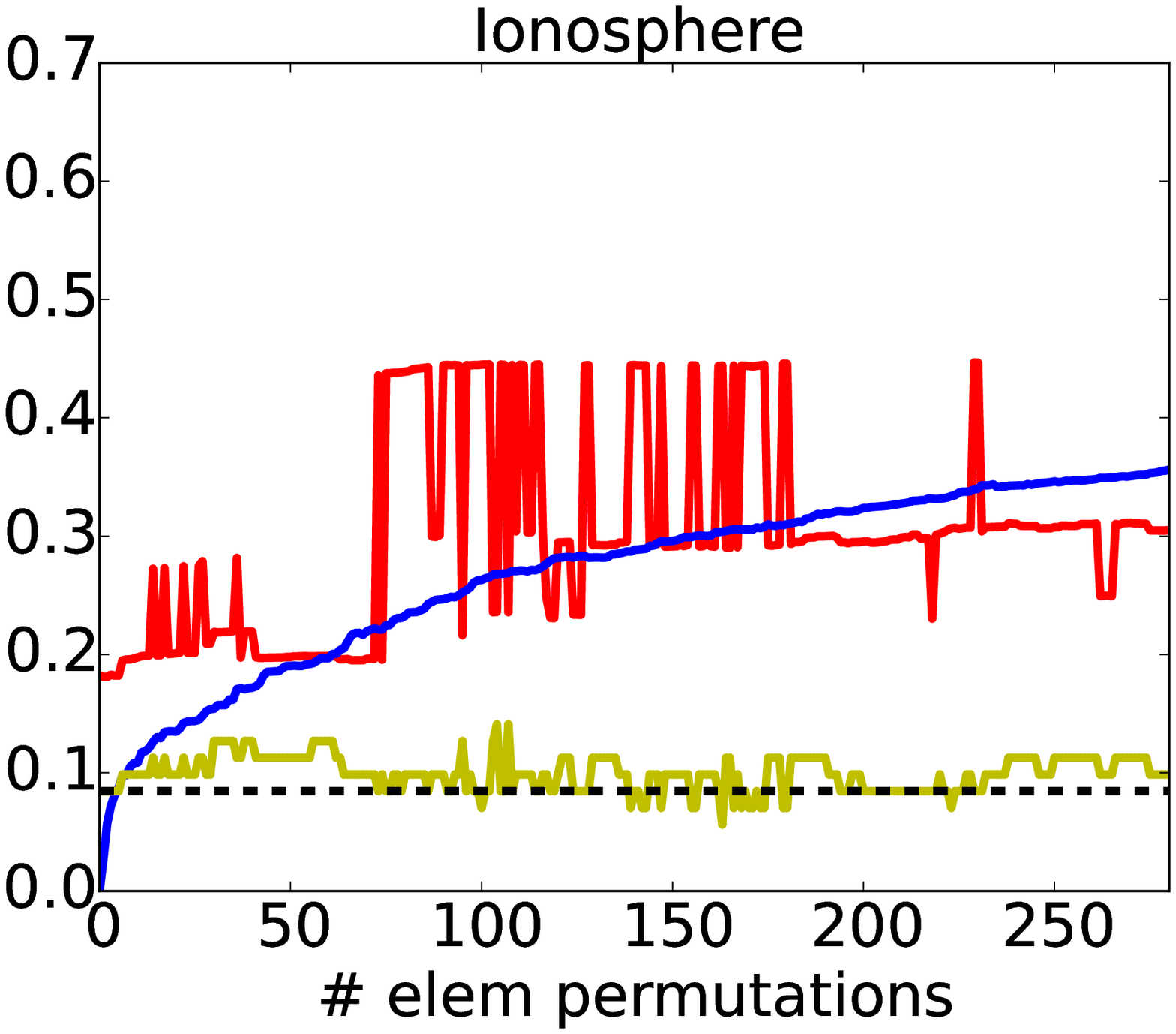}\\ 
$\matrice{m} \in S^*_m$ & $\matrice{m} \in S_m$ \\ \hline
\includegraphics[height=4.10cm]{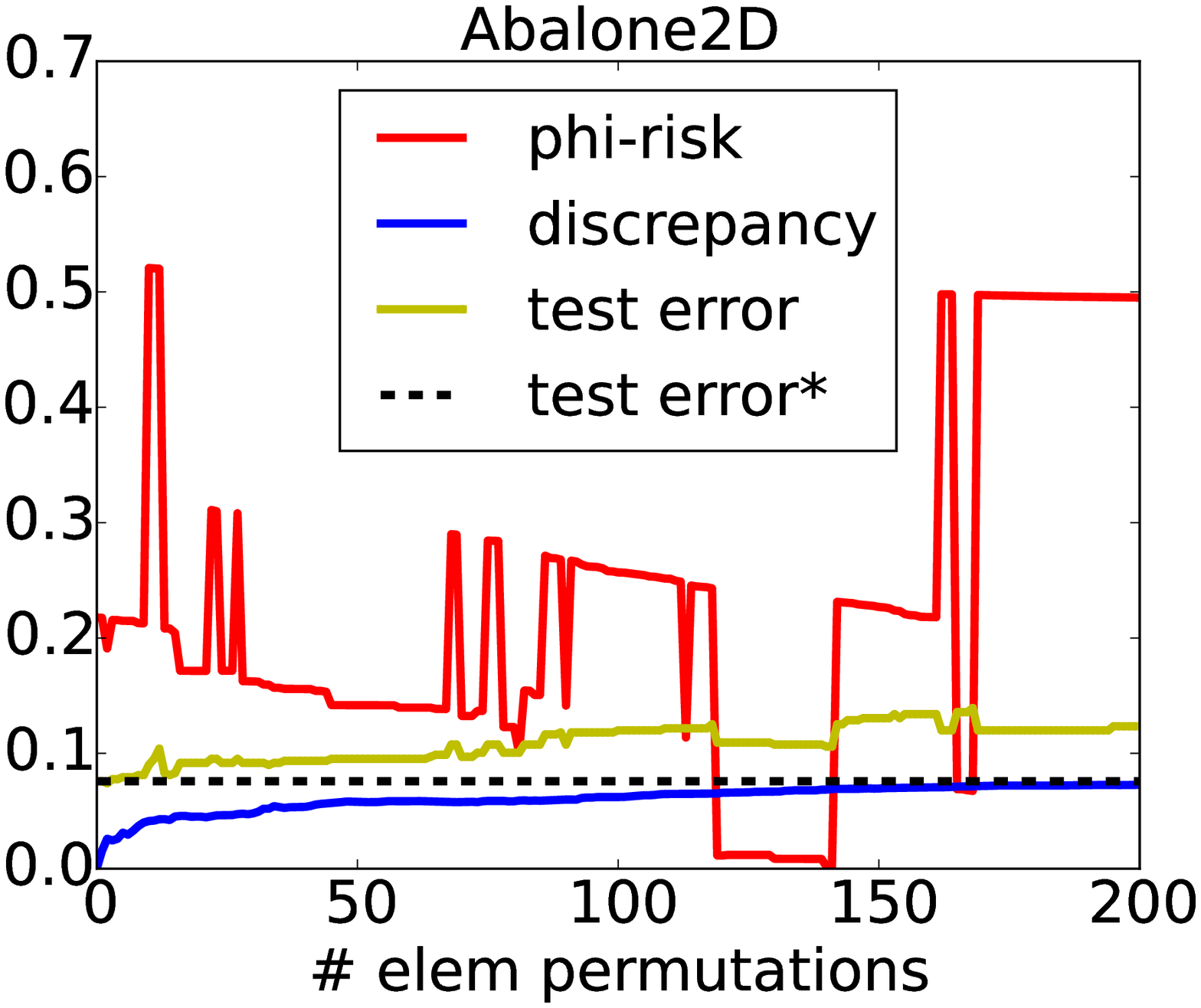}
& \includegraphics[height=4.10cm]{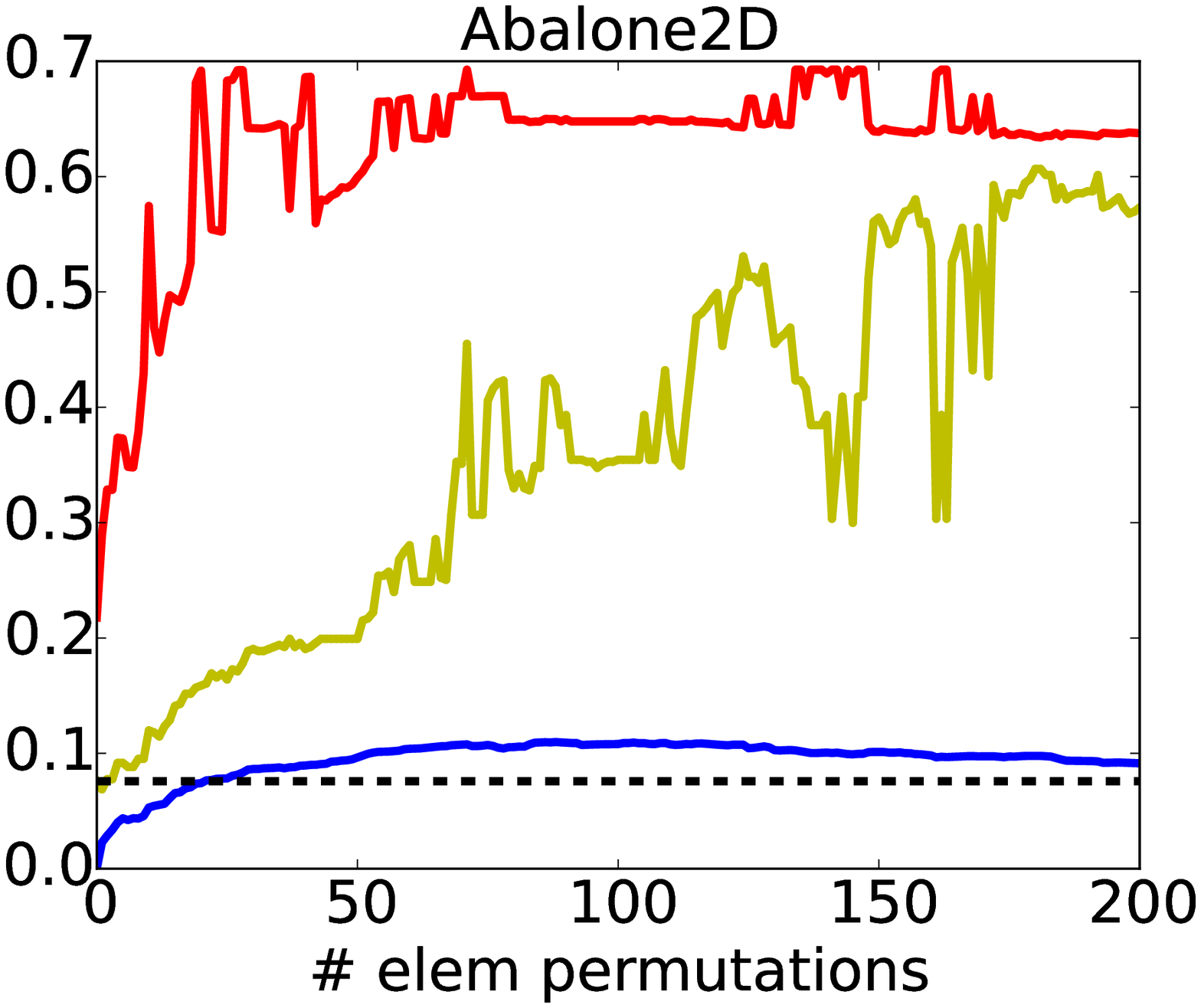}\\ 
$\matrice{m} \in S^*_m$ & $\matrice{m} \in S_m$ \\ \hline
\includegraphics[height=4.10cm]{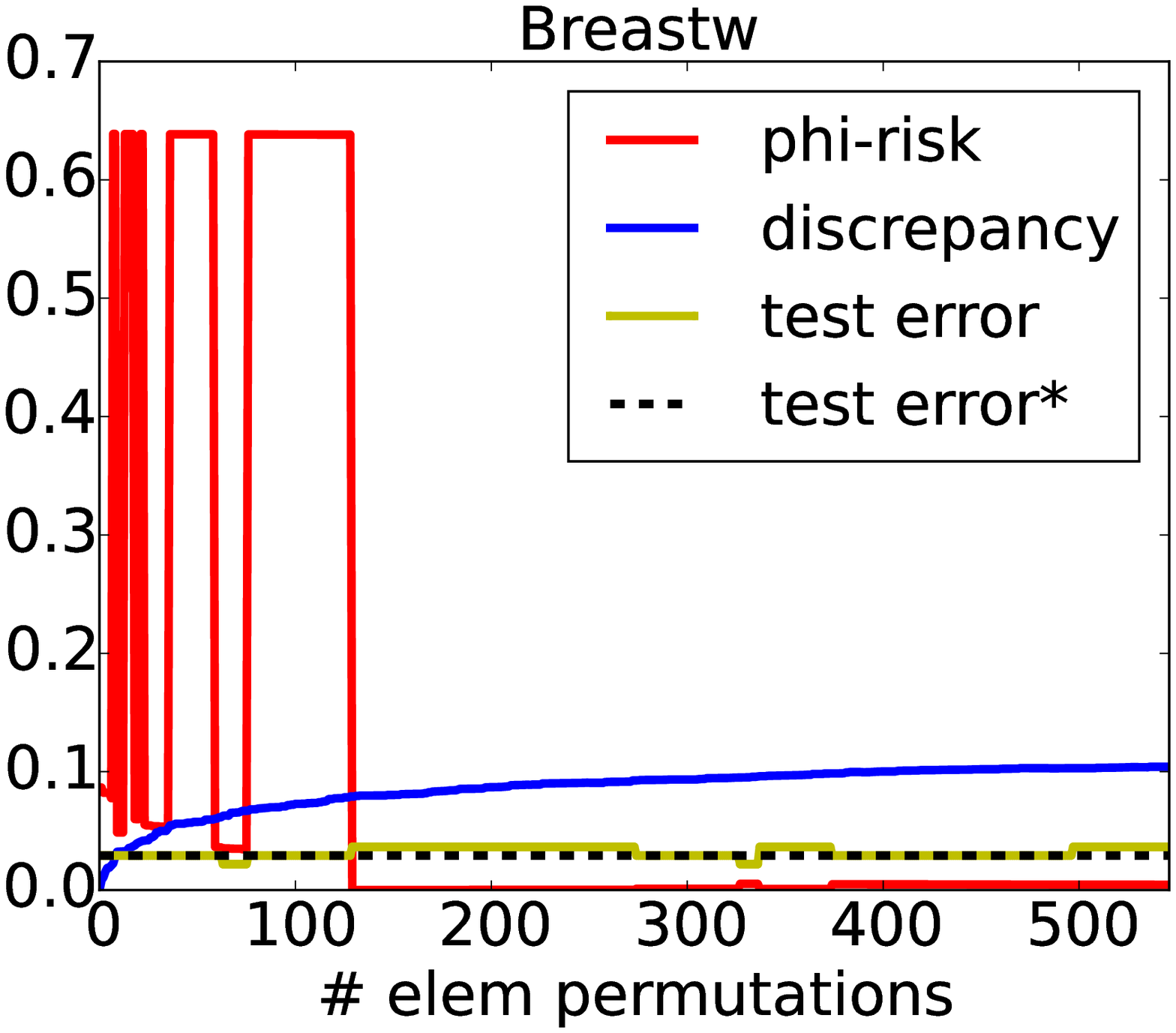}
& \includegraphics[height=4.10cm]{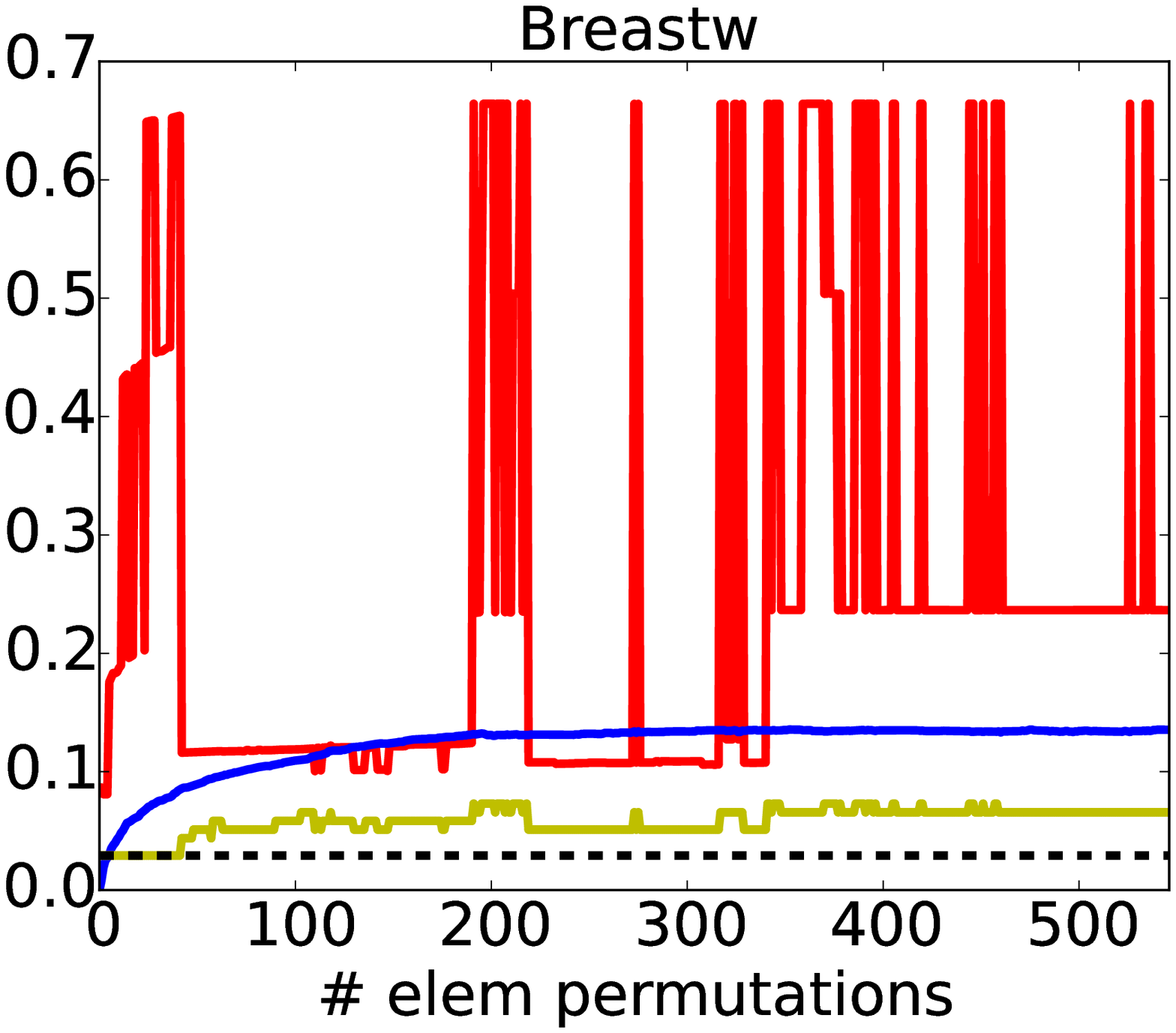}\\ 
$\matrice{m} \in S^*_m$ & $\matrice{m} \in S_m$ \\ \hline
\includegraphics[height=4.10cm]{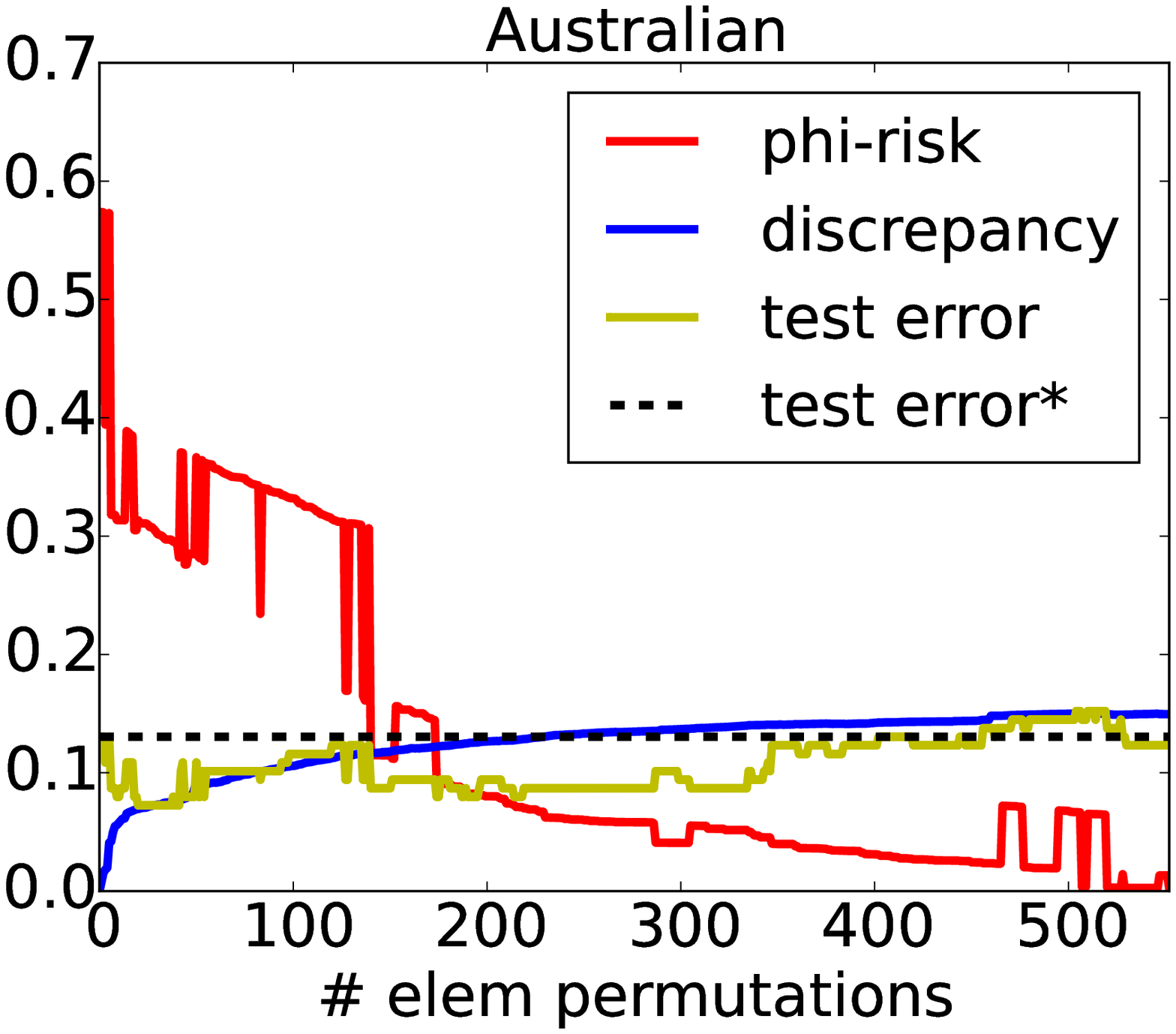}
& \includegraphics[height=4.10cm]{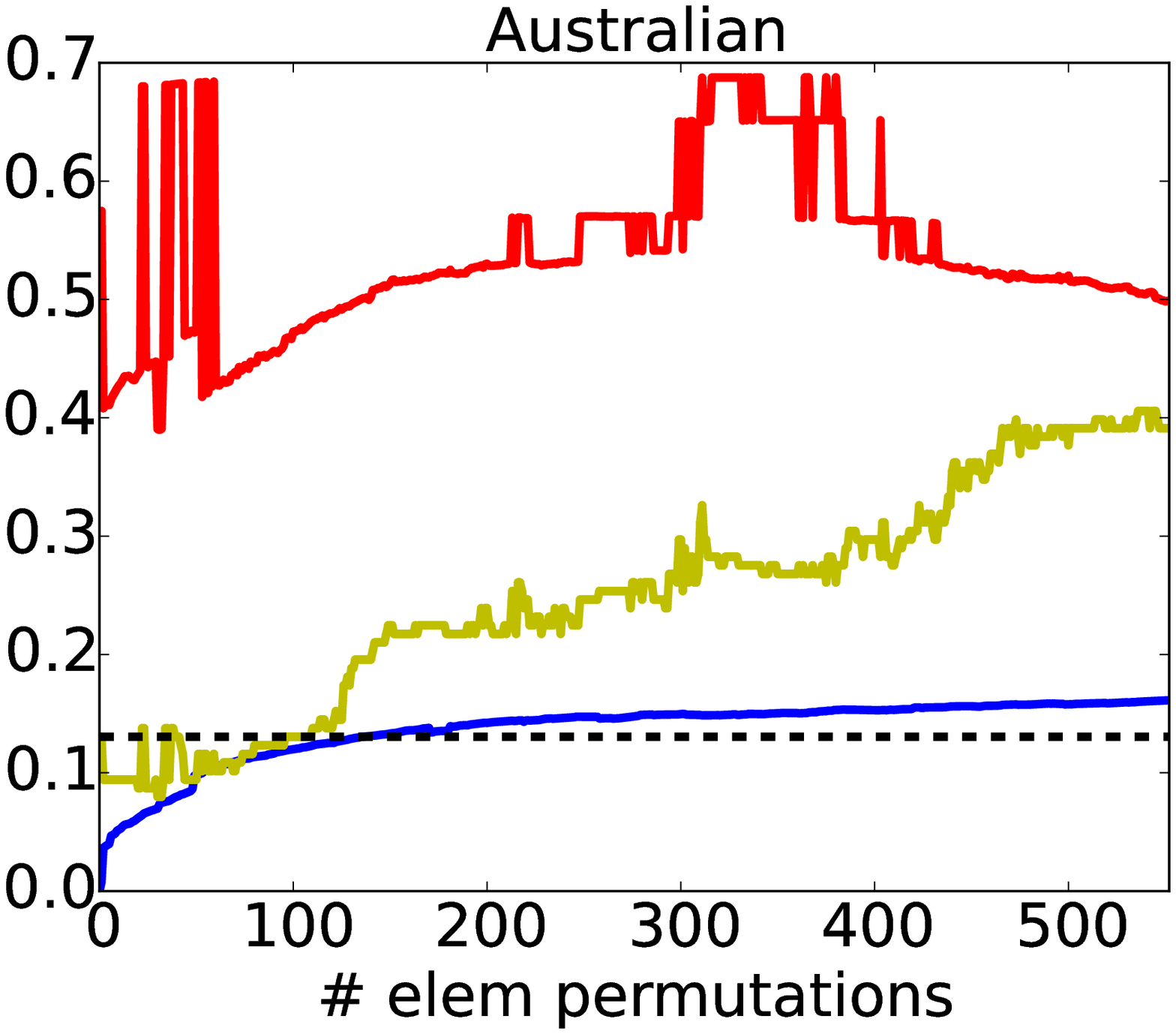}\\ 
$\matrice{m} \in S^*_m$ & $\matrice{m} \in S_m$ \\ \hline
 & \\ \hline\hline
\end{tabular}
\caption{Comparison (cont'd), for data optimisation, of algorithm \perm~in
  which elementary permutation matrices are constrained to be block-class (left), and \textit{not} constrained to be block-class
  (right).}\label{res-comp2}
\end{center}
\end{table}

\end{document}